\documentclass[onefignum,onetabnum]{siamonline190516}

\usepackage{lipsum}
\usepackage{amsfonts}
\usepackage{graphicx}
\usepackage{epstopdf}
\usepackage{bbm}

\newcommand{\bz}{\mathbf{z}}
\newcommand{\bx}{\mathbf{x}}
\newcommand{\by}{\mathbf{y}}
\newcommand{\bv}{\mathbf{v}}
\newcommand{\bu}{\mathbf{u}}

\newcommand{\bd}{\mathbf{d}}
\newcommand{\bP}{\mathbf{P}}
\newcommand{\bT}{\mathbf{T}}
\newcommand{\bD}{\mathbf{D}}
\newcommand{\bK}{\mathbf{K}}
\newcommand{\bL}{\mathbf{L}}
\newcommand{\bM}{\mathbf{M}}
\newcommand{\bI}{\mathbf{I}}
\newcommand{\bN}{\mathbf{N}}
\newcommand{\bA}{\mathbf{A}}
\newcommand{\bB}{\mathbf{B}}
\newcommand{\bX}{\mathbf{X}}

\newcommand{\R}{\mathbb{R}}
\newcommand{\bU}{\mathbf{U}}
\newcommand{\bS}{\mathbf{S}}
\newcommand{\bR}{\mathbf{R}}
\newcommand{\bmu}{\bm{\mu}}
\newcommand{\bV}{\mathbf{V}}
\newcommand{\cW}{\mathcal{W}}
\newcommand{\cA}{\mathcal{A}}
\newcommand{\cB}{\mathcal{B}}

\newcommand{\cJ}{\mathcal{J}}

\newcommand{\pathvar}{\mathbf{p}}

\usepackage{booktabs}
\usepackage{adjustbox}
\usepackage{multirow}
\newcommand{\wone}{\mathcal{W}^{(1)}_J}
\newcommand{\wtwo}{\mathcal{W}^{(2)}_J}

\usepackage{mathtools}
\usepackage{amsmath}
\usepackage{amssymb}
\usepackage{amsfonts}
\usepackage{bm} 
\usepackage{color}
\usepackage{comment}
\usepackage{caption}
\usepackage{subcaption}

\usepackage{xr}
\makeatletter
\newcommand*{\addFileDependency}[1]{
  \typeout{(#1)}
  \@addtofilelist{#1}
  \IfFileExists{#1}{}{\typeout{No file #1.}}
}
\makeatother

\usepackage{algorithmic}
\ifpdf
  \DeclareGraphicsExtensions{.eps,.pdf,.png,.jpg}
\else
  \DeclareGraphicsExtensions{.eps}
\fi

\usepackage{enumitem}
\setlist[enumerate]{leftmargin=.5in}
\setlist[itemize]{leftmargin=.5in}

\newsiamremark{remark}{Remark}
\newsiamremark{hypothesis}{Hypothesis}
\crefname{hypothesis}{Hypothesis}{Hypotheses}
\newsiamthm{claim}{Claim}

\headers{Generalized Geometric Scattering}{M.\ Perlmutter, A.\ Tong, F.\ Gao, G.\ Wolf, M.\ Hirn}

\title{Understanding Graph Neural Networks with Generalized Geometric Scattering Transforms}

\author{Michael Perlmutter\thanks{University of California, Los Angeles, \texttt{perlmutter@ucla.edu}} \and Alexander Tong\thanks{Universit\'{e} de Montr\'{e}al and Mila -- Quebec AI Institute, \texttt{alexander.tong@mila.quebec}} \and Feng Gao\thanks{Columbia University, \texttt{fg2539@cumc.columbia.edu}} \and Guy Wolf\thanks{Universit\'{e} de Montr\'{e}al and Mila -- Quebec AI Institute, \texttt{guy.wolf@umontreal.ca}} \and Matthew Hirn\thanks{Michigan State University, \texttt{mhirn@msu.edu}}}
\usepackage{amsopn}

\ifpdf
\hypersetup{
  pdftitle={Understanding Graph Neural Networks with Generalized Geometric Scattering Transforms},
  pdfauthor={Michael Perlmutter, Alexander Tong, Feng Gao, Guy Wolf, and Matthew Hirn}
}
\fi

\begin{document}

\maketitle

\begin{abstract}

    The scattering transform is a multilayered wavelet-based  architecture that acts as a model of convolutional neural networks. Recently, several works have generalized the scattering transform  to graph-structured data. Our work builds upon these constructions by introducing windowed and non-windowed geometric scattering transforms for graphs based upon two very general classes wavelets, which are in most cases based upon asymmetric matrices. We show that these transforms have many of the same theoretical guarantees as their symmetric counterparts. 
    As a result, the proposed construction unifies and extends known theoretical results for many of the existing graph scattering architectures. Therefore, it helps bridge the gap between geometric scattering and other graph neural networks by introducing a large family of networks with provable stability and invariance guarantees. These results lay the groundwork for future deep learning architectures for graph-structured data that have learned filters and also provably have desirable theoretical properties. 
    
\end{abstract}



\section{Introduction}



The scattering transform is a wavelet-based  model of convolutional neural networks (CNNs), introduced for signals defined on $\mathbb{R}^n$ by S. Mallat in \cite{mallat:scattering2012}. Like the front end of a CNN, the scattering transform produces a representation of an inputted signal through an alternating cascade of filter convolutions and pointwise nonlinearities. It primarily differs from CNNs in two respects: i) it uses predesigned, wavelet filters rather than filters learned through training data, and ii) it uses the complex modulus $|\cdot|$ as its nonlinear activation function rather than more common choices such as the rectified linear unit (ReLU).  These differences lead to a network which provably has desirable mathematical properties. In particular, the Euclidean scattering transform is: i) nonexpansive on $\mathbf{L}^2(\mathbb{R}^n)$, ii) invariant to translations up to a certain scale parameter, and iii) stable to certain diffeomorphisms. In addition to these theoretical properties, the scattering transform has also been used to achieve very good numerical results in fields such as
audio processing \cite{anden:scatAudioClass2011}, medical signal processing \cite{talmon:scatManifoldHeart2014}, computer vision \cite{oyallon2015deep}, and quantum chemistry~\cite{hirn:waveletScatQuantum2016}. 

While CNNs have proven tremendously effective for a wide variety of machine learning tasks, they typically assume that inputted data has a Euclidean grid-like structure.  However, many data sets of interest such as social networks, molecules, or surfaces appearing in computer graphics have an intrinsically non-Euclidean structure and are naturally modeled as graphs or manifolds. This has motivated the rise of geometric deep learning, a field which aims to generalize deep learning methods to non-Euclidean settings. In particular, a number of papers have  produced versions of the scattering transform for graphs \cite{gama:stabilityGraphScat2019,gama:diffScatGraphs2018,gao:graphScat2018,zou:graphCNNScat2018} and manifolds \cite{perlmutter:geoScatCompactManifold2020}. These constructions provide a model of geometric deep learning architectures such as graph neural networks in a manner analogous to the way that \cite{mallat:scattering2012} models CNNs. 

In this paper, we construct two new families of wavelet transforms on a graph $G$ from matrices $\bK$, which are in most cases asymmetric, and  provide a theoretical analysis of both of these wavelet transforms as well as the windowed and non-windowed scattering transforms constructed from them. Because the matrices $\bK$ are in general not symmetric, our wavelet transforms will not be a nonexpansive frames on the standard unweighted inner product space. Instead, they will be  nonexpansive on a certain weighted inner product space $\mathbf{L}^2(\bM),$ where $\bM$ is an invertible weighting matrix. In important special cases, our matrix $\bK$ will be either the lazy random walk matrix $\bP,$ its transpose $\bP^T,$ or its  symmetric counterpart given by  $\bT=\bD^{-1/2}\bP\bD^{1/2}.$  In these cases, the weighting matrix will depend on the geometry of $G.$ 

We will use these wavelets to construct windowed and non-windowed versions of the scattering transform on $G.$ The windowed scattering transform inputs a signal $\bx\in\mathbf{L}^2(\bM)$ and outputs a sequence of functions which we refer to as the scattering coefficients. We may view the windowed scattering transform as producing a sequence of features for each vertex. Therefore, it is well-suited for tasks such as node classification.  The non-windowed scattering transform replaces the low-pass matrix used in the definition of the windowed scattering transform with an averaging operator $\bmu$ and instead outputs a sequence of scalar-valued coefficients. 
In some cases, it can be viewed as the limit of the windowed scattering transform as the scale of the low-pass tends to infinity (evaluated at some fixed coordinate $0\leq i \leq n-1$). Since the non-windowed scattering transform produces a single set of coefficients for the entire graph, it is well suited for whole-graph level tasks such as graph classification or regression.
 

\subsection{Preliminaries}\label{sec: notation}

Let $G=(V,E,W)$ be a weighted, connected graph consisting of vertices $V$, edges $E$, and weights $W$, with $n$ vertices. Without loss of generality we take the vertices as $V = \{0, \ldots, n-1 \}$. If $\bx=(\bx(0),\ldots,\bx(n-1))^T$ is a function defined on $V$, we will identify $\bx$ with the corresponding vector in $\mathbb{R}^n$. Let $\mathbf{A}$ denote the {\it{weighted}} adjacency matrix of $G$, let $\mathbf{d}=(\bd(0),\ldots,\bd(n-1))^T$ be the weighted degree vector, and let $\bD=\text{diag}(\mathbf{d}).$ We will let    
$$    \bN\coloneqq\bI-\bD^{-1/2}\bA\bD^{-1/2}$$ 
be the normalized graph Laplacian, let $0\leq \omega_0\leq\omega_1\leq \ldots\leq \omega_{n-1} \leq 2$ denote the eigenvalues of $\bN,$ and let $\bv_0,\ldots,\bv_{n-1}$ be an orthonormal eigenbasis for $\mathbb{R}^n$ (with respect to the standard, unweighted inner product) such that $\bN\bv_i=\omega_i \bv_i.$ The matrix $\bN$ may be factored as 
$$\bN=\bV\Omega\bV^T,$$ 
where $\Omega=\text{diag}(\omega_0,\ldots,\omega_{n-1}),$ and $\bV$ is the unitary matrix whose $i$-th column is $\bv_i.$ 
One may check that $\omega_0=0$ and that we may choose $\mathbf{v}_0= \bd^{1/2} / \|\bd^{1/2}\|_2,$ where $\bd^{1/2}$ is defined componentwise. 
We note that since $G$ is connected,  we have  
$\omega_1 > 0$.

Our wavelet transforms will be constructed from the matrix $\bT_g$ defined by 
\begin{equation} \label{eqn: Tfactorization} 
    \bT_g \coloneqq \bV g(\Omega)\bV^T\coloneqq\bV\Lambda_g\bV^T,
\end{equation}
where $g:[0,2]\rightarrow[0,1]$ is some monotonically decreasing function such that $g(0)=1$ and $g(2)=0$, and  
$
\Lambda_g\coloneqq \text{diag}(g(\omega_0),\ldots,g(\omega_{n-1}))\coloneqq \text{diag}(\lambda_0,\ldots,\lambda_{n-1}).
$ 
We note that by construction we have $1= \lambda_0> \lambda_1\geq\ldots\geq\lambda_{n-1} \geq 0.$ As observed in e.g., \cite{levie2019transferability}, in the case where $g$ is a rational function, the matrix $\bT_g$ can be constructed in the spatial domain via functional calculus and there is no need to explicitly diagonalize $\mathbf{N}$, which may be computationally expensive. Indeed, this approach is used in many popular graph neural networks \cite{Defferrard2018,kipf2016semi}. When there is no potential for confusion, we will suppress dependence on $g$ and write $\bT$ and $\Lambda$ in place of $\bT_g$ and $\Lambda_g.$  As our main example, we will set $g(t) = g_\star(t)\coloneqq1-t/2,$ yielding
\begin{equation}\label{eqn: defT}
    \bT_{g_\star} =\bI-\frac{1}{2}\left(\bI-\bD^{-1/2}\bA\bD^{-1/2}\right)=\frac{1}{2}\left(\mathbf{I}+\bD^{-1/2}\mathbf{A}\bD^{-1/2}\right).
\end{equation}

In \cite{gama:diffScatGraphs2018}, Gama et al.\ constructed a graph scattering transform  using wavelets which are polynomials in $\bT_{g_\star},$ and in \cite{gao:graphScat2018}, Gao et al. defined a closely related graph scattering transform from polynomials of the lazy random walk matrix
\begin{equation*}
    \bP\coloneqq\bD^{1/2}\bT_{g_\star}\bD^{-1/2}=\frac{1}{2}\left(\mathbf{I}+\mathbf{A}\mathbf{D}^{-1}\right).
\end{equation*}
In order to unify and generalize these frameworks, we will let $\bM$ be an $n\times n$ invertible matrix and let $\bK$
be the matrix defined by 
\begin{equation}\label{eqn: K}
    \bK = \bK_{g, \bM} \coloneqq\bM^{-1}\bT_g\bM.
\end{equation}
Note that $\bK$ depends on the choice of both $g$ and $\bM$, and thus includes a very large family of matrices. However, we shall suppress this dependence in order to avoid cumbersome notation. As important special cases, we note that we may obtain $\bK = \bT$ by setting $\bM = \bI$, and we may obtain $\bP$ and $\bP^T$ by setting $g(t)=g_\star(t)$ and letting $\bM = \bD^{-1/2}$ and $\bM = \bD^{1/2}$, respectively. More generally, as we shall see below, the function $g$ controls the eigenvalues of $\bK$ while the matrix $\bM$ influences the eigenvectors of $\bK$. Since the eigenvalues and eigenvectors uniquely determine $\bK$, they strongly affect any wavelets derived from $\bK$. Thus, the two parameters $g$ and $\bM$ allow for a wide range of possible graph wavelet constructions.

In Section \ref{sec: wavelets}, we will construct two wavelet transforms $\cW^{(1)}_J$ and $\cW^{(2)}_J$ from functions of $\bK$. 
We note that the matrix $\bK$  is not self-adjoint on the standard, unweighted inner product space (except in the case $\bK=\bT$). Therefore, we will introduce a weighted inner product space, $\mathbf{L}^2(\bM),$ of signals defined on $V$ with inner product defined by\footnote{To avoid confusion, we note that our definition differs from the notation sometimes used in the literature where $\langle\mathbf{x},\mathbf{y}\rangle_{\mathbf{M}}$ is defined by $\mathbf{y}^T\mathbf{M}\mathbf{x}$.}   
\begin{equation}\label{eqn: weighted inner product}
    \langle \bx,\by\rangle_{\bM} \coloneqq 
    \langle \bM\bx,\bM\by\rangle_2,
\end{equation}
where $\langle\cdot,\cdot\rangle_2$ denotes the standard, unweighted inner product on $\mathbb{R}^n$.
To better understand this definition, we note that if $\bM=\bD^{\alpha/2}$ then
$
  \langle \bx,\by\rangle_{\bD^{\alpha/2}} = 
    \sum_{i=0}^{n-1} \bx(i)\by(i)\bd(i)^{\alpha}
$. 
Thus, $\langle \bx,\by\rangle_{\bD^{\alpha/2}}$ is a weighted $\mathbf{L}^2$ inner product with weights depending on the degree of each vertex. We note that the norms 
   $\|\bx\|_\bM^2\coloneqq\langle\bx,\bx\rangle_{\bM}$ and $\|\bx\|_2^2=\langle\bx,\bx\rangle_2$
are equivalent and 
    $\frac{1}{\|\bM^{-1}\|_2}\|\bx\|_2\leq \|\bx\|_{\bM}\leq \|\bM\|_2\|\bx\|_2,$
where for  a matrix  $\mathbf{B},$  we shall let $\|\mathbf{B}\|_{\bM}$ and $\|\mathbf{B}\|_2$  denote its operator norms on $\mathbf{L}^2(\bM)$ and on the standard, unweighted  $\mathbf{L}^2$ space. If $\mathcal{I}$ is a countable indexing set, and  $\mathcal{X}=\{\mathbf{x}_i\}_{i\in\mathcal{I}}$ is a family of vectors, then we will also let $\|\mathcal{X}\|_{\bM}$ denote the $\bm{\ell}^2(\mathbf{L}^2(\bM))$ norm of $\mathcal{X}$ and $\|\mathcal{X}\|_{2}$ denote its norm on the vector-valued $\mathbf{L}^2$ space. If $\Gamma=(\bB_i)_{i\in\mathcal{I}}$ is a family of matrices, we will say that $\Gamma$ is a frame if there exist $0<c\leq C<\infty$ such that 
\begin{equation}
    c\|\bx\|^2_{\bM}\leq\|\mathcal{B}\bx\|^2_{\bM}\coloneqq \sum_{i\in\mathcal{I}}\|\mathbf{B}_i\bx\|^2_{\bM}\leq C\|\bx\|^2_{\bM},\quad \forall \, \mathbf{x}\in\mathbf{L}^2(\bM).\label{eqn: frameAB}
\end{equation}
The following lemma 
will be useful in studying the frame bounds of the wavelet transforms constructed from $\bK.$
 
\begin{lemma}\label{lem: selfadjoint}
The matrix $\bK$ is self-adjoint on $\mathbf{L}^2(\bM).$
\end{lemma}
\begin{proof}
Since $\bT$ is symmetric, we may use \eqref{eqn: K} and \eqref{eqn: weighted inner product} to see
\begin{align*}
    \langle\bK\bx,\by\rangle_\bM 
    =\langle\bM \bx,\bT\bM\by\rangle_2
    =\langle\bM\bx,\bM(\bM^{-1}\bT\bM)\by\rangle_2
    =\langle\bM\bx,\bM\bK\by\rangle_2
    =\langle\bx,\bK\by\rangle_\bM.
\end{align*}
\end{proof}

It will  frequently be useful to consider the eigenvector decompositions of $\bT$ and $\bK.$ By definition we have $\bT=\bV\Lambda\bV^T$, and therefore 
$\bT\mathbf{v}_i=\lambda_i\mathbf{v}_i.$ 
Since the matrices $\bT$ and $\bK$ are similar with $\bK=\bM^{-1}\bT\bM,$   
one may use the definition of $\langle\cdot,\cdot\rangle_\bM$ to verify that the vectors 
\begin{equation*}
    \mathbf{u}_i\coloneqq\bM^{-1}\mathbf{v}_i
\end{equation*}
form an orthonormal eigenbasis for $\mathbf{L}^2(\bM)$ with $\bK\mathbf{u}_i=\lambda_i\mathbf{u}_i.$ One may also verify that 
$   \mathbf{w}_i \coloneqq \bM^T \mathbf{v}_i
$ 
is a left eigenvector of $\bK$ and $\mathbf{w}_i^T\bK = \lambda_i\mathbf{w}_i^T$ for all $0\leq i \leq n-1.$ 

In the following section, we will construct wavelets from polynomials of $\bK,$ where for a polynomial  $p(t)=a_kt^k+\ldots+a_1t+a_0$ we define  $p(\mathbf{B}) \coloneqq a_k\mathbf{B}^k+\ldots+a_1\mathbf{B}+a_0\bI.$
The following lemma uses  \eqref{eqn: Tfactorization} to derive a formula for polynomials of $\bK$ and $\bT$ and relates the operator norms of  polynomials of $\bK$ to polynomials of $\bT.$ It will be useful for studying the wavelet transforms introduced in the following section. For a proof, please see Section \ref{sec: Proof of Polynomial properties}.
\begin{lemma}\label{lem: polynomialproperties}
For any polynomial $p,$  we have 
\begin{equation} \label{eqn: Tpolys}
    p(\bT)= \bV p(\Lambda)\bV^T \quad\text{and} \quad p(\bK) = \bM^{-1} p(\bT) \bM =\bM^{-1} \bV p(\Lambda)\bV^T \bM.
\end{equation}
Consequently, for all $\bx\in\mathbf{L}^2(\bM)$, we have
    $\|p(\bK)\bx\|_\bM=\|p(\bT)\bM\bx\|_2.$
\end{lemma}
In light of Lemma \ref{lem: polynomialproperties}, for any polynomial $p,$ we may define $p(\bT)^{1/2}$ and $p(\bK)^{1/2}$ by 
\begin{equation}\label{eqn: defsquareroot}
    p(\bT)^{1/2}\coloneqq \bV p(\Lambda)^{1/2}\bV^T\quad\text{and}\quad p(\bK)^{1/2}=\bM^{-1}\bV p(\Lambda)^{1/2}\bV^T\bM,
\end{equation}
where the square root of the diagaonal matrix $p(\Lambda)$ is defined entrywise. We
may readily verify that 
$
    p(\bT)^{1/2}p(\bT)^{1/2}=p(\bT) \quad\text{and}\quad p(\bK)^{1/2}p(\bK)^{1/2}=p(\bK).
$

\subsection{Previous work on graph scattering transforms} \label{sec: related}

Several previous works have introduced different formulations of the graph scattering transform. \cite{gama:diffScatGraphs2018,gama:stabilityGraphScat2019,zou:graphCNNScat2018} construct the scattering transform using symmetric wavelets and show their constructions have similar stability and invariance properties to the Euclidean scattering transform. There has also been works
 empirically demonstrating that the graph scattering transform is effective for tasks such as graph classification \cite{gao:graphScat2018}, vertex classification \cite{tong2020data, min2020scattering, min2021gsan}, graph synthesis \cite{zou:graphScatGAN2019, Castro2020}, and  combinatorial optimization  \cite{min2022can}. However, much of this empirical work has been done using asymmetric wavelets to which the theoretical guarantees of \cite{gama:diffScatGraphs2018,gama:stabilityGraphScat2019} and \cite{zou:graphCNNScat2018} do not apply.

In this paper, we will focus on unifying and generalizing the theoretical properties of the different formulations of the graph scattering transform. 
Analogously to the Euclidean scattering transform, we will show that the windowed graph scattering transform is: i) nonexpansive on $\mathbf{L}^2(\bM),$ ii) invariant to permutations of the vertices, up to a factor depending on the scale of the low-pass (for certain choices of $\bK$), and iii) stable to  graph perturbations. Similarly, we will show that the non-windowed scattering transform is i) Lipschitz continuous on $\mathbf{L}^2(\bM),$ ii) fully invariant to permutations, and iii) stable to graph perturbations. Importantly, we note that this is the first work to produce such theoretical guarantees for graph scattering transforms using asymmetric wavelets and is also the first to establish Lipshitz continuity of the non-windowed graph scattering transform.


In \cite{zou:graphCNNScat2018}, the authors  construct a family of wavelet convolutions using the spectral decomposition of the unnormalized graph Laplacian and define a windowed scattering transform as an iterative series of wavelet convolutions and nonlinearities. They then prove results analogous to Theorems \ref{thm: nonexpansive}, \ref{thm: conservation of energy}, and \ref{thm: permuatation invariance} of this  paper for their windowed scattering transform. They also introduce a notion of stability to graph perturbations. However, their notion of graph perturbations is significantly different than the one we consider in Section \ref{sec: stability}. 

In \cite{gama:diffScatGraphs2018}, the authors construct a family of wavelets from polynomials of $\bT_g,$ in the case where $g(t)=g_\star(t)=1- t/2,$ and showed that the resulting non-windowed scattering transform was stable to graph perturbations. This construction was generalized in \cite{gama:stabilityGraphScat2019}, where the authors introduced a more general class of graph convolutions, constructed from a class of symmetric matrices known as ``graph shift operators.'' 
The wavelet transform considered in \cite{gama:diffScatGraphs2018} is nearly identical to the $\cW^{(2)}_J$ introduced in Section \ref{sec: wavelets}, in the special case where $g(t) = g_\star(t)$ and $\bM=\bI,$ with the only difference being that our wavelet transform includes a low-pass filter. 

In \cite{gao:graphScat2018}, wavelets were constructed from the lazy random walk matrix $\bP=\bD^{1/2}\bT\bD^{-1/2}.$ These wavelets are essentially the same as the $\cW^{(2)}_J$ in the case where $g(t) = g_\star(t)$ and $\bM=\bD^{-1/2},$ although similarly to \cite{gama:diffScatGraphs2018}, the wavelets in \cite{gao:graphScat2018} do not use a low-pass filter. In all of these previous works, the authors carry out substantial numerical experiments and demonstrate that scattering transforms are effective for a variety of graph deep learning tasks. 
We also note that Chen, Cheng, and Mallat first introduced a substantially different version of the graph scattering transform in \cite{chen:scatHaar2014, cheng2016deep} using Haar wavelets. However, the construction and analysis 
considered there
differs substantially from the previously discussed works. 

In addition to helping us understand the stability and invariance properties of deep networks on graphs, the scattering transform also helps us investigate the important question: ``what sort of filters should be used in a graph neural network?" Most popular graph neural networks, such as \cite{kipf2016semi}, typically average information over one-hop neighborhoods in order to produce a smooth hidden representation of the vertices, which effectively corresponds to filtering out high-frequency information via a low-pass filter. Scattering, on the other hand, uses multiscale filters that can encode long-range dependencies and effectively capture high-frequency information. In \cite{min2020scattering,tong2020data}, it was shown that this allows for improved numerical performance in situations where high-frequency information is important. Moreover, graph scattering style networks also been shown to be effective for molecule generation \cite{zou:graphScatGAN2019} and solving combinatorial optimization problems \cite{min2022can}. In the former case, the large receptive field of the wavelets allows scattering to capture the global structure of the molecule and in the latter case the use of wavelets rather than low-pass filters allows one to distinguish a member of the clique from a node which is connected to to most, but not all nodes within the clique.

Here we shall focus on unifying and generalizing the theory of several of these previous constructions. Our introduction of the matrix $\bM$ allows us to obtain wavelets very similar to either \cite{gama:diffScatGraphs2018} or \cite{gao:graphScat2018} as special cases. Moreover, the  introduction of the tight wavelet frame $\cW^{(1)}_J$ allows us to produce a network with provable conservation of energy and nonexpansive properties analogous to \cite{zou:graphCNNScat2018}. To highlight the generality of our setup, we introduce both windowed and non-windowed versions of the scattering transform using general (wavelet) frames and provide a detailed theoretical analysis of both. 
%
 
 \subsection{Organization, contributions, and summary of main results} In Section \ref{sec: wavelets}, we will construct two family of graph wavelets $\mathcal{W}_J^{(1)}$ and $\mathcal{W}_J^{(2)}$. In Section \ref{sec: scattering}, we will introduce windowed and non-windowed versions of the graph scattering trasnform and analyze their continuity and invariance properties. Then, in Section \ref{sec: stability}, we will analyze the stability of the networks to perturbations. In Section \ref{sec: GNNS}, we will discuss the relationship between scattering and other GNNs, and, in Section \ref{sec: experiments}, we will present numerical experiments before offering a brief conclusion in Section \ref{sec: future}.
 
 The following is a summary of our main theoretical results. Unless otherwise stated, all results apply to both choices of wavelets, both the windowed and non-windowed scattering transform, and to general diffusion matrices $\mathbf{K}$.
 \begin{itemize}
     \item \textbf{Section \ref{sec: wavelets}:} Propositions \ref{prop: waveletisometries} and \ref{prop: nonexpansivewaveletframes} show that the wavelets $\mathcal{W}^{(1)}_J$ is an isometry and is $\mathcal{W}^{(2)}_J$ a non-expansive frame on $\mathbf{L}^2(\mathbf{M})$.
     \item \textbf{Section \ref{sec: scattering}:} Proposition \ref{prop: limit} shows that if $\bmu=\bu_0$, then the windowed scattering transform converges to the non-windowed scattering transform as $J\rightarrow\infty$. Theorem \ref{thm: nonexpansive}, shows that the windowed scattering transform is non-expansive and that the non-windowed scattering transform is Lipschitz continuous on $\mathbf{L}^2(\mathbf{M})$. Theorems \ref{thm: energydecay} shows that the energy in the $m$-th order scattering coefficients decays exponentially in $m$ and Theorem \ref{thm: conservation of energy} shows that therefore the scattering transform preserves all of the energy of the input signal if we use the wavelets $\mathcal{W}_J^{(1)}$. Theorems \ref{thm: equivariance} and  \ref{thm: permuationinvariancewindowed} show that the windowed scattering transform is equivariant and that the non-windowed scattering transform is invariant to permutations. Lastly, Theorem \ref{thm: permuatation invariance} shows that the windowed scattering transform is invariant in the limit at $J\rightarrow\infty$ if $\bK=\bP^T$.
     \item \textbf{Section \ref{sec: stability}:} Theorem \ref{thm: wavelet stability1} and \ref{thm: wavelet stability2} provide stability guarantees for  $\mathcal{W}^{(1)}$ and $\mathcal{W}^{(2)}$ in the special case where $\bK=\bT$. Theorem \ref{thm: transferTtoP} then allows these results to be extended to general $\bK$. Finally, Theorem \ref{thm: windowstability}
 and  \ref{thm: scattering stability no window} establish stability of the windowed and non-windowed scattering transforms. 
 \end{itemize}
 \subsubsection{Contributions}
As discussed in Section \ref{sec: related}, there has been a significant amount of work developing, analyzing, and applying different versions of the graph scattering transform. Therefore, in this section, for the sake of clarity, we will highlight  several aspects of our paper which are different than these previous works. 
 \begin{itemize}
     \item  The wavelet family $\mathcal{W}^{(2)}_J$ includes the wavelets of \cite{gama:diffScatGraphs2018} and \cite{gao:graphScat2018} as special cases, but it also includes many other wavelets, including in particular, a one family parameter of wavelets based on diffusion operators of the form $\bK=\bD^{\alpha}\bT\bD^{-\alpha}$, where the wavelets from \cite{gama:diffScatGraphs2018} and \cite{gao:graphScat2018} correspond to $\alpha=0$ and $\alpha=-.5$. Moreover, the wavelets  $\mathcal{W}^{(1)}_J$ are new and are not utilized in any previous version of the graph scattering transform. In our experiments in Section \ref{sec: experiments}, we consider scattering transforms which use wavelets $\mathcal{W}^{(\beta)}$, based on diffusion matrices $\mathbf{K}=\bD^{-\alpha}\bT\bD^\alpha$, for $\beta=1,2$ and $\alpha=-0.5,-0.25,0,0.25,0.5$. We  find that the optimal choice of wavelet varies significantly from one data set to another. Therefore, it is our recommendation that practitioners use a validation procedure to select the optimal wavelets for a given task. 
     Moreover, we note that the new wavelets, $\mathcal{W}_J^{(1)}$  out perform $\mathcal{W}_J^{(2)}$ on most data sets and that the intermediate values of $\alpha,$ i.e., $\alpha=\pm0.25$ often deliver superior performance to the values considered in previous work ($\alpha=0$ or $-0.5$).
     \item \cite{gama:diffScatGraphs2018} and \cite{gao:graphScat2018} considered only non-windowed versions of the scattering transform, where we consider both windowed and non-windowed versions. Importantly, unlike its non-windowed counterpart, the windowed scattering transform can be applied to vertex-level tasks such as node classification or combinatorial optimization problems. Additionally, utilizing both versions of the scattering transform is crucial to proving that the non-windowed scattering transform is Lipschitz continuous in Theorem \ref{thm: nonexpansive} and no analogous result for the non-windowed scattering transform exists in previous work.
     \item Our stability result for the non-windowed scattering transform, Theorem \ref{thm: scattering stability no window} considers perturbations both to the graph structure and to the input signal $\mathbf{x}$ whereas the analogous result in \cite{gama:diffScatGraphs2018} only considered perturbations to the graph structure. The proof of this result directly utilizes the Lipschitz continuity of the non-windowed scattering transform discussed above and therefore would have been  non-trivial for the authors of \cite{gama:diffScatGraphs2018} to prove since they did not consider a windowed scattering transform.
     \end{itemize}
 
\section{The Graph Wavelet Transform}
\label{sec: wavelets}

In this section, we will construct two families of  graph wavelet transforms based off of the  matrix $\bK=\bM^{-1}\bT\bM$ introduced in Section \ref{sec: notation}.  In the following sections, we provide a theoretical analysis of the scattering transforms constructed from each of these wavelet transforms.

Let $J\geq 0.$ For $0\leq j \leq J+1,$ let $p_j$ be the polynomial defined by
\begin{equation*}
    p_j(t) \coloneqq \begin{cases}
    1-t & \text{if }j=0\\
    t^{2^{j-1}}-t^{2^j} &\text{if }1\leq j \leq J\\
    t^{2^J}&\text{if }j=J+1
    \end{cases},
\end{equation*}
and let $q_j(t) \coloneqq p_j(t)^{1/2}$ for $0\leq t\leq 1.$ We note that by construction 
\begin{equation}\label{eqn: sum to 1}
    \sum_{j=0}^{J+1}p_j(t)=\sum_{j=0}^{J+1}q_j(t)^2=1\text{ for all }0\leq t\leq 1.
\end{equation}
Given these functions, we define two wavelet transforms by 
\begin{align*}
    \cW^{(1)}_J &\coloneqq \left\{\Psi^{(1)}_j, \Phi^{(1)}_J\right\}_{0 \leq j \leq J}, \quad \Psi^{(1)}_j \coloneqq q_j(\bK), \quad \Phi^{(1)}_J \coloneqq q_{J+1}(\bK),\quad\text{and}
\\
\cW^{(2)}_J &\coloneqq \left\{\Psi^{(2)}_j, \Phi^{(2)}_J\right\}_{0 \leq j \leq J}, \quad \Psi^{(2)}_j \coloneqq p_j(\bK), \quad \Phi^{(2)}_J \coloneqq p_{J+1}(\bK),
\end{align*}
where $q_j(\bK)$ is defined as in \eqref{eqn: defsquareroot}. The next two propositions show $\cW^{(1)}_J$ is an isometry and $\cW^{(2)}_J$ is a nonexpansive frame on $\mathbf{L}^2(\bM)$ (defined in \eqref{eqn: weighted inner product}). We provide proofs in Section \ref{sec: The proof of Proposition wavelet isometrices}.

\begin{proposition}\label{prop: waveletisometries}
$\cW^{(1)}_J$ is an isometry from $\mathbf{L}^2(\bM)$ to $\bm{\ell}^2(\mathbf{L}^2(\bM)).$ That is, 
\begin{equation*}
    \big\|\mathcal{W}^{(1)}_J\bx\big\|^2_{\bM}\coloneqq \sum_{j=0}^{J}\big\|\Psi^{(1)}_j\bx\big\|^2_\bM + \big\|\Phi^{(1)}_J\bx\big\|^2_\bM  = \|\bx\|^2_\bM\quad\text{for all }\bx \in\mathbf{L}^2(\bM).
\end{equation*}
\end{proposition}

\begin{proposition}\label{prop: nonexpansivewaveletframes}
$\cW^{(2)}_J$ is a nonexpansive frame,   i.e., there exists a universal constant $c>0$, which in particular is independent of $\bM,$ $J$, or the eignenvalues of $\bT$,   such that
 \begin{equation*}
c\|\bx\|_\bM^2\leq    \big\|\mathcal{W}^{(2)}_J\bx\big\|^2_{\bM}\coloneqq \sum_{j=0}^{J}\big\|\Psi^{(2)}_j\bx\big\|^2_\bM + \big\|\Phi^{(2)}_J\bx\big\|^2_\bM  \leq  \|\bx\|^2_\bM \quad \text{for all } \bx\in\mathbf{L}^2(\bM).
\end{equation*}
\end{proposition}

\begin{remark}
If we omit 
the low-pass operator $\Phi_J^{(2)},$ we can repeat the arguments of  Proposition 4.1 of \cite{gama:diffScatGraphs2018} to show that $\big\{\Psi^{(2)}_j\big\}_{0 \leq j \leq J}$ is a non-expansive frame, with a lower bound depending on the geometry of $G$, when restricted to $\mathbf{x}$ such that $\langle\mathbf{x},\bu_0\rangle_\mathbf{M}=0.$ For certain tasks, this may be advantageous. For example, if  $\bK=\bP^T,$ we may check that $\bu_0$ is a constant vector and that $\Psi_j\bu_0=p_j(\bK)\bu_0=0$. Therefore, these restricted wavelets could be used to produce a representation of an input $\bx$ which is invariant to the addition of a constant vector. 
\end{remark}

\section{The Geometric Scattering Transform}\label{sec: scattering}

In this section, we will construct the scattering transform as a multilayered architecture built off of a  frame $\cW$ such as the  wavelet transforms  $\mathcal{W}_J^{(1)}$ and $\mathcal{W}_J^{(2)}$ introduced in Section \ref{sec: wavelets}. We shall see the scattering transform   is a continuous operator on $\mathbf{L}^2(\bM)$ whenever $\cW$ is nonexpansive. 
We shall also see that it has desirable conservation of energy bounds when $\cW=\cW^{(1)}_J$ due to the fact that $\cW^{(1)}_J$ is an isometry. On the other hand, we shall see in the following section that the scattering transform  has stronger stability guarantees, which are independent of the graph size, when $\cW=\cW^{(2)}_J.$ 

\subsection{Definitions}


We generalize the wavelet frames $\cW_J^{(1)}$ and $\cW_J^{(2)}$ defined in Section \ref{sec: wavelets}. Let $\mathcal{J}$ be any countable indexing set, and let
$    \mathcal{W} \coloneqq \{\Psi_j,\Phi\}_{j\in\mathcal{J}}
$ be a frame on $\mathbf{L}^2(\bM)$ with
\begin{equation*}
    c\|\bx\|^2_{\bM}\leq\|\mathcal{W}\bx\|^2_{\bM}\coloneqq \sum_{j\in\mathcal{J}}\|\Psi_j\bx\|^2_{\bM}+ \|\Phi\bx\|^2_{\bM}\leq C\|\bx\|^2_{\bM},
\end{equation*}
for some $0 < c \leq C < \infty.$    In this paper, we are primarily interested in the case where $\mathcal{J}=\{0,\ldots,J\}$ and $\cW$ is either $\cW^{(1)}_J$ or $\cW^{(2)}_J$ meaning that $\Psi_j = \Psi_j^{(i)}$ and $\Phi = \Phi_J^{(i)}$ for $i=1,2$, respectively.
If a result is specific to these cases, will write ``let $\cW$ be either $\cW_J^{(1)}$ or $\cW_J^{(2)}$,'' in which case we are implicitly assuming that $\cJ = \{0, \ldots, J \}$.
In general, it will be useful to think of the matrices $\Psi_j$ as wavelets and $\Phi$ as a low-pass filter, but we emphasize that this specification is not required for many of the results that follow. Indeed, we will define the geometric scattering transform for generic frames in order to highlight the relationship between  properties of the scattering transform and properties of the underlying frame.


We let $M$  be the pointwise modulus operator on, $M\bx \coloneqq ( |\bx(0)|, \ldots, |\bx(n-1)|)$ and  let
$    \bU[j]\bx \coloneqq M\Psi_j\bx$
 for $j\in\mathcal{J}.$ We view this transformation $\bU[j]$ as a hidden-layer of our network and will construct a multilayered architecture by iteratively applying this transformation at different values of $j$. Formally, for $m >  0,$ let $\mathcal{J}^m$ denote the $m$-fold Cartesian product of $\mathcal{J}$ with itself, and for an index path $\pathvar=(j_1,\ldots,j_m)\in\mathcal{J}^m$ let 
\begin{equation*}
     \bU[\pathvar]\bx \coloneqq \bU[j_m]\ldots\bU[j_1]\bx = M\Psi_{j_m} \cdots M\Psi_{j_1}\bx.
\end{equation*}
For $m=0$ we declare that $\mathcal{J}^0$ is the empty set and interpret $\bU[\pathvar_e]\bx=\bx$ when $\pathvar_e$ is the ``empty index.'' Next, we define the windowed scattering coefficients by 
\begin{equation*}
    \bS[\pathvar]\bx \coloneqq \Phi\bU[\pathvar]\bx.
\end{equation*}
In this definition, the final multiplication by the (low-pass) matrix $\Phi$ is interpreted as a local averaging, but one could choose a different type of matrix $\Phi$ so long as the frame condition still holds. We will also define non-windowed scattering coefficients which replaces this local averaging by a weighted global averaging. Specifically, we let $\bmu\in\mathbf{L}^2(\bM)$ be a weighting vector and define the non-windowed scattering coefficients by
\begin{equation*}
    \overline{\bS}[\pathvar]\bx \coloneqq \langle \bmu,\bU[\pathvar]\bx\rangle_\bM.
\end{equation*}
One natural choice is  $\bmu=\left(\bM^T\bM\right)^{-1}\mathbf{1},$ where $\mathbf{1}$ is the vector of all ones. In this case, one may verify that  $\overline{\bS}[\pathvar]\bx=\|\bU[\pathvar]\bx\|_1,$
and we recover a setup similar to \cite{gao:graphScat2018}. Another natural choice is $\bmu=\mathbf{u}_0 = \bM^{-1} \bv_0,$ in which case we recover a setup similar to \cite{gama:diffScatGraphs2018} if we set $\bM=\bI.$

Given these coefficients, we define an operator $\bU:\mathbf{L}^2(\bM)\rightarrow \bm{\ell}^2(\mathbf{L}^2(\bM)),$ by
\begin{equation*}
    \bU\bx \coloneqq \{\bU[\pathvar]\bx: m\geq 0, \, \pathvar=(j_1,\ldots,j_m) \in\mathcal{J}^{m}\}.
\end{equation*}
Next we define the windowed and nonwindowed scattering transform $\bS:\mathbf{L}^2(\bM)\rightarrow\bm{\ell}^2(\mathbf{L}^2(\bM))$ and $\overline{\bS}:\mathbf{L}^2(\bM)\rightarrow\bm{\ell}^2$ by 
\begin{equation*}
    \bS\bx \coloneqq \{\bS[\pathvar]\bx: m\geq 0, \, \pathvar \in\mathcal{J}^{m}\} \quad\text{and}\quad 
    \overline\bS\bx \coloneqq \{\overline\bS[\pathvar]\bx: m\geq 0, \, \pathvar \in\mathcal{J}^{m}\}.
\end{equation*}
When $\cW$ is either of the graph wavelet transforms $\cW^{(1)}_J$ or $\cW^{(2)}_J$ constructed in Section \ref{sec: wavelets} and $\cJ=\{0,\ldots,J\}$, we may write $\bS_J$ in place of $\bS$ if we want to emphasize the dependence on $J$. Similarly, when we want to emphasize dependence on $\bmu$ we will write $\overline{\bS}_{\bmu}$ in place of $\overline{\bS}.$  We will let $\mathbf{S}^\ell$ and $\overline{\bS}^\ell$ denote $\ell$-th layer of the windowed and nonwindowed scattering transform:
\begin{equation*}
    \bS^{(\ell)}\bx \coloneqq \{\bS[\pathvar]\bx:  \pathvar=(j_1,\ldots,j_\ell) \in\mathcal{J}^{\ell}\},\quad\text{and}\quad \overline{\bS}^{(\ell)}\bx \coloneqq \{\overline{\bS}[\pathvar]\bx:  \pathvar=(j_1,\ldots,j_\ell) \in\mathcal{J}^{\ell}\}.
\end{equation*}

When $\cW$ is either of the wavelet transforms $\cW^{(1)}_J$ or $\cW^{(2)}_J$ constructed in Section \ref{sec: wavelets} and $\bmu=\bu_0,$ we may view the non-windowed scattering transform as the limit of the windowed scattering transform as $J\rightarrow\infty.$ Formally, we can prove the following proposition.

\begin{proposition}\label{prop: limit}
Let $\cW$ be either of the wavelet transforms $\cW^{(1)}_J$ or $\cW^{(2)}_J$ constructed in Section \ref{sec: wavelets}, and let $\bmu=\bu_0.$ Then for all paths $\pathvar\in\mathcal{J}^m$ and all $\bx\in\mathbf{L}^2(\bM),$ 
\begin{equation*}
    \lim_{J\rightarrow\infty}\|\bS_J[\pathvar]\bx - (\overline{\bS}_{\bu_0}[\pathvar]\bx)\bu_0\|_\bM=0,
\end{equation*}
where, on the left-hand side, we  assume that $J$ is large enough such that $\bS_J[\pathvar]$ is well-defined.
\end{proposition}
For a proof of Proposition \ref{prop: limit}, please see Section \ref{sec: the proof of proposition limi}. 
 
\subsection{Continuity and Conservation of Energy Properties}

The following theorem shows that the windowed scattering transform $\bS$ is  nonexpansive and the non-windowed scattering transform  $\overline{\bS}$ is Lipschitz continuous whenever the underlying frame $\cW$ is nonexpansive. 

\begin{theorem}[Nonexpansiveness]\label{thm: nonexpansive}
If $\cW$ is a frame with $C\leq 1$ in \eqref{eqn: frameAB}, then 
\begin{equation}\label{eqn: nonexpanisvewindow}
    \|\bS\mathbf{x}-\bS\mathbf{y}\|_{\bM}\leq  \|\mathbf{x}-\mathbf{y}\|_{\bM}\quad \text{for all }\bx,\by\in\mathbf{L}^2(\bM).
\end{equation}
Furthermore, if $\cW$ is either of the wavelet transforms $\cW^{(1)}_J$ or $\cW^{(2)}_J$ constructed in Section \ref{sec: wavelets},  $\bmu=\bu_0,$ and $\min_i|\bu_0(i)|>0,$ then we have 
\begin{equation}\label{eqn: nonexpanisvenonwindow}
    \|\overline{\bS}_{\bu_0}\mathbf{x}-\overline{\bS}_{\bu_0}\mathbf{y}\|_{2}\leq  \frac{1}{\sqrt{n}}\frac{\| \bM^{-1} \|_2}{\min_i|\bu_0(i)|}\|\mathbf{x}-\mathbf{y}\|_{\bM}.
\end{equation}
\end{theorem}

The proof of \eqref{eqn: nonexpanisvewindow} is similar to analogous results in e.g., \cite{mallat:scattering2012} and  \cite{zou:graphCNNScat2018}. Equation \eqref{eqn: nonexpanisvenonwindow} is proved by using Proposition \ref{prop: limit} to view $\overline{\bS}$ as the rescaled limit of $\bS_J$ as $J\rightarrow \infty$ and  applying Fatou's lemma. We note that the scaling factor of $\frac{1}{\sqrt{n}}$ is a consequence of the fact that we have assumed $\bu_0$ to have unit norm on $\mathbf{L}^2(\bM).$ A full proof is provided in Section \ref{sec: The proof of thm: non expansive}.
\begin{remark} The proof of \eqref{eqn: nonexpanisvenonwindow} relies on \eqref{eqn: nonexpanisvewindow}, Proposition \ref{prop: limit}, and Fatou's Lemma. Therefore, we are not able to establish this inequality for general  $\bmu$ and $\cW$, which do not satisfy the assumptions of  Proposition \ref{prop: limit}. However, if $\Phi$ is invertible one can use the relationship $\bU \bx=\Phi^{-1}\bS \bx$ to show $\overline{\bS}_{\bm{\mu}}$ is still Lipschitz continuous since 
\begin{equation*}
    \|\overline{\bS}_{\bmu}\mathbf{x}-\overline{\bS}_{\bmu}\mathbf{y}\|_{2}\leq    \|\bmu\|_\bM\|\Phi^{-1}\|_\bM\|\bS\mathbf{x}-\bS\mathbf{y}\|_{\bM}\leq \|\bmu\|_\bM\|\Phi^{-1}\|_\bM\|\mathbf{x}-\mathbf{y}\|_{\bM}.
\end{equation*} 
\end{remark}
The next theorem shows that if $\mathcal{W}$ is either of the  wavelet transforms constructed in Section \ref{sec: wavelets}, then $\bU$ experiences rapid energy decay. This implies that it is possible to obtain a good representation of an input signal $\mathbf{x}$ using only a few layers. Our arguments use ideas similar to the proof of Proposition 3.3 in \cite{zou:graphCNNScat2018}, with minor modifications to account for the fact that our wavelet constructions are different. Please see Section \ref{sec: proof of energy decay} for a complete proof.

\begin{theorem}[Energy Decay]\label{thm: energydecay}
Let $\cW$ be either of the wavelet transforms $\cW^{(1)}_J$ or $\cW^{(2)}_J$ constructed in Section \ref{sec: wavelets}. Then for all $\bx\in\mathbf{L}^2( \bM)$ and all $m\geq 1$,
\begin{equation}\label{eqn: ratio}
    \sum_{\pathvar\in\mathcal{J}^{m+1}}\|\bU[\pathvar]\mathbf{x}\|_{\bM}^2\leq \left(1-\frac{\bd_{\min}}{\|\bd\|_1}\right) \sum_{\pathvar\in\mathcal{J}^m}\|\bU[\pathvar]\mathbf{x}\|_{\bM}^2.
\end{equation}
Therefore, for all $m\geq 0,$
\begin{equation}\label{eqn: decay}
    \sum_{\pathvar\in\mathcal{J}^{m+1}}\|\bU[\pathvar]\mathbf{x}\|_{\bM}^2\leq \left(1-\frac{\bd_{\min}}{\|\bd\|_1}\right)^{m} \| \bx \|_\bM^2.
\end{equation}
\end{theorem}

The next theorem shows that if $\mathcal{W}=\mathcal{W}_J^{(1)},$ then the windowed graph scattering transform conserves energy on $\mathbf{L}^2(\bM).$ Its proof, is nearly identical to the proof of Theorem 3.1 in \cite{zou:graphCNNScat2018}. However, we give a full proof in Section \ref{sec: the proof of conservation of energy} for the sake of completeness.  

\begin{theorem}[Energy Conservation]\label{thm: conservation of energy}
If $\mathcal{W}=\mathcal{W}_J^{(1)}$, then  
$$\left\|\bS_J\bx\right\|_{\bM}=\|\bx\|_{\bM}\quad
\text{for all }\bx\in\mathbf{L}^2(\bM).$$
\end{theorem}

\subsection{Permutation Invariance and Equivariance}

In tasks such as graph classification, two graphs as equivalent if one is a permutation of the other. In this section, we will show that both $\bU$ and the windowed graph scattering transform  are equivariant with respect to permutations. As a consequence, we will show the non-windowed scattering transform is fully permutation invariant and the windowed-scattering transform, under certain assumptions, is permutation invariant up to a factor depending on the scale of the low-pass filter.
 
Let $S_n$ denote the permutation group on $n$ elements, and, for $\Pi\in S_n,$  let $\Pi(G)$ be the graph obtained by permuting the vertices of $G.$ 
If $G'=\Pi(G)$, we define $\bM',$  by
    $\bM' \coloneqq \Pi\bM\Pi^T$.
To motivate this definition, we note that if $\bM$ is the identity, then $\bM'$ is also the identity. Additionally, in the case where $\bM$ is the square-root degree matrix $\bD^{1/2}$, we note that the square-root degree matrix on $G'$ is given by
    $(\bD')^{1/2} = \Pi\bD^{1/2}\Pi^T$
and a similar formula holds when $\bM=\bD^{-1/2}.$ Therefore, in these three cases (which correspond to $\bK=\bT,\bP^T,$ and $\bP$, respectively, if $g=g_\star$) we may  view $\bM'$ as the analog of $\bM$ associated to $G'$. Similarly, we consider the analogs of $\cW$ and $\bmu$ on $G',$ given by
\begin{equation}\label{eqn: dWprime}
    \cW'\coloneqq\Pi\cW\Pi^T\coloneqq\{\Pi\Psi_j\Pi^T,\Pi\Phi\Pi^T\}_{j\in\cJ}\quad\text{and}\quad  \bmu' \coloneqq \Pi\bmu,
\end{equation} 
We also let $\bU',\bS'$ and $\overline{\bS}'$ denote  analogs of $\bU,\bS,$ and $\overline{\bS}$ on $G'$  constructed from $\cW'$ and $\bmu'.$

To further understand the definition of $\cW'$, we note that  the natural analog of $\bT$ on $G'$ is given by
 $   \bT' \coloneqq \Pi\bT\Pi^T.$
Therefore, Lemma \ref{lem: polynomialproperties} implies that  for any polynomial $p,$
\begin{align*}
    p\left((\bM')^{-1}\bT'\bM'\right) = (\bM')^{-1}p\left(\bT'\right)\bM' &= \left(\Pi\bM\Pi^T\right)^{-1}p\left(\Pi\bT\Pi^T\right)\left(\Pi\bM\Pi^T\right) \\ 
    &=\Pi \bM^{-1}p(\bT)\bM\Pi^T\\&=\Pi p\left(\bM^{-1}\bT\bM\right)\Pi^T
\end{align*}
with a similar formula holding for $q\coloneqq p^{1/2}.$ Thus, if $\cW$ is either of the wavelet transforms $\cW_J^{(1)}$ or $\cW_J^{(2)},$ then $\cW'$ is the analogous wavelet transform constructed from $\bK'\coloneqq(\bM')^{-1}\bT'\bM'.$ 
Given our definitions, it is now straightforward to prove the following equivariance theorem.

\begin{theorem}[Equivariance]\label{thm: equivariance}
Let $\Pi\in S_n$ be a permutation, let $G'=\Pi(G)$ and  let $\cW'$ be the wavelet transform on $G'$  defined as in \eqref{eqn: dWprime}. Then, for all $\bx\in\mathbf{L}^2(\bM)$,
\begin{equation*}
     \bU'\Pi\bx=\Pi\bU\bx\quad\text{and}\quad \bS'\Pi\bx=\Pi\bS\bx.
\end{equation*}
\end{theorem}
\noindent 
The following result shows that the nonwindowed scattering transform is permuation invariant.

\begin{theorem}[Invariance for the Nonwindowed Scattering Transform]\label{thm: permuationinvariancewindowed}
Let $\Pi\in S_n$ be a permutation, $G'=\Pi(G)$ and,  let $\cW'$ be the wavelet transform on $G'$  defined as in \eqref{eqn: dWprime}. Then
\begin{equation*}
    \overline{\bS}'\Pi\bx=\overline{\bS}\bx \quad\text{ for all }\bx\in\mathbf{L}^2(\bM).
\end{equation*}
\end{theorem}

For proofs of Theorems \ref{thm: equivariance} and \ref{thm: permuationinvariancewindowed}, please see Section \ref{sec: invariance proofs}.
Next, we will use Theorem \ref{thm: equivariance} to show that if $\cW$ is either $\cW_J^{(1)}$ or $\cW_J^{(2)}$ and $\bM = \bD^{1/2}$ then the windowed scattering transform is invariant on $\mathbf{L}^2(\bD^{1/2})$ up to a factor depending on the scale of the low-pass filter. We note that $0<\lambda_1<1.$ Therefore, $\lambda_1^t$ decays exponentially fast as $t\rightarrow\infty,$ and so if $J$ is large, the right hand side of equation \eqref{eqn: partial invariance} below will be nearly zero. We also recall that if our spectral function is given by  $g(t)=g_\star(t)$ then this choice of $\bM$ will imply that $\bK=\bP^T.$ 

\begin{figure}
	\centering
	\begin{subfigure}{0.4\textwidth} 
		\includegraphics[width=\textwidth]{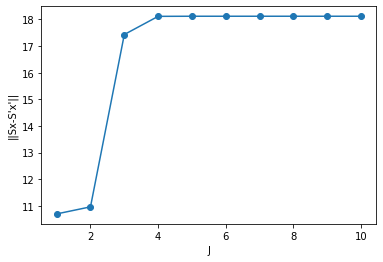}
		\caption{$\bK=\bT$} 
	\end{subfigure}
	\vspace{1em} 
	\begin{subfigure}{0.4\textwidth} 
		\includegraphics[width=\textwidth]{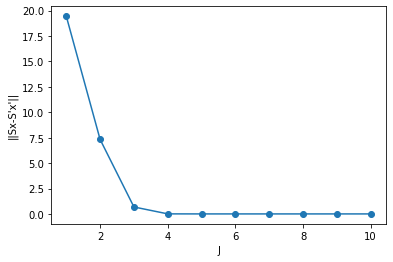}
		\caption{$\bK=\bP^T$} 
	\end{subfigure}
	\caption{Using $\cW_J^{(2)}$, we compare the windowed scattering coefficients $\bS_J \bx$ of a signal $\bx$ on a graph $G$ to the windowed scattering coefficients $\bS_J \Pi \bx$ of the permuted signal $\Pi\bx$ on the graph $G'=\Pi(G)$ for increasing $J$ and with $\bK$ equal to both $\bT$ and $\bP^T$. As we can see, the difference rapidly converges to zero when $\bK=\bP,$  but not when $\bK=\bT$. Here, we took $G$ to be an Erd\H{o}s-Reny random graph with $100$ vertices and the probability of connecting two vertices with an edge as $p=0.7$. The entries of $\bx$ are i.i.d. standard Gaussian random variables.}
	\label{fig: KPvKT}
\end{figure}

\begin{theorem}[Invariance for the Windowed Scattering Transform]\label{thm: permuatation invariance} 
Let $\cW$ be either $\cW^{(1)}_J$ or $\cW^{(2)}_J,$ 
let $\Pi\in S_n$, $G'=\Pi(G)$, and let $\cW'$ be the wavelet transform on $G'$ defined as in \eqref{eqn: dWprime}. Assume $\bM=\bD^{1/2}$ so $\bK=\bP^T$. Let $t=2^{J-1}$ if $\cW=\cW^{(1)}_J$ and $t=2^J$ if $\cW=\cW^{(2)}.$ Then, 
\begin{equation}\label{eqn: partial invariance}
    \|\bS'_J\Pi\bx-\bS_J\bx\|_{\bD^{1/2}} \leq \lambda_1^{t}\|\Pi-\bI\|_{\bD^{1/2}} \left(1 + \frac{\| \bd \|_1}{\bd_{\min}}\right)^{1/2} \|\bx\|_{\bD^{1/2}},\; \text{for all } \bx\in\mathbf{L}^2(\bD^{1/2}).
\end{equation}
\end{theorem}

For a proof of Theorem \ref{thm: permuatation invariance}, please see Section \ref{sec: Invariance Windowed}.
We note that while the term right-hand side of \eqref{eqn: partial invariance} does depend on the permutation $\Pi$, one may use the triangle inequlity to bound it uniformly by $\|\Pi-\mathbf{I}\|_{\mathbf{D}^{1/2}}\leq 1 + \sqrt{  \frac{\|\mathbf{d}\|_\infty}{\mathbf{d}_{\min}}  }.$ In particular, this implies that the right-hand side converges to zero as $J\rightarrow\infty.$ It is also important to note, as illustrated by Figure \ref{fig: KPvKT}, that the windowed scattering transform is not invariant for general $\bK$. Indeed, inspecting the proof, we see the invariance of the windowed scattering transform is a result of the invariance of the low-pass filter $\Phi_J.$ For large values of $J,$ this operator essentially projects the signal $\bx$ onto the bottom eigenvector $\bu_0.$ Thus, the invariance of the low-pass filter is a result of $\bu_0$ being a constant vector. Since $\bu_0=\bM^{-1}\bd^{1/2},$ this only occurs when $\bM=\bD^{1/2}.$ We also note that in \cite{zou:graphCNNScat2018}, the scattering transform is constructed from the unnormalized Laplacian, whose bottom eigenvector is always constant, which is key to the invariance result obtained there.

\section{Stability to Graph Perturbations}\label{sec: stability}

In this section let $G=(V,E,W)$ and $\widetilde G=(\widetilde V,\widetilde E ,\widetilde W)$ be  weighted, connected graphs with $|V|=|\widetilde V|=n,$ and let $\bM$ and $\widetilde \bM$  be invertible matrices. Throughout this section, for any object $X$ associated to $G$ or $\mathbf{L}^2(\bM),$ we will let $\widetilde X$ denote the analogous object on $\widetilde  G$ or $\mathbf{L}^2(\widetilde \bM)$, so e.g., $\widetilde \bd$ is the degree vector on $\widetilde G.$

Our analysis here is motivated by two problems in the machine learning literature on graphs. The first problem is graph classification and the second problem is graph alignment. In graph classification problems, one wishes to assign a label to a graph $G$. In this setting, one requires a representation of $G$ that is invariant to permutations of the vertex indices, since this operation does not change the underlying graph. Relatedly, graphs that are ``similar'' often have similar labels (or even the same label). Thus, the representation should also be stable (or at least continuous) to small perturbations of the graph $G$. On the other hand, one desires a representation that is able to distinguish between graphs that are ``dissimilar,'' indicating that it should retain as much information about $G$ as possible.

The graph alignment problem involves trying to find the ``best'' correspondence between the vertices of two graphs $G$ and $\widetilde{G}$ with  the same number of vertices. In this setting one requires a representation of each vertex in the graph, so that the representation of each vertex in $G$ may be compared to the representation of each vertex in $\widetilde{G}$. In this case, permutation invariance is not required, and in fact would be a detriment, as a permutation invariant representation would not be able to align a graph with a permutation of itself. Instead, one requires a permutation equivariant representation. Additionally, one may want the representation to be locally quasi-invariant, in the sense that permutations of small numbers of vertices (in practice located near each other in the graph) will not affect the representation too much. As with the graph classification problem, the representation should be stable to perturbations of the graph and retain as much information about the structure of the graph as possible. 

Following \cite{gama:diffScatGraphs2018}, as well as \cite{coifman:diffusionMaps2006} and \cite{nadler:dmDynamic2006}, our first measure of the distance between two graphs $G$ and $\widetilde{G}$ will be the  ``diffusion distances''  given by
    $\mathrm{dist} (\bT, \widetilde{\bT}) \coloneqq \|\bT-\widetilde \bT\|_2.$ 
This distance measure it not permutation invariant in the sense that if $\widetilde{G}$ is a permutation of $G$, then $\| \bT - \widetilde{\bT} \|_2$ is not zero in general. 
However, since $\mathbf{T}$ is a diffusion operator, the distance measure provided by $\| \mathbf{T} - \widetilde{\mathbf{T}} \|$ will be stable with respect to relatively small perturbations of the edge weights. Such a distance measure is useful for equivariant models, used in e.g., in graph alignment. 
For invariant problems, though, such as the graph classification task, if $\widetilde{G}$ is equal to $G$ up to a permutation of its vertex/edge indices, then $\widetilde{G}$ is the same graph as $G$. As such, the following permutation invariant graph distance is preferred:
\begin{equation*}
    \mathrm{dist}_{inva} (G, \widetilde{G}) = \min_{\substack{\Pi \in S_n \\ G' = \Pi (G)}} \| \mathbf{T}' - \widetilde{\mathbf{T}} \|_2 \, .
\end{equation*}

We will show the geometric wavelet transform, which is a permutation equivariant transform, is stable with respect to $\| \bT - \widetilde{\bT} \|_2$. We will then use this result to prove the non-windowed geometric scattering transform, which is a permutation invariant representation of the graph $G$, is stable with respect to $\mathrm{dist}_{inva} (G, \widetilde{G})$. Additionally, we show the windowed geometric scattering transform, which is an equivariant transform, is nevertheless stable with respect to $\mathrm{dist}_{inva} (G, \widetilde{G})$ plus a term that measures the size of the permutation required to align $G$ with $\widetilde{G}$ but which is dampened by the inverse of the scale of the low-pass filter. 


In order to carry out this analysis, we introduce additional terms which measure the difference between the matrices $\bM$ and $\widetilde \bM$. Specifically, we  let  
\begin{equation*}
    \bR_1 \coloneqq
    \bR_1(\bM,\widetilde{\bM}) \coloneqq\bM^{-1}\widetilde \bM\quad\text{and}\quad \bR_2 \coloneqq
    \bR_2(\bM,\widetilde{\bM})\coloneqq \widetilde \bM\bM^{-1},
\end{equation*}
and consider the quantities $\kappa(\bM,\widetilde \bM)$ and $R(\bM,\widetilde \bM)$ defined by 
\begin{align*}
    \kappa(\bM,\widetilde \bM)&\coloneqq \max_{i=1,2}\{\max\{\|\bI-\bR_i\|_{2},\|\bI-\bR_i^{-1}\|_{2}\}\}\quad\text{and},\\
    R(\bM,\widetilde \bM) &\coloneqq \max_{i=1,2}\{\max\{\|\bR_i\|_2,\|\bR_i^{-1}\|_2\}\}. 
\end{align*}
When we choose  $\bK$ to be  $\bP$ or $\bP^T$, we have $\bM=\bD^{\pm1/2}$ which implies 
     $\bR_1=\bR_2=\text{diag}\left(\frac{\widetilde{d_i}}{d_i}\right)^{\pm 1/2}$.
Therefore,  $\kappa(\bM,\widetilde \bM)$ measures how different the degree vectors $\bd$ and $\widetilde \bd$ are. More generally, if $\mathbf{M}$ and $\widetilde{\mathbf{M}}$ are diagonal matrices, then $\mathbf{L}^2(\mathbf{M})$ and $\mathbf{L}^2(\widetilde{\mathbf{M}})$ can be viewed as a weighted versions of the standard $\mathbf{L}^2$ space and and $\kappa(\mathbf{M},\widetilde{\mathbf{M}})$ and ${R}\kappa(\mathbf{M},\widetilde{\mathbf{M}})$ measure how different these weightings are.

We note that 
by construction we have $1\leq R(\bM,\widetilde \bM)\leq \kappa(\bM,\widetilde \bM)+1.$ Thus, if $\bM\approx\widetilde \bM$, we will have   $\kappa(\bM,\widetilde \bM)\approx 0$ and consequently $R(\bM,\widetilde \bM)\approx 1.$ We also note that we will have $\kappa(\bM,\widetilde \bM)=0$ and $R(\bM,\widetilde \bM)=1,$ if either  $\bM=\bI$ (so that $\bK=\bT$) or if $\bM=\bD^{\pm1/2}$ and the graphs $G$ and $\widetilde G$ have the same degree vector. The latter situation occurs if e.g. $G$  and $\widetilde G$ are regular graphs of the same degree. 
Furthermore, we note that if $\bM$ is diagonal, then $\bR_1=\bR_2.$ We only need two separate matrices $\bR_1,\bR_2$ when $\bM$ is not a diagonal matrix. However, in our prototypical examples we have $\bM=\bI$ or $\bM=\bD^{\pm 1/2},$ all of which are diagonal.
 
\subsection{Stability of the Wavelet Transforms}\label{sec: Wavelet Stabilty}

In this section, we  analyze the stability of the wavelets $\mathcal{W}_J^{(1)}$ and $\mathcal{W}_J^{(2)}$ constructed in Section \ref{sec: wavelets}.  Our first two theorems provide stability bounds for $\bK=\bT.$ These results will be extended to general $\bK$ by Theorem \ref{thm: transferTtoP}.  

\begin{theorem}[Stability of $\mathcal{W}_J^{(1)}$ with $\bK=\bT$] \label{thm: wavelet stability1} 
Suppose $G=(V,E,W)$ and $\widetilde G=(\widetilde V,\widetilde E,\widetilde W)$ are weighted, connected graphs with $|V|=|\widetilde V|=n.$ Let $\lambda_1^*=\max\{\lambda_1,\widetilde \lambda_1\},$ 
and let $\bM=\bI$ so that $\bK=\bT.$ Let $\mathcal{W}^{(1)}_J$ be the wavelets  constructed  from $\bT$ in Section \ref{sec: wavelets}, and let $\widetilde \cW^{(1)}_J$ be the corresponding wavelets on $\widetilde G$ constructed from $\widetilde \bT.$ Then there exists a constant $C_{\lambda^*_1},$ depending only on $\lambda_1^*$, such that  
\begin{align*}
  \|\cW^{(1)}_J-\widetilde{\cW}^{(1)}_J\|_2^2 
     &\leq C_{\lambda_1^*} 2^J n \|\bT-\widetilde{\bT}\|_2.
\end{align*}
\end{theorem}

The proof of Theorem \ref{thm: wavelet stability1} is in Section \ref{sec: proof of wavelet stability 1}. Our next result provides stability bounds for $\cW^{(2)}_J$ when $\bM=\bI$ (i.e.  $\bK=\bT$). The proof, given in Section \ref{sec: The Proof of wavelet 2 stability}, is closely modeled after the proofs of Lemmas 5.1 and 5.2 in \cite{gama:diffScatGraphs2018}. However, due to a small change in the derivation, our result appears in a slightly different form. 

\begin{theorem}[Stability of $\mathcal{W}_J^{(2)}$ with $\bK=\bT$]\label{thm: wavelet stability2}
Suppose $G=(V,E,W)$ and $\widetilde G=(\widetilde V,\widetilde E,\widetilde W)$ are weighted, connected graphs with $|V|=|\widetilde V|=n.$ Let $\lambda_1^*=\max\{\lambda_1,\widetilde \lambda_1\},$ 
and let $\bM=\bI$ so that $\bK=\bT.$  Let $\mathcal{W}^{(2)}_J$  be  the wavelets constructed  from $\bT$ in Section \ref{sec: wavelets}, let $\widetilde{\mathcal{W}}^{(2)}_J$ be the corresponding wavelets constructed from $\widetilde \bT.$ Then
\begin{equation*}
    \left\|\mathcal{W}_J^{(2)}- \widetilde{\mathcal{W}}_J^{(2)}\right\|_{2}^2 \leq 
    C_{\lambda_1^*}\|\bT-\widetilde \bT\|.
\end{equation*}
\end{theorem}

Comparing Theorems \ref{thm: wavelet stability1} and \ref{thm: wavelet stability2}, we note that while $\cW^{(1)}_J$ has the advantage of being a tight frame, $\cW^{(2)}_J$ has the advantage of possessing  stronger stability guarantees. Numerical experiments indicate that $\cW_J^{(1)}$ is indeed less stable to minor graph perturbations than $\cW_J^{(2)}$. In Figure \ref{fig: wavelet stability}, we plot the stability of both $\cW_J^{(1)}$ and $\cW_J^{(2)}$ on Erd\H{o}s-R\'{e}nyi, Watts-Strogatz (small world) and Barab\'{a}si-Albert (preferential attachment) random graphs. For each random graph $G,$ we obtained a perturbed graph $\widetilde{G}$ by adding mean-zero Guassian noise to each of the edge weights at noise level $\sigma=0.1$. For all three random graphs, the operator norm of $\cW^{(1)}_J-\widetilde{\cW}^{(1)}_J$ was significantly larger than $\cW^{(2)}_J-\widetilde{\cW}^{(2)}_J$. We also note the bounds in Theorem \ref{thm: wavelet stability1} increase exponentially in $J$, but we do not observe this behavior in practice. 

\begin{figure}[htp] 
\begin{subfigure}{0.41\textwidth}
\includegraphics[width=\linewidth]{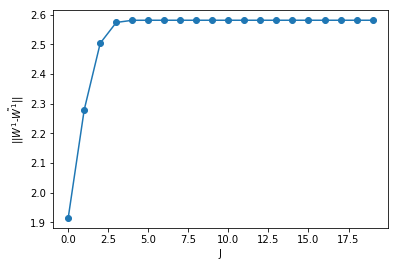}
\caption{$\mathcal{W}^{(1)}_J$ on Erd\H{o}s R\'{e}nyi} \label{fig:a}
\end{subfigure}\hspace*{\fill}
\begin{subfigure}{0.41\textwidth}
\includegraphics[width=\linewidth]{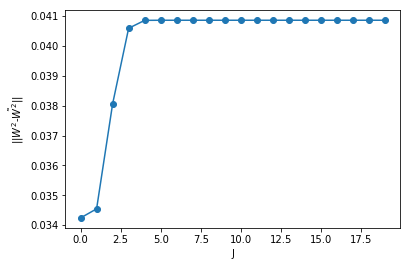}
\caption{$\mathcal{W}^{(2)}_J$ on Erd\H{o}s R\'{e}nyi} \label{fig:b}
\end{subfigure}

\medskip
\begin{subfigure}{0.41\textwidth}
\includegraphics[width=\linewidth]{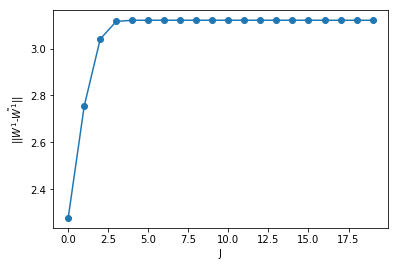}
\caption{$\mathcal{W}^{(1)}_J$ on Watts Strogatz} \label{fig:c}
\end{subfigure}\hspace*{\fill}
\begin{subfigure}{0.41\textwidth}
\includegraphics[width=\linewidth]{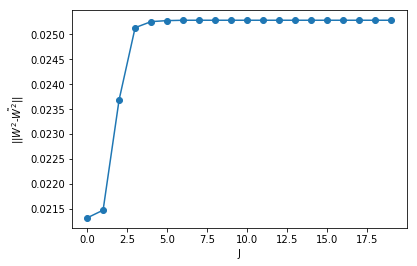}
\caption{$\mathcal{W}^{(2)}_J$ on Watts Strogatz} \label{fig:d}
\end{subfigure}

\medskip
\begin{subfigure}{0.41\textwidth}
\includegraphics[width=\linewidth]{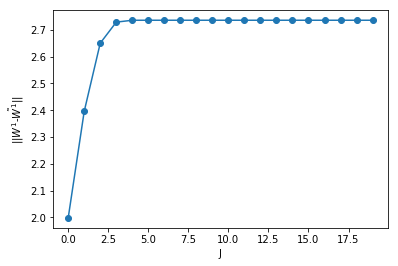}
\caption{$\mathcal{W}^{(1)}_J$ on Barab\'{a}si Albert} \label{fig:e}
\end{subfigure}\hspace*{\fill}
\begin{subfigure}{0.41\textwidth}
\includegraphics[width=\linewidth]{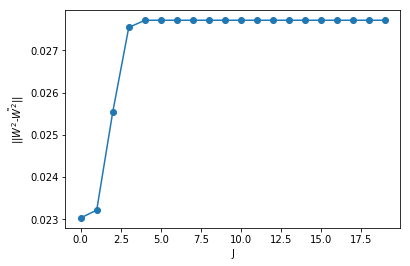}
\caption{$\mathcal{W}^{(2)}_J$ on Barab\'{a}si Albert} \label{fig:f}
\end{subfigure}

\caption{We plot the operator norm of $\mathcal{W}_J^{(1)}-\widetilde{W}_J^{(1)}$ and $\mathcal{W}_J^{(2)}-\widetilde{W}_J^{(2)}$ as a function of $J$  on Erd\H{o}s-R\'{e}nyi ($n=100, p=.1$), Watts-Strogatz ($n=100, k=20, p=.1$) and Barab\'{a}si-Albert ($n=100, m=10$) random graphs.  The perturbed graphs $\widetilde{G}$ were obtained by adding Gaussian noise ($\sigma=.1$) to the edge weights of $G$. Observe that for all three random graphs the wavelets $\mathcal{W}^{(2)}_J$ are more stable than $\mathcal{W}^{(1)}.$ However, it appears that the operator norm of both $\mathcal{W}_J^{(1)}-\widetilde{W}_J^{(1)}$ and $\mathcal{W}^{(2)}-\widetilde{W}_J^{(2)}$ can be taken to be independent of $J.$ In these experiments, we chose $\bK=\bT$ and did not include the low-pass filter $\Phi_J$ in our wavelets. \label{fig: wavelet stability}}
\end{figure}

Theorems \ref{thm: wavelet stability1} and \ref{thm: wavelet stability2} show that  the wavelets $\cW^{(1)}_J$ and $\cW^{(2)}_J$ are stable on $\mathbf{L}^2$ in the special case that $\bK=\bT.$ Our next theorem extends this analysis to general $\bK$ and to more general functions of $\bK$. In particular, it can be applied to any situation where $\cW^{\bT} = \{r_j(\bT)\}_{j\in\cJ}$ and $\cW^{\widetilde{\bT}} = \{r_j(\widetilde \bT)\}_{j\in\cJ}$ form frames on the unweighted $\mathbf{L}^2$ space, where $\cJ$ is some indexing set, and  each of the functions $r_j$ is a polynomial or the square root of a polynomial. We note that in the case where $\bM$ is close to $\widetilde{\bM}$ we have $\kappa(\bM,\widetilde{\bM})\approx 0.$ Therefore, Theorem \ref{thm: transferTtoP} will imply 
$\left\|\cW^{\bK}-\cW^{\widetilde \bK}\right\|_{\bM}^2 \lesssim 6\left\|\cW^{\bT}-\cW^{\widetilde \bT}\right\|_{2}^2$. For a proof, please see Section \ref{sec: proof of general K}.

\begin{theorem}[Wavelet Stability for General $\bK$]\label{thm: transferTtoP}
Suppose $G=(V,E,W)$ and $\widetilde G=(\widetilde V,\widetilde E,\widetilde W)$ are weighted, connected graphs with $|V|=|\widetilde V|=n,$ and let $\bM$ and $\widetilde \bM$ be invertible matrices. Let $\mathcal{J}$ be an indexing set,  and for $j\in\cJ,$ let $r_j(\cdot)$ be either a polynomial or the square root of a polynomial. 
Suppose that $\mathcal{W}^{\bT}=\{r_j(\bT)\}_{j\in\mathcal{J}}$ and $\mathcal{W}^{\widetilde \bT} =\{r_j(\widetilde \bT)\}_{j\in\mathcal{J}}$ are frames  on the unweighted $\mathbf{L}^2$ space   with $C\leq 1$ in \eqref{eqn: frameAB} for both $\cW^{\bT}$ and $\cW^{\widetilde \bT}.$  
Let $\bK=\bM^{-1}\bT\bM$ and let $\cW^{\bK}$ and $\cW^{\widetilde \bK}$ be the frames defined by   $\{r_j(\bK)\}_{j\in\mathcal{J}}$ and $\{r_j(\widetilde \bK)\}_{j\in\mathcal{J}}.$ 
Then,
\begin{equation*}
    \left\|\cW^{\bK}-\cW^{\widetilde \bK}\right\|_{\bM}^2 \leq 6\left(\left\|\cW^{\bT}-\cW^{\widetilde \bT}\right\|_{2}^2+\kappa(\bM,\widetilde \bM)^2\big(\kappa(\bM,\widetilde \bM\big)+1)^2 \right).
\end{equation*}
\end{theorem}

The following corollary is an immediate consequence of Theorem \ref{thm: transferTtoP} combined with  Theorems \ref{thm: wavelet stability1} and \ref{thm: wavelet stability2}. As noted prior to the statement of Theorem \ref{thm: transferTtoP}, in the case where $\bM$ is close to $\widetilde{\bM},$ the terms involving $\kappa(\bM,\widetilde{\bM})$ will be small.

\begin{corollary}[Stability of $\mathcal{W}_J^{(1)}$ and  $\mathcal{W}_J^{(2)}$ with General $\bK$]\label{cor: stabilityK1}
Suppose $G=(V,E,W)$ and $\widetilde G=(\widetilde V,\widetilde E,\widetilde W)$ are weighted, connected graphs with  $|V|=|\widetilde V|=n.$  Let $\bM$ and $\widetilde \bM$ be invertible matrices, and let $\lambda_1^*=\max\{\lambda_1,\widetilde \lambda_1\}.$ 
Let $\mathcal{W}^{(1)}_J$ and $\mathcal{W}^{(2)}_J$  be the wavelet transforms  constructed  from $\bK$ in Section \ref{sec: wavelets}, and let $\widetilde{\mathcal{W}}^{(1)}_J$ and $\widetilde{\mathcal{W}}^{(2)}_J$ be the corresponding wavelet transforms constructed from $\widetilde\bK.$ Then,
\begin{align*}
    \left\|\mathcal{W}^{(1)}_J- \widetilde{\mathcal{W}}^{(1)}_J\right\|_{2}^2 &\leq  
    C_{\lambda_1^*}\left(2^J n\|\bT-\widetilde{\bT}\| +\kappa(\bM,\widetilde{\bM})^2(\kappa(\bM,\widetilde{\bM})+1)^2\right),\:\text{and}\\
    \left\|\mathcal{W}^{(2)}_J- \widetilde{\mathcal{W}}^{(2)}_J\right\|_2^2 &\leq  
    C_{\lambda_1^*}\left(\|\bT-\widetilde{\bT}\|_2+\kappa(\bM,\widetilde \bM)^2(\kappa(\bM,\widetilde \bM)+1)^2\right).
\end{align*}
\end{corollary}
\begin{remark}
Inspecting Corollary \ref{cor: stabilityK1}, we see that our wavelets may become less stable as $\kappa(\bM,\widetilde{\bM})$ increases. In particular, in the case where $\bM=\bD^{\alpha}$ and $\widetilde{\bM}=\widetilde{\bD}^{\alpha}$  we have 
\begin{equation*}
    \kappa(\bM,\widetilde{\bM})=\max\left\{\left|1-\left(\frac{d_i}{\tilde{d}_i}\right)^\alpha\right|,\left|1-\left(\frac{\tilde{d}_i}{d_i}\right)^\alpha\right|\right\}.
\end{equation*}
Therefore, values of $\alpha$ closer to zero lead to tighter stability bounds.
\end{remark}

One might also wish to replace Corollary  \ref{cor: stabilityK1} with an inequality written in terms of $\|\bK-\widetilde \bK\|_\bM$ rather than $\|\bT-\widetilde \bT\|_2.$ This can be done by the following proposition. Recall that if $\bM\approx\widetilde{\bM}$ then
$\kappa(\bM,\widetilde \bM) \approx0$ and $R(\bM,\widetilde \bM)\approx1.$ Therefore, \eqref{eqn: TdisttoPdist} implies 
$\|\bT-\widetilde \bT\|_2\lesssim \|\bK-\widetilde \bK\|_\bM.$ For a proof of Proposition 
\ref{prop: Tdist to Pdist}, please see Section \ref{sec: proofofTtoK}.
\begin{proposition}
\label{prop: Tdist to Pdist}
Suppose $\:G=(V,E,W)$ and $\widetilde G=(\widetilde V,\widetilde E,\widetilde W)$ are are weighted, connected graphs with  $|V|=|\widetilde V|=n.$ Then,
\begin{equation}\label{eqn: TdisttoPdist}
    \|\bT-\widetilde \bT\|_2\leq\kappa(\bM,\widetilde \bM)\left(1+R(\bM,\widetilde \bM)^{3}\right) + R(\bM,\widetilde \bM)\|\bK-\widetilde \bK\|_\bM.
\end{equation}
\end{proposition}

\subsection{Stability of the Geometric Scattering Transform}
\label{sec: scattering stability}

In this section, we will prove  stability bounds for the windowed and non-windowed scattering transform. We will state these results in terms of the stability bound of the underlying wavelet transform and also the upper frame bound of  $\widetilde{\cW}$ when considered as on operator on $\mathbf{L}^2(\bM).$   We do this both to emphasize that the stability of the scattering transform is a consequence of the stability of the underlying frame and so our result can be applied to scattering transforms built onto of other graph wavelet constructions. We will assume that $G=(V,E,W)$ and $\widetilde G=(\widetilde V,\widetilde E,\widetilde W)$ are weighted, connected graphs such that $|V|=|\widetilde V|=n,$ let $\bM$ and $\widetilde \bM$ be $n\times n$ invertible matrices, and assume that $\mathcal{W}=\{\Psi_j,\Phi\}_{j\in\mathcal{J}}$ and $\widetilde{\mathcal{W}}=\{\widetilde{\Psi}_j,\widetilde{\Phi}\}_{j\in\mathcal{J}}$ are frames on $\mathbf{L}^2(\bM)$ and $\mathbf{L}^2(\widetilde \bM)$ such that $C\leq 1$ in \eqref{eqn: frameAB}. 
For a family of matrices $\Gamma = (\bB_i)_{i \in \mathcal{I}}$, where $\mathcal{I}$ is a countable index set, we define the operator norm of $\Gamma : \bL^2 (G, \bM) \rightarrow \bm{\ell}^2 (\bL^2 (G, \bM))$ as
 $    \| \Gamma \|_{\bM} \coloneqq \sup_{\| x \|_{\bM} = 1} \| \Gamma \bx \|_{\bM} \, .
$
If $\Pi$ is a permutation and  $\widetilde{G}'=\Pi(\widetilde{G}),$ then
we will let $\widetilde{\cW}'\coloneqq\Pi\widetilde{\mathcal{W}}\Pi^T= \{\Pi\widetilde{\Psi}_j\Pi^T,\Pi\widetilde{\Phi}\Pi^T\}_{j\in\mathcal{J}}$ denote the corresponding permuted wavelet frame. 


Theorem \ref{thm: windowstability} provides stability guarantees for the windowed  scattering transform with bounds that 
are functions of 
$\|\cW- \widetilde{\cW}\|_{\bM}$ and $\|\widetilde{\cW}\|_{\bM}$.  Similarly, Theorem \ref{thm: scattering stability no window} guarantees stability for the nonwindowed transform. 
By Theorems \ref{thm: wavelet stability1}, \ref{thm: wavelet stability2}, and \ref{thm: transferTtoP} as well as Proposition \ref{thm: Cstability}, these results imply that the scattering transforms constructed  from $\cW^{(1)}_J$ or $\cW^{(2)}_J$ are stable in the sense  that if $G$ is similar to $\widetilde G$ and $\bM$ is similar to $\widetilde \bM$,  then the scattering transforms $\bS$ and $\widetilde \bS$ will produce similar representations of an inputted signal $\bx.$ Many of the ideas in the proof of Theorems \ref{thm: windowstability} and \ref{thm: scattering stability no window} are similar to those used to prove Theorem 5.3 in \cite{gama:diffScatGraphs2018}. The primary difference is Lemma \ref{lem: nonexpansiveU} which is needed because $\widetilde \cW$ is  not (in general) a non-expansive frame on $\mathbf{L}^2(\bM),$ and therefore our results will involve terms related to the operator norm of $\widetilde{\mathcal{W}}$ on $\mathbf{L}^2(\bM),$ which can be controlled by Proposition \ref{thm: Cstability}. We also note that both Theorem \ref{thm: windowstability} and \ref{thm: scattering stability no window} consider pertubations to both the input signal $\mathbf{x}$ an to the graph structure whereas Theorem 5.3 of \cite{gama:diffScatGraphs2018} only considered perturbations to the graph structure.

\begin{theorem}[Stability for the Windowed Scattering Transform]\label{thm: windowstability}

Let $G=(V,E,W)$ and $\widetilde G=(\widetilde V,\widetilde E,\widetilde W)$ be weighted, connected graphs with $|V|=|\widetilde V|=n,$  let $\bM$ and $\widetilde \bM$ be invertible $n\times n$ matrices, 
and let $\mathcal{J}$ be an indexing set. Let  $\mathcal{W}=\{\Psi_j,\Phi\}_{j\in\mathcal{J}}$ and  $\widetilde{\mathcal{W}}=\{\widetilde{\Psi}_j,\widetilde{\Phi}\}_{j\in\mathcal{J}}$ be frames on $\mathbf{L}^2(\bM)$ and $\mathbf{L}^2(\widetilde{\bM}),$ such that $C\leq 1$ in \eqref{eqn: frameAB}. Let $\bS^\ell$ and  $\widetilde{\bS^\ell}$ be the $\ell$-th layers of the windowed scattering transforms on $G$ and $\widetilde{G}$.  
Further assume that 
$\bS^{\ell}$ is approximately permutation invariant up to a factor of $\cB$ in the sense that
\begin{equation}\label{eqn: permutation assumption}
\left\|\Pi\bS^\ell\bx-\bS^\ell\bx\right\|_{\bM}\leq \mathcal{B}\|\bx\|_{\bM}
\end{equation}
for all $\bx\in\mathbf{L}^2(\bM)$ and $\Pi\in S_n$. Then for all $\bx\in\mathbf{L}^2(\bM)$ and $\widetilde{\bx}\in\mathbf{L}^2(\widetilde{\bM})$
\begin{align}\label{eqn: stability bound for windowed scattering transform}
    &\left\|\bS^\ell\bx-\widetilde{\bS^\ell}\widetilde{\bx}\right\|_{2}\\\leq&     R(\bM,\bI)^2\inf_{\substack{\Pi\in S_n \\ G'=\Pi(G)}}\left(\mathcal{B}\|\bx\|_2 + \|\Pi\bx-\widetilde{\bx}\|_2 + \sqrt{2} \|\cW'-\widetilde{\cW}\|_{\bM'}\left(\sum_{k=0}^\ell\|\widetilde{\cW}\|_{\bM'}^{k}\right)\|\widetilde{\bx}\|_{2}\right).\nonumber
\end{align}

\end{theorem}

\begin{remark}
We can interpret each term in the right hand side of \eqref{eqn: stability bound for windowed scattering transform}. The first term is a direct result of the approximate permutation invariance of the windowed scattering operator. The second term measures the best possible alignment of the input signals $\bx$ and $\widetilde{\bx}$, whereas the third term measures the best possible alignment of the wavelet operators on $G$ and $\widetilde{G}$. Note that if there is a permutation $\Pi$ such that  $\widetilde{G} = \Pi (G)$  and $\widetilde{\bx} = \Pi \bx$, then the only nonzero term is $\mathcal{B} \| \bx \|_2$. This result is more general than an approximate invariance result, though, since it also characterizes the stability of the windowed scattering transform in terms of the stability of the wavelet operators and the similarity of the input signals $\bx$ and $\widetilde{\bx}$ (which are often functions of the graphs $G$ and $\widetilde{G}$, respectively). 
\end{remark}

In order to prove Theorem \ref{thm: windowstability}, we  need the following Lemma (proved in Section \ref{sec: the proof of lem: window stability}).

\begin{lemma}\label{lem: windowstability}
Under the assumptions of Theorem \ref{thm: windowstability},  we have
\begin{equation}\label{eqn: scatteringstabilitynopermwindow}
    \left\|\bS^\ell\bx-\widetilde{\bS^\ell}\bx\right\|_{\bM}\leq  \sqrt{2}\|\cW-\widetilde{\cW}\|_{\bM}\left(\sum_{k=0}^\ell\|\widetilde{\cW}\|_{\bM}^{k}\right)\|\bx\|_\bM\quad\text{for all } \bx\in\mathbf{L}^2(\bM).
\end{equation}  

\end{lemma} 

\begin{proof}[The Proof of Theorem \ref{thm: windowstability}] Let $\Pi\in S_n$ be a permutation and let $G'=\Pi(G)$.
By Theorem \ref{thm: equivariance}, we have $\Pi\bS^\ell=(\bS^\ell)'\Pi$. 
Therefore, the triangle inequality implies 
\begin{equation}\label{eqn: threeparts}
    \left\|\bS^\ell\bx-\widetilde{\bS^\ell}\widetilde{\bx}\right\|_{2}\leq
    \left\|\bS^\ell\bx-\Pi\bS^\ell\bx\|_2+\|(\bS^\ell)'\Pi\bx-(\bS^\ell)'\widetilde{\bx}\|_2+\|(\bS^\ell)'\widetilde{\bx}-\widetilde{\bS^\ell}\widetilde{\bx}\right\|_{2}.
\end{equation} The assumption \eqref{eqn: permutation assumption} implies that 
\begin{equation}
    \|\bS^\ell\bx-\Pi\bS^\ell\bx\|_2\leq R(\bM,\bI) \|\bS^\ell\bx-\Pi\bS^\ell\bx\|_\bM \leq \mathcal{B}R(\bM,\bI) \|\mathbf{x}\|_\bM\leq \mathcal{B}R(\bM,\bI)^2 \|\mathbf{x}\|_2.\label{eqn: assumption term}
\end{equation}
Similarly, by Theorem \ref{thm: nonexpansive}, we have that 
\begin{equation}\label{eqn: nonexpansiveterm}
    \|(\bS^\ell)'\Pi\bx-(\bS^\ell)'\widetilde{\bx}\|_2 
    \leq R(\bM',\bI)^2   \|\Pi\bx-\widetilde{\bx}\|_2,
\end{equation}
and applying Lemma \ref{lem: windowstability} yields
\begin{align}
\|(\bS^\ell)'\widetilde{\bx}-\widetilde{\bS^\ell}\widetilde{\bx}\|_{2}
&\leq\sqrt{2}R(\bM',\bI)^2 \|\cW'-\widetilde{\cW}\|_{\bM'}\left(\sum_{k=0}^\ell\|\widetilde{\cW}\|_{\bM'}^{k}\right)\|\widetilde{\bx}\|_{2}.\label{eqn: fourpointtwoterm}
\end{align}
Since $G'=\Pi(G),$ one may check that $R(\bM,\bI)=\max\{\|\bM\|_2,\|\bM^{-1}\|_2\}=R(\bM',\bI)$.
Therefore, combining \eqref{eqn: threeparts} with 
 \eqref{eqn: assumption term},\eqref{eqn: nonexpansiveterm}, and \eqref{eqn: fourpointtwoterm} yields
\begin{equation*}
    \left\|\bS^\ell\bx-\widetilde{\bS^\ell}\widetilde{\bx}\right\|_{2}\leq    R(\bM,I)^2\left(\mathcal{B}\|\bx\|_2 + \|\Pi\bx-\widetilde{\bx}\|_2 + \sqrt{2} \|\cW'-\widetilde{\cW}\|_{\bM'}\left(\sum_{k=0}^\ell\|\widetilde{\cW}\|_{\bM'}^{k}\right)\|\widetilde{\bx}\|_{2}\right),
\end{equation*}
and so infimizing over $\Pi$ completes the proof.
\end{proof}

The following proposition shows that if $\bM$ is close to  $\widetilde \bM$, then an upper frame bound for $\cW$ can be used to produce an upper frame bound for $\widetilde \cW$ on $\mathbf{L}^2(\bM)$ provided the wavelet operators satisfy certain conditions. This result 
yields a more transparent upper bound in the stability results of Theorem \ref{thm: windowstability} for the types of wavelet operators covered by the proposition. More specifically, combining Proposition \ref{thm: Cstability} with Theorem \ref{thm: windowstability} allows one to refine \eqref{eqn: stability bound for windowed scattering transform} to:
\begin{align*}
    &\left\|\bS^\ell\bx-\widetilde{\bS^\ell}\widetilde{\bx}\right\|_{2}\\\leq&     R(\bM,\bI)^2\inf_{\substack{\Pi\in S_n \\ G'=\Pi(G)}}\left(\mathcal{B}\|\bx\|_2 + \|\Pi\bx-\widetilde{\bx}\|_2 + \sqrt{2} \|\cW'-\widetilde{\cW}\|_{\bM'}\left(\sum_{k=0}^\ell R (\bM, \widetilde{\bM})^{2k}\right)\|\widetilde{\bx}\|_{2}\right).\nonumber
\end{align*}
 We note that this result can be applied to both $\cW_J^{(1)}$ and $\cW_J^{(2)}.$ For a proof, see  Section \ref{sec: proofCStability}. 
 
\begin{proposition}\label{thm: Cstability}
Suppose $G=(V,E,W)$ and $\widetilde G=(\widetilde V,\widetilde E,\widetilde W)$ are weighed, connected graphs with $|V|=|\widetilde V|=n.$ Let $\bM$ and $\widetilde \bM$ be invertible $n\times n$ matrices and let $\bK=\bM^{-1}\bT\bM,  \widetilde\bK=\widetilde\bM^{-1}\widetilde\bT\widetilde\bM$. Let $\mathcal{J}$ be an indexing set, and for $j\in\cJ,$ let $r_j(\cdot)$ be either a polynomial or the square root of a polynomial. 
Suppose that $\mathcal{W}=\{r_j(\bK)\}_{j\in\mathcal{J}}$ is a frame on $\mathbf{L}^2(\bM)$ with $C\leq 1$ in \eqref{eqn: frameAB}, and similarly assume that $\widetilde{\mathcal{W}} = \{r_j(\widetilde{\bK})\}_{j\in\mathcal{J}}$ is a frame on $\mathbf{L}^2(\widetilde{\bM})$ also with $C\leq 1$ in \eqref{eqn: frameAB}. Then
$\widetilde \cW$ is a bounded operator on $\mathbf{L}^2(\bM)$ and
\begin{equation*}
    \| \widetilde{\cW} \bx \|_{\bM}^2 =  \sum_{j\in\cJ}\|r_j(\widetilde{\bK})\bx\|_\bM^2\leq R(\bM,\widetilde \bM)^4\|\bx\|^2_\bM.
\end{equation*}
\end{proposition}

The following is the analogue of Theorem \ref{thm: windowstability} for the non-windowed scattering transform. For a proof, please see Section \ref{sec: proof of scattering stability no window}.

\begin{theorem}[Stability for the Non-windowed Scattering Transform]\label{thm: scattering stability no window}

Let $G=(V,E,W)$ and $\widetilde G=(\widetilde V,\widetilde E,\widetilde W)$ be weighted, connected graphs with $|V|=|\widetilde V|=n,$  let $\bM$ and $\widetilde \bM$ be invertible $n\times n$ matrices, 
and let $\mathcal{J}$ be an indexing set. Let  $\mathcal{W}=\{\Psi_j,\Phi\}_{j\in\mathcal{J}}$ and  $\widetilde{\mathcal{W}}=\{\widetilde{\Psi}_j,\widetilde{\Phi}\}_{j\in\mathcal{J}}$ be frames on $\mathbf{L}^2(\bM)$ and $\mathbf{L}^2(\widetilde{\bM}),$ such that $C\leq 1$ in \eqref{eqn: frameAB}. Let $\bS^\ell$ and  $\widetilde{\bS^\ell}$ be the $\ell$-th layers of the windowed scattering transforms on $G$ and $\widetilde{G}$. Further assume that $\bmu=\bu_0$ and $\widetilde{\bmu}=\widetilde{\bu}_0$. Then,
\begin{align}
\left\|\overline{\bS^\ell}\bx-\widetilde{\overline{\bS^\ell}}\widetilde{\bx}\right\|_{2}^2 &\leq 3\inf_{\Pi\in S_n}\Bigg( \frac{2 R(\bM, \mathbf{I})^4}{n\min_i |\bu_0(i)|^2 }\|\Pi \mathbf{x}-\widetilde{\mathbf{x}}\|^2_2 \label{eqn: stability non-windowed main result}\\
&\quad\quad\quad\quad\quad + 2 R(\bM, \mathbf{I})^2 \|\bmu\|_{\bM}^2 \left(\sum_{k=0}^{\ell-1}\|\widetilde{W}\|_{\bM'}^k \right)^2 \|\mathcal{W}'-\widetilde{\mathcal{W}}\|_{\bM'}^2 \|\tilde{\bx}\|^2_2 \nonumber \\
&\quad\quad\quad\quad\quad+ R(\bM',\widetilde{\bM})^2 \|\bm{\mu}'-\tilde{\bm{\mu}}\|_{\bM'}^2 \|\widetilde{\bx}\|^2_{\widetilde{\bM}} \nonumber \\
&\quad\quad\quad\quad\quad+ R(\widetilde{\bM}, \mathbf{I})^4 R(\bM,\bI)^2 (1+\bR(\bM',\widetilde{\bM}))\|\widetilde{\bm{\mu}}\|_{\widetilde{\bM}}^2 \kappa(\bM',\widetilde{\bM}) \|\widetilde{\bx}\|_{\widetilde{\bM}}\Bigg). \nonumber
\end{align}
\end{theorem}

\begin{remark}
As in Theorem \ref{thm: windowstability}, we can interpret each term in the right hand side of \eqref{eqn: stability non-windowed main result}. In fact, the first two terms above are very similar to the second and third terms of \eqref{eqn: stability bound for windowed scattering transform} and have the same interpretation. The third term in \eqref{eqn: stability non-windowed main result} measures the best possible alignment of the measures $\bm{\mu}$ and $\widetilde{\bm{\mu}}$, while the last term measures the best possible alignment of $\bM$ and $\widetilde{\bM}$ through $\kappa (\bM', \widetilde{\bM})$. While $\bx$, $\bM$, and $\bm{\mu}$ can be chosen independently of $G$, often they depend on $G$ and thus it is reasonable to assume that a small perturbation of $G$, reflected in $\widetilde{G}$, would lead to a small change in these quantities. 

Furthermore, when interpreting the right hand side of the above equation, one should keep in mind that if there exists a permutation $\Pi$ such that $\bM' = \Pi \bM \Pi^T$ is close to $\widetilde{\bM}$, we will have $\bR(\bM', \widetilde{\bM})\approx 1$ and $\kappa(\bM', \widetilde{\bM})\approx 0$. In our canonical examples, we either have $\bM=\bI$ or $\bM=\bD^{\pm1/2}$ which implies that $R(\bM,\bI)\leq \|\bd\|^{1/2}_\infty.$ Moreover, if one assumes that $\bM=\widetilde{\bM}=\bM'=\bI$ and $\bmu=\widetilde{\bmu}=\bmu'=\frac{1}{\sqrt{n}}\mathbf{1}$ to recover a setup similar to \cite{gama:diffScatGraphs2018}, then we have $\kappa(\bM',\widetilde{\bM})=0$ and $\bK=\bT$ which implies $\bu_0=\frac{\bd^{1/2}}{\|\bd^{1/2}\|_2}$, and so the above result simplifies to
\begin{align*}
\left\|\overline{\bS^\ell}\bx-\widetilde{\overline{\bS^\ell}}\widetilde{\bx}\right\|_{2}^2&\leq C_G\inf_{\Pi\in S_n}\left(\frac{\|\bd\|_\infty^2\|\bd\|_1}{\min_i|\bd(i)|^2}\|\Pi\mathbf{x}-\widetilde{\mathbf{x}}\|^2_{2} + \ell^2\|\mathcal{W}'-\widetilde{\mathcal{W}}\|^2_{2'} \|\tilde{\bx}\|^2_{2}\right). 
\end{align*}
Thus, setting $\Pi\bx=\widetilde{\bx}$, this result essentially includes Corollary 5.4 of \cite{gama:diffScatGraphs2018} as a special case.

\end{remark}


\section{Scattering view of graph neural networks}
\label{sec: GNNS}

The main contributions of the theoretical framework established here 
are twofold. First, it  provides an overarching characterization of geometric descriptors captured by scattering features, which have been shown effective both at node level and whole-graph level tasks. Second, it provides a mathematical foundation for studying a wide range of graph neural networks (GNNs), analogous to the role played by the Euclidean scattering transform.  To further establish this second aspect, we  review several prominent GNNs and discuss the relationship between them and the scattering transform. 

\subsection{Aggregate-Transform architectures}
Most GNN architectures  can be modeled as alternating between two fundamental types of operations:
i)
    \textbf{Aggregate:} For each signal $\mathbf{x},$  
    replace $\bx(i)$ with a value aggregated from its neighbors, e.g., $\mathbf{x}(i) \leftarrow \sum_{j\in\mathcal{N}_i} w_{ij} \bx(j)$, where $\mathcal{N}_i$ a local neighborhood of node $i$.  Different architectures vary in their definition of the $\mathcal{N}_i$ and the weights $w_{ij}.$ 
    However, a common theme is that this operation is applied separately on each graph signal. 
    ii) \textbf{Transform:} Between aggregation steps, GNNs apply a node-wise transformation on the features of each node, typically in the form of a shallow neural network. 

Together, these two operations extend the ``convolutions'' used in CNNs over images. We note that the typical implementation of such convolutions differs from the traditional mathematical definition of a convolution by not only considering the sliding-window application of a local filter 
(essentially equivalent to the aggregate step), but also an operation that mixes together resulting filter outputs (essentially equivalent to the transform step).
On images the grid organization of pixels naturally lends itself to learnable filters that can easily translated over the image. However, in graph domains there is no such natural notion of translation. Therefore,  graph neural networks are limited in the aggregations they can use and most of them default to employing predetermined aggregations, somewhat analogous to the use of predefined wavelet filters in the original scattering transform. This makes the scattering framework discussed here particularly suitable for providing solid foundation for not only understanding, but also constructing graph neural networks. Indeed, in the Euclidean case, the scattering transform deviates from CNNs both by removing the channel-mixing (transform) operation, and by replacing learned filters with  handcrafted ones. Here, however, 
the only deviation comes from the  exclusion of channel-mixing operations in the transform step. Moreover, these channel-mixing operations can be added to the scattering transform in hybrid architectures such as~\cite{min2020scattering,min2021gsan}. To further emphasize this perspective, we will discuss common implementations of the aggregation step in GNNs and  how they relate to the scattering framework.

\subsection{Message-passing graph neural networks} 

As a representative example of message-passing aggregation, we  consider the operation used in Graph Isomorphism Network (GIN) \cite{Xu2018}. Let $\bX : V \rightarrow \mathbb{R}^m$ be a vector valued-function on the vertices; which may be thought of as either a feature vector for each vertex or  an $\bX$ as an $n \times m$ feature matrix. There, 
the hidden state of the network at the $\ell$-th layer is given by
$$
\bX^{(\ell)} (i) = \text{MLP}^{(\ell)}\left((1+\varepsilon^{(\ell)}) \bX^{(\ell-1)}(i) + \sum_{j : (i,j) \in E} W(i,j) \bX^{(\ell-1)}(j) \right) ,
$$
where $W(i,j)$ are the edge weights of $G = (V, E, W)$ and $\text{MLP}^{(\ell)}$ implements the transform step applied to the vector of aggregated features of each node $i$. Note, the input into $\text{MLP}^{(\ell)}$ can be written as $(1+\varepsilon^{(\ell)}) (\mathbf{H}_{\varepsilon^{(\ell)}} \bX^{(\ell-1)})(i)$ with $\mathbf{H}_{\varepsilon} = \bI + (1+\varepsilon)^{-1} \bA$. We note that the scaling constant $(1+\varepsilon^{(\ell)})$ can  be combined into the MLP. Therefore, we simply consider the aggregation step as applying the filter $\mathbf{H}_{\varepsilon}$. We note that $\mathbf{H}_{\varepsilon}$ is quite similar in form to the matrix $\bT_{g_\star}$ defined in \eqref{eqn: defT}. We also note that the  decay of $\mathbf{H}_{\varepsilon}$ is determined by a parameter $\varepsilon$ that effectively balances the information retained by node $i$ compared to the information aggregated from its neighbors. Over multiple applications of the filter, the decay parameter controls the propagation  speed over the graph, which can essentially be interpreted as controlling the spatial localization of information aggregated by the network in each layer.
\begin{table}[ht]
    \caption{Dataset statistics, diameter, nodes, edges, average clustering coefficient (CC).}\label{tab:dataset_stats}
\centering
\adjustbox{width=0.6\linewidth}{
    \begin{tabular}{lrrrrrr}
\toprule
{} &  Graphs &  Classes &  Diameter &   Nodes &    Edges &  CC \\
\midrule
COLLAB           &      5000 &          3 &      1.86 &   74.49 &  2457.22 &          0.89 \\
DD               &      1178 &          2 &     19.81 &  284.32 &   715.66 &          0.48 \\
IMDB-B     &      1000 &          2 &      1.86 &   19.77 &    96.53 &          0.95 \\
IMDB-M       &      1500 &          3 &      1.47 &   13.00 &    65.94 &          0.97 \\
MUTAG            &       188 &          2 &      8.22 &   17.93 &    19.79 &          0.00 \\
NCI1             &      4110 &          2 &     13.33 &   29.87 &    32.30 &          0.00 \\
NCI109           &      4127 &          2 &     13.14 &   29.68 &    32.13 &          0.00 \\
PROTEINS         &      1113 &          2 &     11.62 &   39.06 &    72.82 &          0.51 \\
PTC              &       344 &          2 &      7.52 &   14.29 &    14.69 &          0.01 \\
REDDIT-B    &      2000 &          2 &      8.59 &  429.63 &   497.75 &          0.05 \\
REDDIT-5K  &      4999 &          5 &     10.57 &  508.52 &   594.87 &          0.03 \\
\bottomrule
\end{tabular}}
\vspace{-5mm}
\end{table}

\begin{table}[htb]
\caption{Test accuracy $(\mu \pm \sigma)$ over 11 datasets and 10 seeds of $\wone$ vs.\ $\wtwo$ scattering for different choices of $\alpha \in [-0.5, -0.25, 0, 0.25, 0.5]$ where $\mathbf{M} = \mathbf{D}^{\alpha}$, corresponding to $\mathbf{K} = \mathbf{D}^{-\alpha} \mathbf{T} \mathbf{D}^{\alpha}$. In most settings of $\alpha$, $\wone$ slightly outperforms $\wtwo$. This suggests $\wone$ type filters should be explored further.}
\centering
\begin{tabular}{rrr}
\toprule
$\alpha$ & {$\wone$ (exact)} & {$\wtwo$} \\
\midrule
-0.5 & \textbf{0.617 $\pm$ 0.007} & 0.616 $\pm$ 0.012 \\
-0.25 & \textbf{0.640 $\pm$ 0.005} & 0.626 $\pm$ 0.009 \\
0.0 & \textbf{0.626 $\pm$ 0.012} & 0.623 $\pm$ 0.006 \\
0.25 & 0.619 $\pm$ 0.010 & \textbf{0.638 $\pm$ 0.008} \\
0.5 & \textbf{0.626 $\pm$ 0.009} & 0.616 $\pm$ 0.009 \\
\bottomrule
\end{tabular}
\label{tab:w1_vs_2}
\end{table}

\subsection{Spectral graph convolutional networks}
Spectral 
networks 
generalize 
convolution 
via the eigendecomposition of the graph Laplacian in a manner analogous to \eqref{eqn: Tfactorization}. The hidden state of the $\ell$-th layer will be an $n\times F_\ell$ feature matrix given by
\begin{equation*}
\mathbf{X}^{(\ell)} = \sigma(\mathbf{H}^{(\ell)} \mathbf{X}^{(\ell - 1)} \bm{\Theta}^{(\ell)}) =  \sigma(\bV \text{diag}(\mathbf{h}^{(\ell)}) \bV^T \mathbf{X}^{(\ell - 1)} \bm{\Theta}^{(\ell)}),
\end{equation*}
where $\sigma(\cdot)$ is a nonlinear activation function (e.g., ReLU, sigmoid, or absolute value in our case) and $\mathbf{h}^{(\ell)}$ is a filter in the graph Fourier domain. Here, the weight matrix $\bm{\Theta}^{(\ell)}$ implements the transform step, while the convolutional filter $\mathbf{H}^{(\ell)}$ implements the aggregation step. 
The primary difference between this network architecture and our model is the transform steps, which are not included in our formulation of the scattering transform. Indeed, in our analysis, we assume that we have a single input feature $\mathbf{x}$. 
However, we do note that scattering networks that incorporate transform steps were shown to be effective for node classification in \cite{min2020scattering,min2021gsan}. 

We further note that while some networks propose to parameterize $\mathbf{h}^{(\ell)}$ as a function of Laplacian eigenvalues (e.g., using Chebyshev\cite{Defferrard2018} or Caley\cite{Levie:CayleyNets2017} polynomials), some of the more popular architectures rely on nearly fully predefined filters. For example, the GCN architecture from \cite{kipf2016semi}, which is often used as a representative example of spectral GNNs, essentially implements $\mathbf{H}^{(\ell)}$ as a low-pass filter derived from first-order approximation of a parameterized polynomial. Indeed, up to renormalization (or reparametrization), the filter considered there is given by  $\mathbf{H}^{(\ell)} = \bI + \bD^{-1/2} \bA \bD^{-1/2} = 2 \mathbf{T}$. 
Therefore,  
as alluded to earlier, 
the relation  between the scattering framework and the GCN model can be regarded as stronger than in the traditional Euclidean case, since the aggregation step in such GCN architectures (similar to the message passing case) relies on predetermined filters, rather than flexible learned ones.

\begin{table}[ht]
\caption{Examination of the best parameters $\wone$ vs.\ $\wtwo$ and $\alpha \in \{-0.5, -0.25, 0, 0.25, 0.5\}$. Shows the number of times each setting was the best performing over 11 datasets and 10 seeds. We find that on average $\wone$ outperforms $\wtwo$ (64 to 46) and $\alpha \in \{-0.25, 0.25, 0.5\}$ outperforms $\alpha \in \{-0.5, 0.0\}$ (28 to 14).}
\centering
\begin{tabular}{lrr|r}
\toprule
$\alpha$ & $\wone$ & $\wtwo$ & Total \\
\midrule
-0.5 & 12 & 2 & 14 \\
-0.25 & 12 & 16 & 28 \\
0.0 & 8 & 5 & 13 \\
0.25 & 13 & 14 & 27 \\
0.5 & 19 & 9 & 28 \\
\midrule
Total & 64 & 46 & 110 \\
\bottomrule
\end{tabular}
\label{tab:w1_vs_2_count}
\end{table}
\begin{table}[htb]
\caption{Test accuracy $(\mu \pm \sigma)$ over 11 datasets and 10 seeds of $\wone$ scattering computation. `exact' computation of $\wone$ scattering requires an eigendecomposition of $T$ and therefore $O(n^3)$ time. Computation with the Chebyshev polynomial of order $\tau$ on a graph with $|E|$ non-zero edges takes $O(\tau (n + |E|))$ time which is efficient for sparse graphs with $|E| \approx O(n)$. We do not see a noticeable drop in performance when approximating with $\tau \in \{10, 100\}$ as compared to the exact implementation suggesting a fast implementation of approximate $\wone$ scattering using Chebyshev approximation.}
\centering
\adjustbox{max width=0.6\textwidth}{
\begin{tabular}{lrrr}
\toprule
$\tau$ & 10 & 100 & exact \\
\midrule
COLLAB & 0.692 $\pm$ 0.010 & 0.683 $\pm$ 0.012 & 0.702 $\pm$ 0.009 \\
DD & 0.685 $\pm$ 0.025 & 0.698 $\pm$ 0.024 & 0.695 $\pm$ 0.039 \\
IMDB-BINARY & 0.688 $\pm$ 0.040 & 0.666 $\pm$ 0.050 & 0.691 $\pm$ 0.031 \\
IMDB-MULTI & 0.406 $\pm$ 0.046 & 0.414 $\pm$ 0.029 & 0.398 $\pm$ 0.027 \\
MUTAG & 0.761 $\pm$ 0.074 & 0.733 $\pm$ 0.063 & 0.688 $\pm$ 0.070 \\
NCI1 & 0.653 $\pm$ 0.016 & 0.659 $\pm$ 0.029 & 0.645 $\pm$ 0.030 \\
NCI109 & 0.668 $\pm$ 0.022 & 0.628 $\pm$ 0.013 & 0.658 $\pm$ 0.023 \\
PROTEINS & 0.772 $\pm$ 0.023 & 0.799 $\pm$ 0.016 & 0.782 $\pm$ 0.022 \\
PTC MR & 0.325 $\pm$ 0.095 & 0.333 $\pm$ 0.097 & 0.387 $\pm$ 0.089 \\
REDDIT-BINARY & 0.805 $\pm$ 0.021 & 0.822 $\pm$ 0.021 & 0.814 $\pm$ 0.021 \\
REDDIT-MULTI-5K & 0.410 $\pm$ 0.013 & 0.422 $\pm$ 0.013 & 0.417 $\pm$ 0.016 \\
\midrule
Mean & 0.624 $\pm$ 0.163 & 0.624 $\pm$ 0.160 & 0.626 $\pm$ 0.151 \\
\bottomrule
\end{tabular}}
\label{tab:chebyshev}
\end{table}

\begin{table}[!ht]
    \caption{Full results, $(\mu \pm \sigma)$ over 10 seeds for 11 datasets with varying $\alpha$, $\{\wone, \wtwo \}$, and Chebyshev approximation of $\wone$ with $10$ and $100$ degree Chebyshev polynomials.}
    \centering
    \adjustbox{max width=\textwidth}{
        \begin{tabular}{llllllll}
\toprule
{$\alpha$} & {model} & {COLLAB} & {DD} & {IMDB-BINARY} & {IMDB-MULTI} & {MUTAG} & {NCI1} \\
\midrule
\multirow[c]{4}{*}{-0.5} & $\wone$ ($\tau=10$) & 0.699 $\pm$ 0.005 & 0.654 $\pm$ 0.021 & \textbf{0.740 $\pm$ 0.022} & 0.384 $\pm$ 0.032 & 0.795 $\pm$ 0.064 & 0.657 $\pm$ 0.011 \\
 & $\wone$ ($\tau=100$) & 0.676 $\pm$ 0.008 & 0.721 $\pm$ 0.017 & 0.727 $\pm$ 0.021 & 0.438 $\pm$ 0.026 & 0.725 $\pm$ 0.035 & 0.681 $\pm$ 0.022 \\
 & $\wone$ (exact) & 0.706 $\pm$ 0.006 & 0.727 $\pm$ 0.013 & 0.698 $\pm$ 0.026 & 0.397 $\pm$ 0.023 & 0.600 $\pm$ 0.000 & 0.622 $\pm$ 0.006 \\
 & $\wtwo$ & 0.699 $\pm$ 0.009 & 0.660 $\pm$ 0.014 & 0.588 $\pm$ 0.012 & 0.335 $\pm$ 0.043 & 0.775 $\pm$ 0.049 & 0.651 $\pm$ 0.006 \\
\multirow[c]{4}{*}{-0.25} & $\wone$ ($\tau=10$) & 0.684 $\pm$ 0.008 & 0.683 $\pm$ 0.013 & 0.655 $\pm$ 0.036 & 0.438 $\pm$ 0.033 & 0.660 $\pm$ 0.070 & 0.642 $\pm$ 0.018 \\
 & $\wone$ ($\tau=100$) & 0.683 $\pm$ 0.007 & 0.694 $\pm$ 0.026 & 0.635 $\pm$ 0.051 & 0.389 $\pm$ 0.016 & 0.700 $\pm$ 0.062 & 0.621 $\pm$ 0.014 \\
 & $\wone$ (exact) & 0.701 $\pm$ 0.010 & 0.711 $\pm$ 0.026 & 0.692 $\pm$ 0.038 & 0.373 $\pm$ 0.014 & 0.770 $\pm$ 0.026 & 0.650 $\pm$ 0.004 \\
 & $\wtwo$ & 0.693 $\pm$ 0.007 & 0.674 $\pm$ 0.011 & 0.579 $\pm$ 0.017 & 0.401 $\pm$ 0.045 & 0.740 $\pm$ 0.039 & 0.649 $\pm$ 0.009 \\
\multirow[c]{4}{*}{0.0} & $\wone$ ($\tau=10$) & 0.686 $\pm$ 0.010 & 0.699 $\pm$ 0.017 & 0.683 $\pm$ 0.019 & \textbf{0.444 $\pm$ 0.030} & 0.750 $\pm$ 0.035 & 0.663 $\pm$ 0.022 \\
 & $\wone$ ($\tau=100$) & 0.698 $\pm$ 0.005 & 0.696 $\pm$ 0.028 & 0.695 $\pm$ 0.012 & 0.402 $\pm$ 0.014 & 0.735 $\pm$ 0.034 & 0.637 $\pm$ 0.017 \\
 & $\wone$ (exact) & 0.690 $\pm$ 0.007 & 0.714 $\pm$ 0.015 & 0.713 $\pm$ 0.021 & 0.421 $\pm$ 0.027 & 0.675 $\pm$ 0.063 & 0.638 $\pm$ 0.006 \\
 & $\wtwo$ & 0.694 $\pm$ 0.007 & 0.720 $\pm$ 0.017 & 0.584 $\pm$ 0.019 & 0.428 $\pm$ 0.016 & 0.640 $\pm$ 0.021 & 0.667 $\pm$ 0.004 \\
\multirow[c]{4}{*}{0.25} & $\wone$ ($\tau=10$) & 0.701 $\pm$ 0.010 & 0.697 $\pm$ 0.023 & 0.688 $\pm$ 0.044 & 0.423 $\pm$ 0.014 & 0.775 $\pm$ 0.026 & 0.650 $\pm$ 0.009 \\
 & $\wone$ ($\tau=100$) & 0.690 $\pm$ 0.005 & 0.686 $\pm$ 0.021 & 0.667 $\pm$ 0.012 & 0.403 $\pm$ 0.036 & 0.725 $\pm$ 0.089 & 0.682 $\pm$ 0.006 \\
 & $\wone$ (exact) & 0.705 $\pm$ 0.005 & 0.675 $\pm$ 0.041 & 0.684 $\pm$ 0.015 & 0.406 $\pm$ 0.016 & 0.690 $\pm$ 0.052 & 0.617 $\pm$ 0.011 \\
 & $\wtwo$ & 0.700 $\pm$ 0.007 & \textbf{0.732 $\pm$ 0.011} & 0.609 $\pm$ 0.032 & 0.427 $\pm$ 0.013 & \textbf{0.845 $\pm$ 0.064} & 0.658 $\pm$ 0.010 \\
\multirow[c]{4}{*}{0.5} & $\wone$ ($\tau=10$) & 0.690 $\pm$ 0.007 & 0.695 $\pm$ 0.021 & 0.680 $\pm$ 0.025 & 0.346 $\pm$ 0.027 & 0.825 $\pm$ 0.035 & 0.654 $\pm$ 0.006 \\
 & $\wone$ ($\tau=100$) & 0.670 $\pm$ 0.009 & 0.695 $\pm$ 0.015 & 0.607 $\pm$ 0.008 & 0.434 $\pm$ 0.012 & 0.780 $\pm$ 0.059 & 0.676 $\pm$ 0.008 \\
 & $\wone$ (exact) & \textbf{0.708 $\pm$ 0.006} & 0.651 $\pm$ 0.031 & 0.666 $\pm$ 0.034 & 0.395 $\pm$ 0.031 & 0.705 $\pm$ 0.055 & \textbf{0.696 $\pm$ 0.010} \\
 & $\wtwo$ & 0.705 $\pm$ 0.010 & 0.724 $\pm$ 0.021 & 0.608 $\pm$ 0.031 & 0.419 $\pm$ 0.045 & 0.590 $\pm$ 0.077 & 0.659 $\pm$ 0.011 \\
\bottomrule
\end{tabular}

    }
    \adjustbox{max width=\textwidth}{
        
\begin{tabular}{lllllll|l}
\toprule
& & {NCI109} & {PROTEINS} & {PTC MR} & {REDDIT-B} & {REDDIT-M} & {Mean} \\
\midrule
\multirow[c]{4}{*}{-0.5} & $\wone$ ($\tau=10$) & 0.631 $\pm$ 0.008 & 0.772 $\pm$ 0.012 & 0.422 $\pm$ 0.106 & 0.794 $\pm$ 0.009 & 0.403 $\pm$ 0.007 & 0.631 $\pm$ 0.016 \\
 & $\wone$ ($\tau=100$) & 0.627 $\pm$ 0.009 & 0.790 $\pm$ 0.014 & 0.292 $\pm$ 0.057 & 0.821 $\pm$ 0.021 & 0.412 $\pm$ 0.012 & 0.627 $\pm$ 0.008 \\
 & $\wone$ (exact) & 0.624 $\pm$ 0.010 & 0.798 $\pm$ 0.013 & 0.423 $\pm$ 0.056 & 0.789 $\pm$ 0.014 & 0.403 $\pm$ 0.007 & 0.617 $\pm$ 0.007 \\
 & $\wtwo$ & 0.663 $\pm$ 0.010 & 0.802 $\pm$ 0.016 & 0.366 $\pm$ 0.072 & 0.814 $\pm$ 0.011 & 0.416 $\pm$ 0.007 & 0.616 $\pm$ 0.012 \\
\multirow[c]{4}{*}{-0.25} & $\wone$ ($\tau=10$) & 0.661 $\pm$ 0.007 & 0.790 $\pm$ 0.027 & 0.251 $\pm$ 0.044 & 0.793 $\pm$ 0.008 & 0.406 $\pm$ 0.012 & 0.607 $\pm$ 0.008 \\
 & $\wone$ ($\tau=100$) & 0.615 $\pm$ 0.015 & 0.803 $\pm$ 0.013 & 0.286 $\pm$ 0.096 & 0.820 $\pm$ 0.023 & 0.424 $\pm$ 0.009 & 0.609 $\pm$ 0.014 \\
 & $\wone$ (exact) & 0.676 $\pm$ 0.004 & 0.791 $\pm$ 0.009 & \textbf{0.456 $\pm$ 0.020} & 0.808 $\pm$ 0.017 & 0.410 $\pm$ 0.016 & \textbf{0.640 $\pm$ 0.005} \\
 & $\wtwo$ & 0.680 $\pm$ 0.008 & 0.791 $\pm$ 0.038 & 0.402 $\pm$ 0.051 & 0.841 $\pm$ 0.007 & \textbf{0.435 $\pm$ 0.009} & 0.626 $\pm$ 0.009 \\
\multirow[c]{4}{*}{0.0} & $\wone$ ($\tau=10$) & 0.675 $\pm$ 0.009 & 0.780 $\pm$ 0.024 & 0.292 $\pm$ 0.090 & 0.791 $\pm$ 0.013 & 0.419 $\pm$ 0.006 & 0.624 $\pm$ 0.010 \\
 & $\wone$ ($\tau=100$) & 0.631 $\pm$ 0.013 & 0.797 $\pm$ 0.013 & 0.289 $\pm$ 0.069 & 0.814 $\pm$ 0.014 & 0.430 $\pm$ 0.013 & 0.620 $\pm$ 0.009 \\
 & $\wone$ (exact) & 0.641 $\pm$ 0.004 & 0.778 $\pm$ 0.024 & 0.384 $\pm$ 0.097 & 0.820 $\pm$ 0.009 & 0.419 $\pm$ 0.011 & 0.626 $\pm$ 0.012 \\
 & $\wtwo$ & 0.647 $\pm$ 0.008 & 0.783 $\pm$ 0.017 & 0.443 $\pm$ 0.054 & 0.825 $\pm$ 0.009 & 0.421 $\pm$ 0.014 & 0.623 $\pm$ 0.006 \\
\multirow[c]{4}{*}{0.25} & $\wone$ ($\tau=10$) & \textbf{0.690 $\pm$ 0.006} & 0.759 $\pm$ 0.019 & 0.294 $\pm$ 0.018 & 0.812 $\pm$ 0.011 & 0.412 $\pm$ 0.017 & 0.627 $\pm$ 0.003 \\
 & $\wone$ ($\tau=100$) & 0.627 $\pm$ 0.005 & \textbf{0.805 $\pm$ 0.007} & 0.375 $\pm$ 0.119 & \textbf{0.844 $\pm$ 0.014} & 0.418 $\pm$ 0.012 & 0.632 $\pm$ 0.013 \\
 & $\wone$ (exact) & 0.673 $\pm$ 0.004 & 0.787 $\pm$ 0.012 & 0.298 $\pm$ 0.102 & 0.842 $\pm$ 0.014 & 0.433 $\pm$ 0.018 & 0.619 $\pm$ 0.010 \\
 & $\wtwo$ & 0.654 $\pm$ 0.010 & 0.796 $\pm$ 0.019 & 0.379 $\pm$ 0.022 & 0.822 $\pm$ 0.013 & 0.398 $\pm$ 0.010 & 0.638 $\pm$ 0.008 \\
\multirow[c]{4}{*}{0.5} & $\wone$ ($\tau=10$) & 0.682 $\pm$ 0.011 & 0.758 $\pm$ 0.013 & 0.367 $\pm$ 0.083 & 0.837 $\pm$ 0.015 & 0.412 $\pm$ 0.014 & 0.631 $\pm$ 0.009 \\
 & $\wone$ ($\tau=100$) & 0.640 $\pm$ 0.007 & 0.800 $\pm$ 0.025 & 0.422 $\pm$ 0.048 & 0.812 $\pm$ 0.017 & 0.429 $\pm$ 0.012 & 0.633 $\pm$ 0.006 \\
 & $\wone$ (exact) & 0.679 $\pm$ 0.006 & 0.756 $\pm$ 0.020 & 0.370 $\pm$ 0.063 & 0.813 $\pm$ 0.010 & 0.420 $\pm$ 0.008 & 0.626 $\pm$ 0.009 \\
 & $\wtwo$ & 0.642 $\pm$ 0.005 & 0.794 $\pm$ 0.016 & 0.453 $\pm$ 0.017 & 0.795 $\pm$ 0.019 & 0.391 $\pm$ 0.006 & 0.616 $\pm$ 0.009 \\
\bottomrule
\end{tabular}
    }
    \label{tab:full_results}
\end{table}
\section{Empirical Results}\label{sec: experiments}
In this section we empirically study the performance of geometric scattering with different wavelet families $\wone$ and  $\wtwo$ and with different choices of $\bK=\bM^{-1}\bT\bM$, focusing on the case where $\bM = \bD^\alpha$ for $\alpha \in \{-0.5, -0.25, 0, 0.25, 0.5\}$\footnote{Code to reproduce these results can be found at \url{https://github.com/atong01/trainable_symmetry}.}. Notably, the case where we use the wavelets $\wtwo$ and set $\alpha=0$ or $-0.5$ correspond to the settings of \cite{gama:diffScatGraphs2018} and \cite{gao:graphScat2018}, but the other settings do not correspond to wavelets previously used in the geometric scattering literature. We perform graph classification over 11 datasets (see Table~\ref{tab:dataset_stats}) and 10 initialization seeds reporting the mean and standard deviation ($\mu \pm \sigma$) of the test set accuracy (See Section~\ref{sec: implementation} for details). Overall, the optimal setting varies significantly between data sets as shown in Table~\ref{tab:full_results}. For example, on DD, the top performing model was $\wtwo$ with $\alpha=0.25$. However, this configuration is more than ten percentage points below the top-performer on IMDB-Binary. This highlights the practical importance of having a large family of scattering transforms, which a practitioner can choose from via a validation procedure.  When averaged over all data sets, the best choice of $\alpha$ is either $-0.25$ or $0.25$ depending on whether one used $\wone$ or $\wtwo$ as shown in Table~\ref{tab:w1_vs_2}. 
We also note that $\wone$ outperforms $\wtwo$ in a majority of cases  as shown in Table 
\ref{tab:w1_vs_2_count}.
Notably, one potential drawback to $\wone$ is that it requires computing the eigenvectors and eigenvalues of $\bK$ which is inefficient for large graphs. However, we show that this difficulty can be overcome via an approximation by Chebyshev polynomials with only a small loss of accuracy as shown in Table~\ref{tab:chebyshev}.

Specifically, an exact application of the $\wone$ filters requires an eigendecomposition of the graph, while the $\wtwo$ filter can be calculated with a polynomial of the Laplacian. For large graphs with $n$ nodes and $|E| \le n^2$ edges the $\wone$ filters take $O(n^3)$ time to compute where the $\wtwo$ filters take $O(2^J (|E| + n))$ time to compute. This means that $\wtwo$ filters are substantially quicker to compute than exact $\wone$ filters especially for sparse graphs and $2^J \ll n$. However, we show that the $\wone$ filter can be approximated with a Chebyshev polynomial approximation of order $\tau$ which has a similar computation time to the $\wtwo$ filter with no noticeable performance drop. In Table~\ref{tab:chebyshev}, we compare the performance of the exact $\wone$ implementation vs.\ a Chebyshev approximation to the filter of order $\tau \in \{10, 100\}$ over 11 datasets. On average over the datasets the exact $\wone$ filter has a test accuracy of $0.626$ vs.\ $0.624$ for both $\tau=10$ and $\tau=100$.



\section{Future Work}\label{sec: future}
We have introduced a large class of scattering networks with provable guarantees. As alluded to in Section \ref{sec: related}, we believe that our work opens up several new lines of inquiry. 
One might attempt to learn the optimal choices of the matrix $\bM$ and the spectral function $g$ from a parameterized family based on training data, yielding a data-driven architecture with theoretical stability guarantees.  
Another possible extension would be to  consider a construction similar to ours  but which uses the spectral decomposition of the unnormalized graph Laplacian rather than the normalized Laplacian. Such a work would generalize \cite{zou:graphCNNScat2018} in a manner analogous to the way that this work generalizes \cite{gama:diffScatGraphs2018} and \cite{gao:graphScat2018}. 
Additionally, one might attempt to incorporate an attention mechanism such as that use in GAT \cite{velivckovic2017graph} into the scattering framework  for improved numerical performance.
Lastly, 
one might wish to study the behavior of the graph scattering transform on  data-driven graphs obtained by subsampling a Riemannian manifold $\mathcal{M}.$ Such graphs typically arise in high-dimensional data analysis and  manifold learning. It can be shown that, under certain conditions, the normalized graph Laplacian  converges pointwise \cite{coifman:diffusionMaps2006, singer:GraphToManifold2006} or in a spectral sense \cite{belkin2007convergence, Burago2013,  Fujiwara1995EigenvaluesOL, Shi2015,  Trillos2018} to the Laplace Beltrami operator on $\mathcal{M}$ as the number of samples tends to infinity. One might hope to use these results to study the convergence of the graph scattering transforms constructed here to the manifold scattering transform constructed in \cite{perlmutter:geoScatCompactManifold2020}.

\section*{Acknowledgements} 

This research was partially funded by IVADO (Institut de valorisation des données) grant PRF-2019-3583139727, FRQNT (Fonds de recherche du Québec - Nature et technologies) grant 299376, Canada CIFAR AI Chair [\emph{G.W.}];  NIH (National Institutes of Health) - NIGMS grant R01GM135929 [\emph{M.H.,G.W.}]; and NSF (National Science Foundation) - DMS grant 1845856 [\emph{M.H.}] The content provided here is solely the responsibility of the authors and does not necessarily represent the official views of the funding agencies.



\bibliographystyle{siamplain}
\bibliography{mainbib}

\newpage
\appendix

\section{The proof of Lemma \ref{lem: polynomialproperties}} \label{sec: Proof of Polynomial properties}

\begin{proof}
Since $\bV$ is unitary, $\bV^{-1}=\bV^T,$ and so it follows from \eqref{eqn: Tfactorization} that
\begin{equation*}
    \bT^r=\bV\Lambda^r\bV^T, \quad \forall \, r \in \mathbb{N}.
\end{equation*}
Moreover, since $\bK=\bM^{-1}\bT\bM$,
\begin{equation*}
    \bK^r=\left(\bM^{-1}\bT\bM\right)^r=\bM^{-1}\bT^r\bM=\bM^{-1}\bV\Lambda^r\bV^T\bM.
\end{equation*}
Linearity now implies \eqref{eqn: Tpolys}. The second claim 
follows by recalling that  $\|\bx\|_\bM=\|\bM\bx\|_2,$  and noting therefore that for all $\bx,$ 
\begin{equation*}
    \|p(\bK)\bx\|_\bM=\|\bM(\bM^{-1}p(\bT)\bM)\bx\|_2=\|p(\bT)\bM\bx\|_2.
\end{equation*}
\end{proof}

\section{The proof of Propositions \ref{prop: waveletisometries} and \ref{prop: nonexpansivewaveletframes}}
\label{sec: The proof of Proposition wavelet isometrices}
\begin{proof}[The proof of Proposition \ref{prop: waveletisometries}]
Proposition \ref{lem: selfadjoint} shows $\bK$ is self-adjoint on $\mathbf{L}^2(\bM).$ 
By \\ Lemma~\ref{lem: polynomialproperties} and by \eqref{eqn: defsquareroot} we have 
\begin{equation*}
    \Psi_j^{(1)}= q_j(\bK)=\bM^{-1}\bV q_j(\Lambda)\bV^T\bM, \quad \forall \, 0\leq j \leq J,
\end{equation*}
and
\begin{equation*}
    \Phi_J^{(1)}=q_{J+1}(\bK)= \bM^{-1}\bV q_{J+1}(\Lambda)\bV^T\bM.
\end{equation*}
Next, one make check that 
\begin{equation*}
    \| \Psi_j^{(1)} \bx \|_{\bM}^2 =
    \langle \Psi_j^{(1)} \bx,\Psi_j^{(1)} \bx\rangle_{\bM}= \bx^T \bM^T \bV q_j (\Lambda)^2 \bV^T \bM \bx,
\end{equation*}
and thus,
\begin{align*}
    \| \cW_J^{(1)} \bx \|_{\bM}^2 &= \sum_{j=0}^J \| \Psi_j^{(1)} \bx \|_{\bM}^2 + \| \Phi_J^{(1)} \bx \|_{\bM}^2 \\
    &= \bx^T \bM^T \bV \left[ \sum_{j=0}^{J+1} q_j (\Lambda)^2 \right] \bV^T \bM \bx \\
    &= \bx^T \bM^T \bV Q_J (\Lambda) \bV^T \bM \bx \\
    &= \langle Q_J (\Lambda) \bV^T \bM \bx, \bV^T \bM \bx \rangle,
\end{align*}
where $Q_J(t) \coloneqq \sum_{j=0}^{J+1}q_j(t)^2$. Therefore the lower frame bound  of $\cW_J^{(1)}$ is given by
\begin{align}
    c_J &\coloneqq \inf_{\bx \neq \bm{0}} \frac{\| \cW_J^{(1)} \bx \|_{\bM}^2}{\| \bx \|_{\bM}^2} \nonumber \\
    &= \inf_{\bx \neq \bm{0}} \frac{\langle Q_J (\Lambda) \bV^T \bM \bx, \bV^T \bM \bx \rangle}{\| \bM \bx \|_2^2} \nonumber \\
    &= \inf_{\bx \neq \bm{0}} \frac{\langle Q_J (\Lambda) \bV^T \bM \bx, \bV^T \bM \bx \rangle}{\| \bV^T \bM x \|_2^2} \label{eqn: lower frame bound proof 01} \\
    &= \inf_{\by \neq \bm{0}} \frac{\langle Q_J (\Lambda) y, y \rangle}{\| \by \|_2^2} \label{eqn: lower frame bound proof 02} \\
    &= \min_{0 \leq i \leq n-1} Q_J (\lambda_i), \label{eqn: lower frame bound proof 03}
\end{align}
where \eqref{eqn: lower frame bound proof 01} follows from the fact that the columns of $\bV$ form an ONB for the standard, unweighted $\mathbf{L}^2$ space, \eqref{eqn: lower frame bound proof 02} results from $\bM$ being invertible, and \eqref{eqn: lower frame bound proof 03} is true since $Q_J (\Lambda)$ is a diagonal matrix with entries $Q_J (\lambda_i)$ along the diagonal. A similar calculation shows the upper frame bound of $\cW_J^{(1)}$ is
\begin{equation*}
    C_J \coloneqq \sup_{\bx \neq \bm{0}} \frac{\| \cW_J^{(1)} \bx \|_{\bM}^2}{\| \bx \|_{\bM}^2} = \max_{0 \leq i \leq n-1} Q_J (\lambda_i).
\end{equation*}

 Therefore, the lower and upper frame bounds of $\mathcal{W}_J^{(1)}$ are given by computing
\begin{equation*}
    \min_{0\leq i\leq n-1} Q_J(\lambda_i)\quad\text{and}\quad\max_{0\leq i\leq n-1} Q_J(\lambda_i).
\end{equation*}
The proof follows from recalling that  by \eqref{eqn: sum to 1}, we have $Q(t)=1$ uniformly on $0\leq t\leq 1$ and therefore $\cW^{(1)}_J$ is an isometry.
\end{proof}

\begin{proof}[The proof of Proposition \ref{prop: nonexpansivewaveletframes}]
By the same reasoning as in the proof of Proposition \ref{prop: waveletisometries}, the 
frame bounds of $\mathcal{W}_J^{(2)}$ are given by computing
\begin{equation*}
    \min_{0\leq i\leq n-1} P_J(\lambda_i)\quad\text{and}\quad\max_{0\leq i\leq n-1} P_J(\lambda_i),
\end{equation*}
where $P_J(t)=\sum_{j=0}^{J+1} p_j(t)^2.$ Since  $0\leq \lambda_i\leq 1$ for all $i,$ we have 
\begin{equation*}
    \max_{0 \leq i \leq n-1} P_J(\lambda_i) \leq \sup_{t \in [0,1]}\sum_{j=0}^{J+1} p_j(t)^2\leq \sup_{t \in [0,1]}\left(\sum_{j=0}^{J+1} p_j(t)\right)^2=1
\end{equation*}
with the second inequality following from the fact that $p_j(t)\geq 0$ for all $t\in[0,1],$ and the last equality following from \eqref{eqn: sum to 1}. For the lower bound, it suffices to show that 
\begin{equation*}
     \inf_{t \in [0,1]}\sum_{j=0}^{J+1} p_j(t)^2 \geq c>0
\end{equation*}
for a universal constant $c.$ To do so, we let $0\leq t\leq 1$ and consider three cases. First, if $0\leq t\leq 1/2,$ then
\begin{equation*}
    \sum_{j=0}^{J+1} p_j(t)^2 \geq p_0(t)^2 = (1-t)^2\geq \left(1-\frac{1}{2}\right)^2=1/4.
\end{equation*}
Secondly, if $t^{2^J}\geq1/2$, 
\begin{equation*}
    \sum_{j=0}^{J+1} p_j(t)^2  \geq p_{J+1}(t)^2 = \left(t^{2^J}\right)^2\geq \left(\frac{1}{2}\right)^2=1/4.
\end{equation*}
In the final case where $t^{2^{J}}<\frac{1}{2}<t,$ there exists a unique $j_0,$ $1\leq j_0\leq J,$ such that $t^{2^{j_0}} < 1/2 \leq t^{2^{j_0-1}}.$ Since $t^{2^{j_0-1}} \geq 1/2$ and $t^{2^{j_0-1}} t^{2^{j_0-1}} = t^{2^{j_0}}$ it follows that $1/4\leq t^{2^{j_0}} < 1/2$ and thus $1/2 \leq t^{2^{j_0-1}} < 1/\sqrt{2}$.
Therefore, in this case we have 
\begin{equation*}
    \sum_{j=0}^{J+1} p_j(t)^2 \geq p_{j_0}(t)^2 = (t^{2^{j_0-1}} - t^{2^{j_0}})^2 \geq \inf_{x \in [\frac{1}{2}, \frac{1}{\sqrt{2}}]} (x - x^2)^2 = c > 0.
\end{equation*}
\end{proof}

\section{The Proof of Proposition \ref{prop: limit}}
\label{sec: the proof of proposition limi}

\begin{proof} 
Let $\bx\in\mathbf{L}^2(\bM),$ let $\pathvar$ be a path, and let $\by=\bU[\pathvar]\bx$. 
Let $t=2^{J-1}$ if $\mathcal{W}=\mathcal{W}^{(1)},$ and let $t=2^{J}$ if $\mathcal{W}=\mathcal{W}^{(2)}$ so that in either case $\Phi=\bK^t.$  By definition, we have $\bS_J[\pathvar]\bx=\bK^t\by$ and $\overline{\bS}_{\bu_0}[\pathvar]\bx =\langle \bu_0, \by\rangle_{\bM}$. Applying \eqref{eqn: Kt}, we have 
\begin{equation*}
\bK^t\by=\sum_{i=0}^{n-1}\lambda_i^t\langle\mathbf{v}_i,\bM\by\rangle_2\mathbf{u}_i.
\end{equation*}
Therefore, since $\lambda_0=1,$ we have
\begin{equation*}
    \bK^t\by-\langle\mathbf{v}_0,\bM\by\rangle_2\bu_0=\sum_{i=1}^{n-1}\lambda_i^t\langle\mathbf{v}_i,\bM\by\rangle_2\mathbf{u}_i.
\end{equation*}
Repeating the arguments used to prove \eqref{eqn: PiKtbound}, we see that 
\begin{equation*}
    \left\|\sum_{i=1}^{n-1}\lambda_i^t\langle\mathbf{v}_i,\bM\by\rangle_2\mathbf{u}_i\right\|_{\bM}\leq \lambda_1^{2t}\|\by\|_\bM^2.
\end{equation*}
 $\omega_1>0$ implies $\lambda_1<1$. Thus, since $\langle\mathbf{v}_0,\bM\by\rangle_2,=\langle\mathbf{u}_0,\by\rangle_\bM=\overline{\bS}_{\bu_0}[\pathvar]\bx$ the result follows.
\end{proof}

\section{Lemma \ref{lem: energylevels}}\label{sec: Lemma EnergyLevels}

In this section, we state and prove the following lemma which is useful in the proof of Theorems \ref{thm: nonexpansive} and \ref{thm: energydecay}. It is quite similar to Proposition 3.2 of \cite{zou:graphCNNScat2018} and also the proof of Proposition 2.5 in \cite{mallat:scattering2012}. However, we give full details for the sake of completeness.
\begin{lemma}\label{lem: energylevels}
Assume  $C\leq 1$ in \eqref{eqn: frameAB}. Then, for all $m\geq 1$ and
for all $\bx\in\mathbf{L}^2(\bM)$
\begin{equation}\label{eqn: energylevelsequals}
    \sum_{\pathvar\in\mathcal{J}^m}\|\bU[\pathvar]\mathbf{x}\|_\bM^2\geq\sum_{\pathvar\in\mathcal{J}^{m+1}}\|\bU[\pathvar]\mathbf{x}\|_\bM^2+\sum_{\pathvar\in\mathcal{J}^{m}}\|\bS[\pathvar]\mathbf{x}\|_\bM^2
\end{equation}
with equality holding if $A=B=1.$ 
Furthermore, for all  $\mathbf{x},\mathbf{y}\in\mathbf{L}^2(\bM),$
we have 
\begin{equation}\label{eqn: energylevelsgeq}
    \sum_{\pathvar\in\mathcal{J}^m}\|\bU[\pathvar]\mathbf{x}-\bU[\pathvar]\mathbf{y}\|_\bM^2\geq\sum_{\pathvar\in\mathcal{J}^{m+1}}\|\bU[\pathvar]\mathbf{x}-\bU[\pathvar]\mathbf{y}\|_\bM^2+\sum_{\pathvar\in\mathcal{J}^{m}}\|\bS[\pathvar]\mathbf{x}-\bS[\pathvar]\mathbf{y}\|_\bM^2.
\end{equation}
\end{lemma}

\begin{proof}[The Proof of Lemma \ref{lem: energylevels}]
Since by assumption we have $C\leq 1$ in \eqref{eqn: frameAB}, it follows that for all $\pathvar\in\mathcal{J}^m$
that 
\begin{align*}
 \|\bU[\pathvar]\bx-\bU[\pathvar]\by\|_\bM^2&\geq\sum_{j_{m+1}\in\mathcal{J}}   \|\Psi_{j_{m+1}}(\bU[\pathvar]\bx-\bU[\pathvar]\by)\|_\bM^2+\|\Phi(\bU[\pathvar]\bx-\bU[\pathvar]\by)\|_\bM^2.
\end{align*}
Therefore,
\begin{align}
 &\sum_{\pathvar\in\mathcal{J}^m}\|\bU[\pathvar]\mathbf{x}-\bU[\pathvar]\mathbf{y}\|_\bM^2\nonumber\\
 &\geq \sum_{\pathvar\in\mathcal{J}^m}\left(\sum_{j_{m+1}\in\mathcal{J}}   \|\Psi_{j_{m+1}}(\bU[\pathvar]\mathbf{x}-\bU[\pathvar]\mathbf{y})\|_\bM^2+\|\Phi(\bU[\pathvar]\mathbf{x}-\bU[\pathvar]\mathbf{y})\|_\bM^2\right)\label{eqn: theothergeq}\\
 &=   \sum_{\pathvar\in\mathcal{J}^m}\left(\sum_{j_{m+1}\in\mathcal{J}}   \|\Psi_{j_{m+1}}\bU[\pathvar]\mathbf{x}-\Psi_{j_{m+1}}\bU[\pathvar]\mathbf{y}\|_\bM^2+\|\Phi(\bU[\pathvar]\mathbf{x}-\bU[\pathvar]\mathbf{y})\|_\bM^2\right)\nonumber\\
 &\geq   \sum_{\pathvar\in\mathcal{J}^m}\left(\sum_{j_{m+1}\in\mathcal{J}}   \|M\Psi_{j_{m+1}}\bU[\pathvar]\mathbf{x}-M\Psi_{j_{m+1}}\bU[\pathvar]\mathbf{y}\|_\bM^2+\|\Phi(\bU[\pathvar]\mathbf{x}-\bU[\pathvar]\mathbf{y})\|_\bM^2\right)\label{eqn: thegeq}\\
 &=   \sum_{\pathvar\in\mathcal{J}^m}\left(\sum_{j_{m+1}\in\mathcal{J}}   \|\bU[j_{m+1}]\bU[\pathvar]\mathbf{x}-\bU[j_{m+1}]\bU[\pathvar]\mathbf{y}\|_\bM^2+\|\Phi(\bU[\pathvar]\mathbf{x}-\bU[\pathvar]\mathbf{y})\|_\bM^2\right)\nonumber\\
 &= \sum_{\pathvar\in\mathcal{J}^{m+1}}\|\bU[\pathvar]\mathbf{x}-\bU[\pathvar]\mathbf{y}\|_\bM^2+\sum_{\pathvar\in\mathcal{J}^{m}}\|\bS[\pathvar]\mathbf{x}-\bS[\pathvar]\mathbf{y}\|_\bM^2\nonumber.
 \end{align}
 This completes the proof of \eqref{eqn: energylevelsgeq}. \eqref{eqn: energylevelsequals} follows from setting $\mathbf{y}=0.$ Lastly, we observe that if that $c=C=1$ in \eqref{eqn: frameAB} and $\by=0,$ we have equality in the inequalities  \eqref{eqn: theothergeq} and \eqref{eqn: thegeq}. 

\end{proof}

\section{The Proof of Theorem  \ref{thm: nonexpansive}}\label{sec: The proof of thm: non expansive}
\begin{proof} 
Applying Lemma  \ref{lem: energylevels}, which is stated in Appendix \ref{sec: Lemma EnergyLevels}, and recalling that $U[\pathvar_e]\bx=\bx,$ we see
\begin{align*}
    \|\bS\mathbf{x}-\bS\mathbf{y}\|_{\bM}^2&=\lim_{N\rightarrow\infty}\sum_{m=0}^N\sum_{\pathvar\in\mathcal{J}^m}\|\bS[\pathvar]\mathbf{x}-\bS[\pathvar]\mathbf{y}\|_\bM^2\\
    &\leq \lim_{N\rightarrow\infty}\sum_{m=0}^N  \left(\sum_{\pathvar\in\mathcal{J}^m}\|\bU[\pathvar]\mathbf{x}-\bU[\pathvar]\mathbf{y}\|_\bM^2-\sum_{\pathvar\in\mathcal{J}^{m+1}}\|\bU[\pathvar]\mathbf{x}-\bU[\pathvar]\mathbf{y}\|_\bM^2\right) \\
    &\leq   \|\mathbf{x}-\mathbf{y}\|_\bM^2-\limsup_{N\rightarrow\infty}\sum_{\pathvar\in\mathcal{J}^{N+1}}\|\bU[\pathvar]\mathbf{x}-\bU[\pathvar]\mathbf{y}\|_\bM^2 \\
&\leq   \|\mathbf{x}-\mathbf{y}\|_\bM^2.
\end{align*} 
This proves \eqref{eqn: nonexpanisvewindow}.
To prove \eqref{eqn: nonexpanisvenonwindow}, we assume that $\cW$ is either of the wavelet transforms $\cW^{(1)}_J$ or $\cW^{(2)}_J$ constructed in Section \ref{sec: wavelets} and apply Proposition \ref{prop: limit} to see that for all $0\leq i \leq n-1$,
\begin{equation*}
    \|\overline{\bS}_{\bu_0}\bx-\overline{\bS}_{\bu_0}\by\|^2_{2} = \sum_\pathvar |\overline{\bS}_{\bu_0}[\pathvar]\bx-\overline{\bS}_{\bu_0}[\pathvar]\by|^2= \sum_{\pathvar} \lim_{J\rightarrow\infty}\left|\frac{\bS_J[\pathvar]\bx(i)-\bS_J[\pathvar]\by(i)}{\bu_0(i)}\right|^2.
\end{equation*}
Therefore, summing over $0\leq i\leq n-1$ and applying Fatou's lemma we have 
\begin{align*}
    \|\overline{\bS}_{\bu_0}\bx-\overline{\bS}_{\bu_0}\by\|^2_{2}&= \sum_{i=0}^{n-1}\frac{1}{n}\sum_\pathvar \lim_{J\rightarrow\infty}\left|\frac{\bS_J[\pathvar]\bx(i)-\bS_J[\pathvar]\by(i)}{\bu_0(i)}\right|^2\\
    &\leq \frac{1}{n}\frac{1}{\min_i|\bu_0(i)|^2}\lim_{J\rightarrow\infty}\sum_{i=0}^{n-1}\sum_\pathvar \left|\bS_J[\pathvar]\bx(i)-\bS_J[\pathvar]\by(i)\right|^2\\
    &= \frac{1}{n}\frac{1}{\min_i|\bu_0(i)|^2}\lim_{J\rightarrow\infty}\|\bS_J\bx-\bS_J\by\|_2^2\\
    &\leq \frac{1}{n}\frac{\| \bM^{-1}\|_2^2}{\min_i|\bu_0(i)|^2}\lim_{J\rightarrow\infty}\|\bS_J\bx-\bS_J\by\|_{\bM}^2 \\
&\leq \frac{1}{n}\frac{\| \bM^{-1} \|_2^2}{\min_i|\bu_0(i)|^2}\|\bx-\by\|_{\bM}^2,
\end{align*}
where the last inequality follows from \eqref{eqn: nonexpanisvewindow}.
\end{proof}

\section{The Proof of Theorem \ref{thm: energydecay}}
\label{sec: proof of energy decay}

\begin{proof}
Let $m\geq 1,$ let $t=2^{J-1}$ if $\mathcal{W}=\mathcal{W}^{(1)},$ and let $t=2^{J}$ if $\mathcal{W}=\mathcal{W}^{(2)},$  so that in either case $\Phi=\bT^t.$ Let $\bx\in\mathbf{L}^2(\bM),$ and let  $\by=\bU[\pathvar]\bx.$ Equation \eqref{eqn: Tpolys} implies that for any $\bz\in\mathbf{L}^2(\bM)$
\begin{equation*}
    \bT^t\bz=\sum_{i=0}^{n-1}\lambda_i^t\langle\mathbf{v}_i,\bz\rangle_2\mathbf{v}_i.
\end{equation*}
Therefore, by Lemma \ref{lem: polynomialproperties} and the relationship $\mathbf{u}_i=\bM^{-1}\mathbf{v}_i,$ we have 
\begin{equation*}
\bK^t\by=\bM^{-1}\bT^t(\bM \by)=\sum_{i=0}^{n-1}\lambda_i^t\langle\mathbf{v}_i,\bM\by\rangle_2\bM^{-1}\mathbf{v}_i=\sum_{i=0}^{n-1}\lambda_i^t\langle\mathbf{v}_i,\bM\by\rangle_2\mathbf{u}_i
\end{equation*}
  By Parseval's identity, the fact that $\{\mathbf{u}_0,\ldots,\mathbf{u}_{n-1}\}$ is an orthornomal basis for $\mathbf{L}^2(\bM),$ and the fact that $\lambda_0=1,$ we have for all $\pathvar\in\mathcal{J}^m,$ 
\begin{align*}
    \left\|\bS[\pathvar]\bx\right\|^2_\bM&=\|\bK^t\by\|_\bM^2\\
    &=\left\|\sum_{i=0}^{n-1}\lambda_i^t\langle\mathbf{v}_i,\bM\by\rangle_2\mathbf{u}_i\right\|_\bM^2\\
    &=\sum_{i=0}^{n-1}\lambda_i^{2t}\left|\langle\mathbf{v}_i,\bM\by\rangle_2\right|^2\\
    &\geq |\langle\mathbf{v}_0,\bM\by\rangle_2|^2.
\end{align*}
Since $\mathbf{v}_0=\frac{\bd^{1/2}}{\|\bd^{1/2}\|_2}=\frac{\bd^{1/2}}{\|\bd\|_1},$
\begin{align*}
    |\langle\mathbf{v}_0,\bM\by\rangle_2|^2&=\left|\sum_{i=0}^{n-1}\frac{\bd(i)}{\|\bd^{1/2}\|_2}(\bM\by)(i)\right|\\
    &\geq \frac{\bd_{\min}}{\|\bd^{1/2}\|_2^2}\|\bM\by\|_1^2\\
    &\geq\frac{\bd_{\min}}{\|\bd\|_1}\|\bM\by\|_2^2\\
    &=\frac{\bd_{\min}}{\|\bd\|_1}\|\by\|_\bM^2 \\
    &= \frac{\bd_{\min}}{\|\bd\|_1}\|\bU [\pathvar] \bx \|_\bM^2.
\end{align*}
Summing over $\pathvar,$ this gives
\begin{equation*}
\sum_{\pathvar\in\mathcal{J}^{m}}\|\bS[\pathvar]\mathbf{x}\|_{\bM}^2\geq \frac{\bd_{\min}}{\|\bd\|_1}  \sum_{\pathvar\in\mathcal{J}^m}\|\bU[\pathvar]\mathbf{x}\|_{\bM}^2.
  \end{equation*}
  Therefore, by
  Lemma \ref{lem: energylevels} (stated in Appendix \ref{sec: Lemma EnergyLevels}) we see
  \begin{equation*}
\sum_{\pathvar\in\mathcal{J}^{m+1}}\|\bU[\pathvar]\mathbf{x}\|_{\bM}^2 \leq 
\sum_{\pathvar\in\mathcal{J}^{m}}\|\bU[\pathvar]\mathbf{x}\|_{\bM}^2-
\sum_{\pathvar\in\mathcal{J}^{m}}\|\bS[\pathvar]\mathbf{x}\|_{\bM}^2
  \leq \left(1-\frac{\bd_{\min}}{\|\bd\|_1}\right)
\sum_{\pathvar\in\mathcal{J}^{m}}\|\bU[\pathvar]\mathbf{x}\|_{\bM}^2.
\end{equation*}
This proves \eqref{eqn: ratio}. To prove \eqref{eqn: decay}, we note that since $C\leq1$ in \eqref{eqn: frameAB}, we have  
\begin{equation*}
    \sum_{\pathvar\in\mathcal{J}}\|\bU[\pathvar]\bx\|_\bM^2\leq \|\bx\|_\bM^2.
\end{equation*}
Therefore, \eqref{eqn: decay} holds when $m=0.$ The result follows for $m\geq 1$ by  iteratively applying \eqref{eqn: ratio}.
\end{proof}

\section{The Proof of Theorem \ref{thm: conservation of energy}}\label{sec: the proof of conservation of energy}

\begin{proof}
Recall from Proposition \ref{prop: waveletisometries} that $\| \cW_J^{(1)} \bx \|_{\bM}= \| \bx \|_{\bM}$, which combined with Lemma \ref{lem: energylevels} gives:
\begin{equation*}
    \sum_{\pathvar\in\mathcal{J}^{m}}\|\bS[\pathvar]\mathbf{x}\|_{\bM}^2 = \sum_{\pathvar\in\mathcal{J}^{m}}\|\bU[\pathvar]\mathbf{x}\|_{\bM}^2 - \sum_{\pathvar\in\mathcal{J}^{m+1}}\|\bU[\pathvar]\mathbf{x}\|_{\bM}^2.
\end{equation*}
Therefore,
\begin{align*}
    \|\bS\bx\|^2_{\bM} &= \lim_{N \rightarrow \infty} \sum_{m=0}^N \sum_{\pathvar \in \cJ^m} \| \bS [\pathvar] \bx \|_{\bM}^2 \\
    &= \lim_{N\rightarrow\infty} \sum_{m=0}^N\left( \sum_{\pathvar\in\mathcal{J}^{m}}\|\bU[\pathvar]\mathbf{x}\|_{\bM}^2 - \sum_{\pathvar\in\mathcal{J}^{m+1}}\|\bU[\pathvar]\mathbf{x}\|_{\bM}^2\right)\\
    &=\|\bx\|_\bM^2-\lim_{N\rightarrow\infty}\left(\sum_{\pathvar\in\mathcal{J}^{N+1}}\|\bU[\pathvar]\mathbf{x}\|_\bM^2\right) \\
    &\geq \| \bx \|_{\bM}^2 - \lim_{N \rightarrow \infty} \left( 1 - \frac{\bd_{\min}}{\| \bd \|_1} \right)^N \| \bx \|_{\bM}^2 \\
    &=\|\bx\|_\bM^2,
\end{align*}
where the inequality follows from  Theorem \ref{thm: energydecay}. However, we also proved in Theorem \ref{thm: nonexpansive} that $\| \bS \bx \|_{\bm{\ell}^2 (\bL^2 (G, \bM))} \leq \| \bx \|_{\bM}$, and thus we conclude that $\| \bS \bx \|_{\bm{\ell}^2 (\bL^2 (G, \bM))} = \| \bx \|_{\bM}$.
\end{proof}
\section{The Proof of Theorems \ref{thm: equivariance} and \ref{thm: permuationinvariancewindowed}}\label{sec: invariance proofs}
\begin{proof}[The proof of Theorem \ref{thm: equivariance}]
Let $\Pi$ be a permutation.  Since $\Pi(M\bx)=M(\Pi\bx)$  and  $\Pi^T=\Pi^{-1},$  it follows that for all $j\in\mathcal{J}$
\begin{equation*}
    \bU'[j]\Pi\bx=M\Psi'_j\Pi\bx= M\Pi \Psi_j\Pi^T\Pi\bx=M\Pi \Psi_j\bx=\Pi M\Psi_j\bx=\Pi \bU[j]\bx.
\end{equation*}
For $\pathvar=(j_1,\ldots,j_m),$ we have $\bU[\pathvar]=\bU[j_1]\ldots\bU[j_m].$ Therefore, it follows inductively that  $\bU$ is equivariant to permutations. Since $\bS=\Phi\bU,$ we have that 
\begin{equation*}
    \bS'\Pi\bx=\Phi'\bU'\Pi\bx=\Pi\Phi\Pi^T\Pi\bU\bx=\Pi\bS\bx.
\end{equation*}
Thus, the windowed scattering transform is permutation equivariant as well.
\end{proof}

\begin{proof}[The proof of Theorem \ref{thm: permuationinvariancewindowed}]
Since  $\mathbf{U}$ is permutation equivariant by Theorem \ref{thm: equivariance} and $\bmu'=\Pi\bmu,$ we may use the fact that $\bM'=\Pi\bM\Pi^T$ and that $\Pi^T=\Pi^{-1}$ to see that for any $\bx$ and any $\pathvar,$
\begin{align*}
    \overline{\bS}'[\pathvar]\Pi\bx&=\langle\bmu',\bU'[\pathvar]\Pi\bx\rangle_{\bM'}\\
    &=\langle\bM'\Pi\bmu,\bM'\Pi\bU[\pathvar]\bx\rangle_2\\
    &=\langle\Pi\bM\bmu,\Pi\bM\bU[\pathvar]\bx\rangle_2\\&=\langle\bM\bmu,\bM\bU[\pathvar]\bx\rangle_2\\&=\overline{\bS}[\pathvar]\bx.
\end{align*}
\end{proof}

\section{The Proof of Theorem \ref{thm: permuatation invariance}}\label{sec: Invariance Windowed}

\begin{proof}
By Theorem \ref{thm: equivariance}, and the fact that $\bS=\Phi\bU$ we see that for any $\bM$,
\begin{align}
    \|\bS'\Pi\bx-\bS\bx\|_{\bM} &=
    \|\Pi\bS\bx-\bS\bx\|_{\bM}
    = \|\Pi\Phi\bU\bx-\Phi\bU\bx\|_{\bM}\nonumber
    \leq \|\Pi\Phi-\Phi\|_\bM\|\bU\bx\|_{\bM}.\label{eqn: reducetoU}
\end{align} 
Let $t=2^{J-1}$ if $\mathcal{W}=\mathcal{W}^{(1)},$ and let $t=2^{J}$ if $\mathcal{W}=\mathcal{W}^{(2)}$ so that in either case $\Phi=\bK^t.$ By \eqref{eqn: Tpolys} we see that for any $\by\in\mathbb{R}^n$,
\begin{equation*}
    \bT^t\by=\sum_{i=0}^{n-1}\lambda_i^t\langle\mathbf{v}_i,\by\rangle_2\mathbf{v}_i.
\end{equation*}
Therefore, by Lemma \ref{lem: polynomialproperties} and the relationship $\mathbf{u}_i=\bM^{-1}\mathbf{v}_i,$ we have 
\begin{equation}\label{eqn: Kt}
    \bK^t\bx=\bM^{-1}\bT^t(\bM \bx)=\sum_{i=0}^{n-1}\lambda_i^t\langle\mathbf{v}_i,\bM\bx\rangle_2\bM^{-1}\mathbf{v}_i=\sum_{i=0}^{n-1}\lambda_i^t\langle\mathbf{v}_i,\bM\bx\rangle_2\mathbf{u}_i.
\end{equation}
Since $\bv_0=\frac{\bd^{1/2}}{\|\bd^{1/2}\|_2},$ and $\bu_i=\bM^{-1}\bv_i,$ the assumption that $\bM=\bD^{1/2}$ implies that $\bu_0=\frac{1}{\|\bd^{1/2}\|_2}\mathbf{1}.$ Therefore, $\Pi\bu_0=\bu_0,$ and so
\begin{align*}
    \Pi\bK^t\bx-\bK^t\bx&=\sum_{i=0}^{n-1}\lambda_i^t\langle\mathbf{v}_i, \bD^{1/2} \bx\rangle_2(\Pi\mathbf{u}_i-\bu_i)=\sum_{i=1}^{n-1}\lambda_i^t\langle\mathbf{v}_i, \bD^{1/2} \bx\rangle_2(\Pi\mathbf{u}_i-\bu_i)\\
    &=(\Pi-\bI)\left(\sum_{i=1}^{n-1}\lambda_i^t\langle\mathbf{v}_i, \bD^{1/2} \bx\rangle_2\mathbf{u}_i\right).
\end{align*}
Therefore, since $\{\bu_0,\ldots,\bu_{n-1}\}$ forms an orthonormal basis for $\mathbf{L}^2(\bD^{1/2}),$ we have by Parseval's identity
\begin{align}
   \| \Pi \Phi \bx - \Phi \bx \|_{\bD^{1/2}}^2 &= \|\Pi\bK^t\bx-\bK^t\bx\|_{\bD^{1/2}}^2 \nonumber \\
   &\leq \|\Pi-\bI\|^2_{\bD^{1/2}}
   \left\|\sum_{i=1}^{n-1}\lambda_i^t\langle\mathbf{v}_i, \bD^{1/2} \bx\rangle_2\mathbf{u}_i\right\|_{\bD^{1/2}}^2 \nonumber\\
   &=  \|\Pi-\bI\|_{\bD^{1/2}}^2 \sum_{i=1}^{n-1} \lambda_i^{2t}|\langle\mathbf{v}_i,\bD^{1/2} \bx\rangle_2|^2\nonumber\\
   &\leq  \|\Pi-\bI\|_{\bD^{1/2}}^2 \lambda_1^{2t} \sum_{i=1}^{n-1}|\langle\mathbf{v}_i, \bD^{1/2} \bx\rangle_2|^2 \nonumber\\
   &\leq  \|\Pi-\bI\|_{\bD^{1/2}}^2 \lambda_1^{2t} \|\bD^{1/2} \bx\|_2^2\nonumber\\
   &= \|\Pi-\bI\|_{\bD^{1/2}}^2 \lambda_1^{2t}\|\bx\|_{\bD^{1/2}}^2.\label{eqn: PiKtbound}
\end{align}
To bound $\|\bU\bx\|_{\bM},$ we note that by Theorem \ref{thm: energydecay} we have for any $\bM$,
\begin{align*}
    \|\bU\bx\|^2_{\bM} &= \|\bx\|_\bM^2 + \left( \sum_{m=0}^\infty\sum_{\pathvar\in\mathcal{J}^{m+1}}\|\bU[\pathvar]\mathbf{x}\|_\bM^2\right) \\
    &\leq \| \bx \|_{\bM}^2 + \sum_{m=0}^{\infty} \left( 1 - \frac{\bd_{\min}}{\| \bd \|_1} \right)^m \| \bx \|_{\bM}^2 \\
    &= \| \bx \|_{\bM}^2 + \| \bx \|_{\bM}^2 \sum_{m=0}^{\infty} \left( 1 - \frac{\bd_{\min}}{\| \bd \|_1} \right)^m \\
    &= \| \bx \|_{\bM}^2 + \frac{ \| \bd \|_1 }{\bd_{\min}} \| \bx \|_{\bM}^2 \\
    &= \left( 1 + \frac{\| \bd \|_1}{\bd_{\min}} \right) \| \bx \|_{\bM}^2.
\end{align*}
Combining this bound with \eqref{eqn: reducetoU} and \eqref{eqn: PiKtbound}  completes the proof.
\end{proof}

\section{The Proof of Theorem \ref{thm: wavelet stability1} }
\label{sec: proof of wavelet stability 1}

The proof of Theorem \ref{thm: wavelet stability1}, relies on the following two lemmas. Lemma \ref{lem: stabilitylipschitzfilters} is closely related to Theorem 11  of \cite{levie2019transferability}. (See also Theorem 5.) However, since the functions $q_j$ are not Lipschitz continuous on $[0,1],$ we can not directly apply the result from \cite{levie2019transferability}. Instead, Lemma \ref{lem: stabilitylipschitzfilters} separates out the lead eigenvector $\bv_0$ and uses the Lipschitz continuity of the $q_j$ on $[0,\lambda_1^*].$ Lemma \ref{lem: vvprime},  which bounds $\|\bv_0-\widetilde\bv_0\|_2$, is a restatement of  Lemma 5.2 of \cite{gama:diffScatGraphs2018}. 

\begin{lemma}\label{lem: stabilitylipschitzfilters}
Let $p(t)$ be a polynomial such that $0\leq p(t)\leq 1$ for all $0\leq t\leq 1$. Let $q(t)$ be the function defined by $q(t)=p(t)^{1/2}$ for $0\leq t\leq 1.$ Suppose $G=(V,E,W)$ and $\widetilde G=(\widetilde V,\widetilde E,\widetilde W)$ are weighted, connected graphs with $|V|=|\widetilde V|=n,$ and let $\lambda_1^*=\max\{\lambda_1,\widetilde \lambda_1\}$.  Then,  
\begin{equation*}
    \|q(\bT)-q(\widetilde{\bT})\|^2_2 \leq C n \left[ \| q' \|_{\mathbf{L}^{\infty}[0, \lambda_1^{\ast}]}^2 \| \bT - \widetilde{\bT} \|_2^2 + \left(\| q \|_{\mathbf{L}^{\infty}[0,1]}^2 + \| q' \|_{\mathbf{L}^{\infty}[0, \lambda_1^{\ast}]}^2 \right) \| \widetilde{\bv}_0 - \bv_0 \|_2^2  \right] 
\end{equation*}
where $C$ is an absolute constant.
\end{lemma}

\begin{lemma}\label{lem: vvprime}
Suppose $G=(V,E,W)$ and $\widetilde G=(\widetilde V,\widetilde E,\widetilde W)$ are weighted, connected graphs with $|V|=|\widetilde V|=n.$ Let $\lambda_1^*=\max\{\lambda_1,\widetilde \lambda_1\}$. Then,
\begin{equation*}
    \|\bv_0-\widetilde{\bv}_0\|_2^2\leq 2 \frac{\|\bT-\widetilde\bT\|_2}{1-\lambda_1^*}.
\end{equation*}
\end{lemma}


We now prove Theorem \ref{thm: wavelet stability1}.  Proofs of Lemmas \ref{lem: stabilitylipschitzfilters} and \ref{lem: vvprime}, are provided in  Appendix \ref{sec: proofs of lemmas stability and vvprime}. 

\begin{proof}[The proof of Theorem \ref{thm: wavelet stability1}]
 
By Lemma \ref{lem: stabilitylipschitzfilters}, we have 
\begin{align*}
    \|\cW^{(1)}_J-\widetilde{\cW}^{(1)}_J\|_{2}^2 &= \sum_{j=0}^{J+1}\|q_j(\bT)-q_j(\widetilde{\bT})\|_2^2 \\ 
    &\leq C n\sum_{j=0}^{J+1} \left[\| q_j' \|_{\mathbf{L}^{\infty}[0, \lambda_1^{\ast}]}^2 \|\bT-\widetilde{\bT}\|^2_2 + \left( \|q_j\|_{\mathbf{L}^\infty[0,1]}^2 + \| q_j' \|_{\mathbf{L}^{\infty}[0, \lambda_1^{\ast}]}^2 \right) \|\bv_0-\widetilde{\bv}_0\|_2^2 \right].
\end{align*}
 
By Lemma \ref{lem: vvprime}, we have 
\begin{equation*}
    \|\bv_0-\widetilde{\bv}_0\|_2^2\leq C_{\lambda_1^*}\|\bT-\widetilde{\bT}\|_2,
\end{equation*}
and it is immediate from the definition of $q_j(t)$ that $\|q_j(t)\|_{\mathbf{L}^\infty([0,1])}\leq 1.$
Therefore,
\begin{align}
    \|\cW^{(1)}_J-\widetilde{\cW}^{(1)}_J\|_{2}^2 &\leq Cn \sum_{j=0}^{J+1} \left[ \|q_j'\|_{\mathbf{L}^{\infty}[0, \lambda_1^{\ast}]}^2 \| \bT - \widetilde{\bT} \|_2^2 + C_{\lambda_1^{\ast}} \left( 1 + \|q_j'\|_{\mathbf{L}^{\infty}[0, \lambda_1^{\ast}]}^2 \right) \| \bT - \widetilde{\bT} \|_2 \right] \nonumber \\
    &= C_{\lambda_1^{\ast}} n \| \bT - \widetilde{\bT} \|_2 \left[ \sum_{j=0}^{J+1} \|q_j'\|_{\mathbf{L}^{\infty}[0, \lambda_1^{\ast}]}^2 ( \| \bT - \widetilde{\bT} \|_2 + 1)  + J \right] \nonumber \\
    &\leq C_{\lambda_1^{\ast}} n \| \bT - \widetilde{\bT} \|_2 \left[ J + \sum_{j=0}^{J+1} \|q_j'\|_{\mathbf{L}^{\infty}[0, \lambda_1^{\ast}]}^2 \right],  \label{eqn: plub back in to here}
\end{align}
where in the last inequality we used 
\begin{equation} \label{eqn: T is contractive} 
    \|\bT-\widetilde{\bT}\|_2\leq \|\bT\|_2+\|\widetilde{\bT}\|_2=\lambda_0+\widetilde{\lambda}_0=2.
\end{equation}

We thus need to bound $\|q_j'\|_{\mathbf{L}^{\infty}[0, \lambda_1^{\ast}]}^2$ for each $j = 0, \ldots, J+1$. When $j=0,$ we have 
\begin{equation*}
    |q_0'(t)|=\left|\frac{d}{dt}\sqrt{1-t}\right|=\frac{1}{2}\frac{1}{\sqrt{1-t}}\leq C_{\lambda_1^*}\quad\text{for all } 0\leq t\leq \lambda_1^*.
\end{equation*}
Likewise, for $j=J+1,$ we have 
  $  |q_{J+1}'(t)|=\left|\frac{d}{dt}t^{2^{J-1}}\right|\leq 2^{J-1}\quad\text{for all } 0\leq t\leq \lambda_1^*.$ 
For $1\leq j\leq J,$ we may write $q_j(t)=q_1(u_j(t)),$ where $u_j(t)=t^{2^{j-1}},$ and use the fact that $|u_j(t)|\leq 1$ 
to see
\begin{equation*}
    |q_j'(t)|=|q_1'(u_j(t))u_j'(t)|=\left|\frac{1-2u_j(t)}{2{t-t^2}}2^{j-1}t^{2^{j-1}-1}\right|\leq 2^{j-1}\frac{t^{2^{j-1}}}{\sqrt{t-t^2}}\leq C_{\lambda_1^*}2^{j-1}
\end{equation*}
for all $0\leq t\leq \lambda_1^*.$ Therefore, 
\begin{equation*}
    \sum_{j=0}^{J+1} \|q_j'\|_{\mathbf{L}^{\infty}[0, \lambda_1^{\ast}]}^2 \leq C_{\lambda_1^*}\left(1+2^J+\sum_{k=1}^{J}2^k\right)\leq C_{\lambda_1^*}2^J.
\end{equation*}
Plugging this in to \eqref{eqn: plub back in to here} completes the proof. 
\end{proof}

\section{The proof of Lemmas \ref{lem: stabilitylipschitzfilters} and \ref{lem: vvprime}}\label{sec: proofs of lemmas stability and vvprime}

\begin{proof}[The Proof of Lemma \ref{lem: stabilitylipschitzfilters}]
 We first note that for any $\bx\in \mathbb{R}^n$ we have 
 \begin{equation*}
     q(\bT)\bx = \bV q(\Lambda) \bV^T \bx =\sum_{i=0}^{n-1}q(\lambda_i)\langle \bv_i,\bx\rangle_2 \bv_i,
 \end{equation*}
and similarly $q(\widetilde\bT)\bx = \sum_{i=0}^{n-1}q(\widetilde{\lambda}_i) \langle \widetilde{\bv}_i,\bx\rangle_2 \widetilde{\bv}_i$. 
 Therefore, it suffices to show that 
\begin{equation}\label{eqn: bound for each vk} \|q(\bT)\bv_k-q(\widetilde\bT)\bv_k\|_2^2 \leq C  \left[ \| q' \|_{\mathbf{L}^{\infty}[0, \lambda_1^{\ast}]}^2 \| \bT - \widetilde{\bT} \|_2^2 + \left(\| q \|_{\mathbf{L}^{\infty}[0,1]}^2 + \| q' \|_{\mathbf{L}^{\infty}[0, \lambda_1^{\ast}]}^2 \right) \| \widetilde{\bv}_0 - \bv_0 \|_2^2  \right]
\end{equation} 

for all $0\leq k \leq n-1.$ The result will then follow by expanding an arbitrary $\bx$ as 
\begin{equation*}
    \bx= \sum_{i=0}^{n-1}\langle \bx,\bv_k\rangle_2 \bv_k.
\end{equation*}

We first consider the case where $1\leq k \leq n-1.$ Since $\{\bv_0, \ldots, \bv_{n-1} \}$ and $\{\widetilde{\bv}_0,\ldots, \widetilde{\bv}_{n-1}\}$ are orthonormal bases for $\mathbb{R}^n$, we have 
\begin{align*}
    q(\bT)\bv_k-q(\widetilde\bT)\bv_k &= \sum_{i=0}^{n-1}q(\lambda_i)\langle \bv_i,\bv_k\rangle_2 \bv_i - \sum_{i=0}^{n-1}q(\widetilde{\lambda}_i)\langle \widetilde{\bv}_i,\bv_k\rangle_2 \widetilde{\bv}_i \\
    &= q(\lambda_k)\bv_k - \sum_{i=0}^{n-1}q(\widetilde{\lambda}_i)\langle \widetilde{\bv}_i,\bv_k\rangle_2 \widetilde{\bv}_i \\
    &=\sum_{i=0}^{n-1}q(\lambda_k)\langle \widetilde{\bv}_i,\bv_k\rangle_2 \widetilde{\bv}_i - \sum_{i=0}^{n-1}q(\widetilde{\lambda}_i)\langle \widetilde{\bv}_i,\bv_k\rangle_2 \widetilde{\bv}_i\\
    &=\sum_{i=0}^{n-1}\left(q(\lambda_k)-q(\widetilde{\lambda}_i)\right)\langle \widetilde{\bv}_i,\bv_k\rangle_2 \widetilde{\bv}_i.
\end{align*}

Therefore, twice applying Parseval's theorem, we have 
\begin{align*}
\|q(\bT)\bv_k &- q(\widetilde\bT)\bv_k\|_2^2 = \\ 
&= \sum_{i=0}^{n-1}\left|q(\lambda_k)-q(\widetilde{\lambda}_i)\right|^2|\langle \widetilde{\bv}_i,\bv_k\rangle_2|^2\\
&\leq \sup_{0\leq t \leq \lambda_1^*}|q'(t)|^2\sum_{i=1}^{n-1}|\lambda_k-\widetilde{\lambda}_i|^2 |\langle \widetilde{\bv}_i,\bv_k\rangle_2|^2 + 4\|q\|^2_{\mathbf{L}^{\infty}[0,1]} |\langle \widetilde{\bv}_0,\bv_k\rangle_2|^2\\
&=\sup_{0\leq t \leq \lambda_1^*}|q'(t)|^2\left\|\sum_{i=1}^{n-1}(\lambda_k-\widetilde{\lambda}_i)\langle \widetilde{\bv}_i,\bv_k\rangle_2 \widetilde{\bv}_i\right\|^2_2+ 4\|q\|^2_{\mathbf{L}^{\infty}[0,1]} |\langle \widetilde{\bv}_0,\bv_k\rangle_2|^2\\
&=\sup_{0\leq t \leq \lambda_1^*}|q'(t)|^2\left\|\sum_{i=1}^{n-1}\lambda_k\langle \widetilde{\bv}_i,\bv_k\rangle_2 \widetilde{\bv}_i-\sum_{i=1}^{n-1}\widetilde{\lambda}_i\langle \widetilde{\bv}_i,\bv_k\rangle_2 \widetilde{\bv}_i\right\|^2_2+ 4\|q\|^2_{\mathbf{L}^{\infty}[0,1]} |\langle \widetilde{\bv}_0,\bv_k\rangle_2|^2\\
&=\sup_{0\leq t \leq \lambda_1^*}|q'(t)|^2\left\|\lambda_k\bv_k - \lambda_k \langle \widetilde{\bv}_0, \bv_k \rangle_2 \widetilde{\bv}_0 - \sum_{i=0}^{n-1}\widetilde{\lambda}_i \langle \widetilde{\bv}_i,\bv_k\rangle_2 \widetilde{\bv}_i + \widetilde{\lambda}_0 \langle \widetilde{\bv}_0, \bv_k \rangle_2 \widetilde{\bv}_0 \right\|_2^2 \\
&\qquad + 4\|q\|^2_\infty |\langle \widetilde{\bv}_0,\bv_k\rangle_2|^2\\
&\leq 2 \sup_{0\leq t \leq \lambda_1^*}|q'(t)|^2 \left[ \left\| \bT \bv_k - \widetilde{\bT} \bv_k \right\|_2^2 + (1 - \lambda_k)^2 |\langle \widetilde{\bv}_0, \bv_k \rangle_2|^2 \right] + 4\|q\|^2_{\mathbf{L}^{\infty}[0,1]} |\langle \widetilde{\bv}_0,\bv_k\rangle_2|^2 \\
&\leq 2 \sup_{0\leq t \leq \lambda_1^*}|q'(t)|^2 \left[ \left\| \bT \bv_k - \widetilde{\bT} \bv_k \right\|_2^2 + |\langle \widetilde{\bv}_0, \bv_k \rangle_2|^2 \right] + 4\|q\|^2_{\mathbf{L}^{\infty}[0,1]} |\langle \widetilde{\bv}_0,\bv_k\rangle_2|^2. 
\end{align*}
Noting that for all $k\geq 1,$
\begin{equation*}
    |\langle \widetilde{\bv}_0, \bv_k\rangle| =|\langle \widetilde{\bv}_0-\bv_0,\bv_k\rangle+\langle\bv_0,\bv_k\rangle| =|\langle \widetilde{\bv}_0-\bv_0,\bv_k\rangle| \leq \|\widetilde{\bv}_0-\bv_0\|_2,
\end{equation*}
we obtain
\begin{equation*}
    \|q(\bT)\bv_k - q(\widetilde\bT)\bv_k\|_2^2 \leq 2 \sup_{0 \leq t \leq \lambda_1^{\ast}} |q'(t)|^2 \left[ \left\| \bT \bv_k - \widetilde{\bT} \bv_k \right\|_2^2 + \| \widetilde{\bv}_0 - \bv_0 \|_2^2 \right] + 4 \| q \|_{\mathbf{L}^{\infty}[0,1]}^2 \| \widetilde{\bv}_0 - \bv_0 \|_2^2, 
\end{equation*}
which completes the proof of \eqref{eqn: bound for each vk} in the case where $k\geq 1.$ 

Now turning our attention to the case where $k=0,$ we can use the fact that $\lambda_0=\widetilde{\lambda_0}=1$ to see that 
\begin{align*}
    \|q(\bT)\bv_0-q(\widetilde{\bT})\bv_0\|_2^2 &= \left\|q(\lambda_0)\bv_0 - \sum_{i=0}^{n-1}q(\widetilde{\lambda}_i)\langle \widetilde{\bv}_0,\bv_0\rangle_2 \widetilde{\bv}_i\right\|_2^2\\
    &\leq2\left\|q(1)(\bv_0- \langle \widetilde{\bv}_0,\bv_0\rangle_2 \widetilde{\bv}_0)\right\|_2^2+2\left\| \sum_{i=1}^{n-1}q(\widetilde{\lambda}_i)\langle \widetilde{\bv}_i,\bv_0\rangle_2 \widetilde{\bv}_i\right\|_2^2.
\end{align*}
To bound the first term we note that 
\begin{align*}
\left\|q(1)(\bv_0- \langle \widetilde{\bv}_0,\bv_0\rangle_2 \widetilde{\bv}_0)\right\|_2^2 &\leq \|q\|_{\mathbf{L}^{\infty}[0,1]}^2 \left\| \sum_{i=1}^{n-1}\langle \widetilde{\bv}_i,\bv_0\rangle_2 \widetilde{\bv}_i\right\|_2^2 \\
&\leq \|q\|_{\mathbf{L}^{\infty}[0,1]}^2 \sum_{i=1}^{n-1}|\langle \widetilde{\bv}_i,\bv_0\rangle_2|^2\\
&= \|q\|_{\mathbf{L}^{\infty}[0,1]}^2 \sum_{i=1}^{n-1}|\langle \widetilde{\bv}_i,\bv_0-\widetilde{\bv}_0\rangle_2|^2\\
&\leq \|q\|^2_{\mathbf{L}^{\infty}[0,1]} \|\bv_0-\widetilde{\bv}_0\|_2^2.
\end{align*}
Similarly, we note that 
\begin{align*}
    \left\| \sum_{i=1}^{n-1}q(\widetilde{\lambda}_i)\langle \widetilde{\bv}_i,\bv_0\rangle_2 \widetilde{\bv}_i\right\|_2^2\leq \|q\|^2_{\mathbf{L}^{\infty}[0,1]} \left\| \sum_{i=1}^{n-1}\langle \widetilde{\bv}_i,\bv_0\rangle_2 \widetilde{\bv}_i\right\|_2\leq \|q\|^2_{\mathbf{L}^{\infty}[0,1]} \|\bv_0-\widetilde\bv_0\|_2^2
\end{align*}
and therefore the proof is complete.
\end{proof}

\begin{proof}[The Proof of Lemma \ref{lem: vvprime}]
Let $\bv = \bv_0$ and $\widetilde{\bv} = \widetilde{\bv}_0$ and let $\alpha=\langle\bv,\widetilde\bv\rangle_2.$ Since $\bT=\bv\bv^T+\overline{\bT}$ and $\widetilde\bT=\widetilde\bv \widetilde\bv^T+\overline{\widetilde\bT}$, with $\overline{\bT}\bv=\alpha\overline{\widetilde\bT}\widetilde\bv = \bm{0},$ we see
\begin{align*}
    (\bT-\widetilde\bT)\bv &= (\bv\bv^T-\widetilde\bv \widetilde\bv^T)\bv+(\overline{\bT}-\overline{\widetilde\bT})\bv\\
    &=\bv-\alpha\widetilde\bv-\overline{\widetilde\bT}\bv-\alpha\overline{\widetilde\bT}\widetilde\bv\\
    &=(\bI-\overline{\widetilde\bT})(\bv-\alpha\widetilde\bv).
\end{align*}
 For any $\bx \in \R^n$, one has $\left\|(\bI-\overline{\widetilde\bT}) \bx \right\|_2 \geq (1-\widetilde{\lambda}_1) \| \bx \|_2 \geq (1-\lambda_1^*) \| \bx \|_2,$ and $\left\|\bv-\alpha\widetilde\bv\right\|_2\geq 1-\alpha.$ Therefore, 
\begin{equation*}
    (1-\alpha)(1-\lambda_1^*) \leq (1 - \lambda_1^{\ast}) \| \bv - \alpha \widetilde{\bv} \|_2 \leq   \left\|\left(\bI-\overline{\widetilde\bT}\right)(\bv-\alpha\widetilde\bv)\right\|_2\leq \left\|\left(\bT-\widetilde\bT\right)\bv\right\|_2\leq \left\|\bT-\widetilde\bT\right\|_2.
\end{equation*}
Therefore, using $\left\|\bv-\widetilde\bv\right\|_2^2=2(1-\alpha)$, we have
\begin{equation*}
\left\|\bv-\widetilde\bv\right\|^2_2\leq 2 \frac{\left\|\bT-\widetilde\bT\right\|_2}{1-\lambda_1^*}.
\end{equation*}
\end{proof}

\section{The Proof of Theorem \ref{thm: wavelet stability2}}\label{sec: The Proof of wavelet 2 stability}

\begin{proof} 
Let  $\bv\coloneqq\mathbf{v}_0$ denote the lead eigenvector of $\bT.$ 
Since $\lambda_1<1,$ we may write 
\begin{equation*}
    \bT=\bv\bv^T+\overline{\bT},
\end{equation*} 
 where 
\begin{equation*}
    \|\overline{\bT}\|_2 =\lambda_1<1,
\end{equation*}
and $\overline{\bT}\bv=0.$
Since $\bv_0,\bv_1,\ldots,\bv_{n-1}$ form an orthonormal basis for $\mathbf{L}^2,$ we have that 
\begin{equation}\label{eqn: powertrick}
    \bT^k=\bv\bv^T+\overline{\bT}^k
\end{equation}
for all $k\geq 1.$ 
Therefore, we see that 
\begin{equation*}
    \Psi_j = \begin{cases}
    \bI-\bT= \bI-\bv\bv-\overline{\bT}&\text{ for } j=0\\
    \bT^{2^{j-1}}-\bT^{2^j}= \overline{\bT}^{2^{j-1}}-\overline{\bT}^{2^j}&\text{ for } 1\leq j \leq J
    \end{cases},
\end{equation*}
and
\begin{equation*}
    \Phi_J = 
     \bT^{2^J} =\bv\bv^T+\overline{\bT}^{2^J},
\end{equation*}
with  similar equations being valid for $\widetilde{\Psi}_j$ and $\widetilde{\Phi}_J.$ Thus,
\begin{align}
    \left\|\mathcal{W}_J^{(2)}-\widetilde{\mathcal{W}}_J^{(2)}\right\|_2^2 &= \|\Phi_J-\widetilde{\Phi}_J\|_2^2+\sum_{j=0}^J \|\Psi_j-\widetilde{\Psi}_j\|_2^2\nonumber\\   &\leq 4\left(\|\bv\bv^T-\widetilde{\bv}\widetilde{\bv}^T\|_2^2+\sum_{j=0}^J\left\|\overline{\bT}^{2^j}-\left(\overline{\widetilde{\bT}}\right)^{2^j}\right\|^2_2\right)\label{eqn: ell2psiestimate}.
\end{align}
The following lemma follows by imitating equation (23) of \cite{gama:diffScatGraphs2018} and summing over $j.$ 
\begin{lemma}\label{lem: eqn23ofGama}
\begin{equation*}
    \sum_{j=0}^J\left\|\overline{\bT}^{2^j}-\left(\overline{\widetilde{\bT}}\right)^{2^j}\right\|^2_2 \leq C_{\lambda_1^*} \left\|\overline{\bT}-\overline{\widetilde{\bT}}\right\|^2_2.
\end{equation*}
\end{lemma}

\begin{proof}
Let 
\begin{equation*}
    \mathbf{H}_j(t)=\left(t\overline{\bT}+(1-t)\overline{\widetilde{\bT}}\right)^{2^j}.
\end{equation*}
Then
\begin{equation*}
    \left\|\overline{\bT}^{2^j}-\left(\overline{\widetilde{\bT}}\right)^{2^j}\right\|_2 = \|\mathbf{H}_j(1)-\mathbf{H}_j(0)\|_2\leq \int_0^1\|\mathbf{H}'_j(t)\|_2 \, dt \leq \sup_{0\leq t\leq1}\|\mathbf{H}_j'(t)\|_2.
\end{equation*}
Since 
\begin{equation*}
    \mathbf{H}_j'(t) = \sum_{\ell=0}^{2^j-1} \left(t\overline{\widetilde{\bT}}+(1-t)\overline{\bT}\right)^\ell  \left(\overline{\bT} - \overline{\widetilde{\bT}} \right) \left(t\overline{\widetilde{\bT}}+(1-t)\overline{\bT}\right)^{2^j-\ell-1},
\end{equation*}
and
$\|\overline{\bT}\|_2,\|\overline{\widetilde{\bT}}\|_2\leq \lambda_1^*,$ this implies 
\begin{equation*}
    \|\mathbf{H}_j'(t)\|_2\leq 2^j\left(\lambda_1^*\right)^{2^j-1}\left\|\overline{\bT}-\overline{\widetilde{\bT}}\right\|^2_2.
\end{equation*}
Therefore,
\begin{align*}
    \sum_{j=0}^J\left\|\overline{\bT}^{2^j}-\left(\overline{\widetilde{\bT}}\right)^{2^j}\right\|^2_2 
    &\leq \left\|\overline{\bT}-\overline{\widetilde{\bT}}\right\|^2_2\sum_{j=0}^\infty 2^{2j}\left(\lambda_1^*\right)^{2^{j+1}-2}\\
    &=C_{\lambda_1^*}\left\|\overline{\bT}^{2^j}-\left(\overline{\widetilde{\bT}}\right)^{2^j}\right\|^2_2, 
\end{align*}
where $C_{\lambda_1^{\ast}} < \infty$ since $\lambda_1^*<1.$
\end{proof}

By the triangle inequality,
\begin{equation*}
\left\|\overline{\bT}-\overline{\widetilde{\bT}}\right\|_2     \leq \left\|\bT-\widetilde{\bT}\right\|_2 + \left\|\bv\bv^T-\widetilde\bv\widetilde{\bv}^T\right\|_2.
\end{equation*}
Therefore, by \eqref{eqn: ell2psiestimate}, we have
\begin{equation}\label{eqn:secondell2psiestimate}
    \left\|\mathcal{W}_J^{(2)}-\widetilde{\mathcal{W}}_J^{(2)}\right\|_2^2\leq C_{\lambda_1^*}\left(\left\|\bT-\widetilde{\bT}\right\|^2_2 + \left\|\bv\bv^T-\widetilde\bv\widetilde{\bv}^T\right\|^2_2\right).
\end{equation}
To bound $\left\|(\bv\bv^T-\widetilde\bv\widetilde{\bv}^T)\bx\right\|_2$, we note that for all $\bx$ 
\begin{align*}
    \left\|\bv\bv^T-\widetilde\bv\widetilde{\bv}^T\right\|_2 &\leq \left\|(\bv(\bv-\widetilde\bv)^T\bx\right\|_2 + \left\|(\bv-\widetilde\bv)\widetilde{\bv}^T\bx\right\|_2 \\
    &\leq \|\bv\|_2\left\|\bv-\widetilde\bv\right\|_2\|\bx\|_2+\left\|\bv-\widetilde\bv\right\|_2\|\widetilde\bv\|_2\|\bx\|_2\\
    &\leq 2\left\|\bv-\widetilde\bv\right\|_2\|\bx\|_2
\end{align*}
with the last equality following from the fact that  
$\|\bv\|_2=\|\widetilde\bv\|_2=1.$ 
  Therefore,  the result now follows from applying Lemma \ref{lem: vvprime} and then \eqref{eqn: T is contractive}.
\end{proof}

\section{The proof of Theorem \ref{thm: transferTtoP}  }\label{sec: proof of general K}
\begin{proof}
Let $\|\bx\|_\bM=1,$ and let $\by=\bM\bx.$ Note that $\|\by\|_2=\|\bM\bx\|_2=\|\bx\|_\bM=1.$
By Lemma \ref{lem: polynomialproperties} and by \eqref{eqn: defsquareroot} we have that for all $i\in\mathcal{I}$
\begin{equation*}
    r_j(\bK)=\bM^{-1}r_j(\bT)\bM\quad\text{and}\quad r_j(\widetilde \bK)=(\widetilde \bM)^{-1}r_j(\widetilde \bT)\widetilde \bM.
\end{equation*}
Therefore,
\begin{align*}
   \|(r_j(\bK) &- r_j(\widetilde \bK))\bx\|_\bM = \\ 
   &= \|\bM\left(\bM^{-1}r_j(\bT)\bM-(\widetilde \bM)^{-1}r_j(\widetilde \bT)\widetilde \bM\right)\bx\|_2\\
    &=\|r_j(\bT)\bM\bx-\bM(\widetilde \bM)^{-1}r_j(\widetilde \bT)\widetilde \bM\bM^{-1}\bM\bx\|_2\\
    &=\|r_j(\bT)\by-\bR_2^{-1} r_j(\widetilde \bT)\bR_2\by\|_2\\
    &\leq \|r_j(\bT)\by-\bR_2^{-1} r_j(\widetilde \bT)\by\|_2+\|\bR_2^{-1} r_j(\widetilde\bT )\by-\bR_2^{-1} r_j(\widetilde \bT)\bR_2\by\|_2\\
    &\leq\|(r_j(\bT)-r_j(\widetilde \bT))\by\|_2+\|(\bI-\bR_2^{-1})r_j(\widetilde \bT)\by\|_2+\|\bR^{-1}_2\|_2\|r_j(\widetilde \bT)(\bI-\bR_2)\by\|_2\\
    &\leq\|(r_j(\bT)-r_j(\widetilde \bT))\by\|_2+\kappa(\bM,\widetilde \bM)\|r_j(\widetilde {\bT})\by\|_2+R(\bM,\widetilde \bM)\|r_j(\widetilde \bT)(\bI-\bR_2^{-1})\by\|_2.
\end{align*}
Therefore, squaring both sides, summing over $j,$ and using the nonexpansiveness of $\cW^{\bT}$ and the fact that $\|\by\|_2=1,$ 
we have 
\begin{align*}
    &\sum_{j\in\mathcal{J}}\|(r_j(\bK)-r_j(\widetilde \bK))\bx\|_\bM^2\\
    \leq& 3\left(\sum_{j\in\mathcal{J}}\|(r_j(\bT)-r_j(\widetilde \bT))\by\|_2^2 + \kappa(\bM,\widetilde \bM)^2\sum_{j\in\mathcal{J}}\|r_j(\widetilde \bT)\by\|^2_2
    +R(\bM,\widetilde \bM)^2\sum_{j\in\mathcal{J}}\|r_j(\widetilde \bT)(\bI-\bR_2)\by\|^2_2\right)\\
    \leq& 3\left(\left\|\cW^{\bK}-\cW^{\widetilde \bK}\right\|_{\bM}^2+\kappa(\bM,\widetilde \bM)^2+R(\bM,\widetilde \bM)^2\|(\bI-\bR_2)\by\|_2^2 \right)\\
    \leq& 3\left(\left\|\cW^{\bK}-\cW^{\widetilde \bK}\right\|_{\bM}^2+\kappa(\bM,\widetilde \bM)^2+R(\bM,\widetilde \bM)^2\kappa(\bM,\widetilde \bM)^2 \right)\\
    \leq& 6\left(\left\|\cW^{\bK}-\cW^{\widetilde \bK}\right\|_{\bM}^2+\kappa(\bM,\widetilde \bM)^2\big(\kappa(\bM,\widetilde \bM\big)+1)^2 \right)
\end{align*}
where the last inequality uses the fact that $R(\bM,\widetilde \bM)\leq (\kappa(\bM,\widetilde \bM)+1).$
\end{proof}

\section{The Proof of Proposition \ref{prop: Tdist to Pdist} }\label{sec: proofofTtoK}
\begin{proof}[The proof of Proposition \ref{prop: Tdist to Pdist}]
Let $\|\bx\|_2=1.$ Then, since $\bT=\bM\bK\bM^{-1}$
\begin{align}
    \|(\bT-\widetilde \bT)\bx\|_2 &= \|\bM\bK\bM^{-1} \bx - \widetilde \bM\widetilde \bK(\widetilde \bM)^{-1} \bx\|_2\nonumber\\
    &= \|\bM\bK\bM^{-1} \bx - (\bM\bM^{-1})\widetilde \bM\widetilde \bK(\widetilde \bM)^{-1}(\bM\bM^{-1}) \bx\|_2\nonumber\\
    &=  \|\bM(\bK - \bR_1^{-1}\widetilde \bK\bR_1)\bM^{-1} \bx\|_2\nonumber\\
    &= \|(\bK - \bR_1^{-1}\widetilde \bK\bR_1)\bM^{-1} \bx\|_\bM\nonumber\\
    &\leq \|\bK - \bR_1^{-1}\widetilde \bK\bR_1\|_\bM\|\bM^{-1} \bx\|_\bM\nonumber\\
    &= \|\bK - \bR_1^{-1}\widetilde \bK\bR_1\|_\bM, \label{eqn: middlebound}
\end{align}
since $\|\bM^{-1} \bx\|_\bM=\|\bx\|_2=1.$ By the triangle inequality,
\begin{align*}
    \|\bK - \bR_1^{-1}\widetilde \bK\bR_1\|_\bM & \leq \|\bK - \bR_1^{-1}\bK\|_\bM + \|\bR_1^{-1}\bK- \bR_1^{-1}\widetilde \bK\bR_1\|_\bM\\
    &\leq \|\bK\|_\bM\|\bI-\bR_1^{-1}\|_\bM + \|\bR_1^{-1}\|_\bM\|\bK-\widetilde \bK\bR_1\|_\bM\\
    &\leq \|\bK\|_\bM\|\bI-\bR_1^{-1}\|_\bM + \|\bR_1^{-1}\|_\bM\|\bK-\widetilde \bK\|_\bM+\|\bR_1^{-1}\|_\bM\|\widetilde \bK(\bI-\bR_1)\|_\bM\\
    &\leq \|\bK\|_\bM\|\bI-\bR_1^{-1}\|_\bM + \|\bR_1^{-1}\|_\bM\|\bK-\widetilde \bK\|_\bM+\|\bR_1^{-1}\|_\bM\|\widetilde \bK\|_\bM\|\bI-\bR_1\|_\bM\\
    &\leq \kappa(\bM,\widetilde \bM) + R(\bM,\widetilde \bM)\|\bK-\widetilde \bK\|_\bM+R(\bM,\widetilde \bM)R(\bM,\widetilde \bM)^2\kappa(\bM,\widetilde \bM)\\
    &=\kappa(\bM,\widetilde \bM)\left(1+R(\bM,\widetilde \bM)^{3}\right) + R(\bM,\widetilde \bM)\|\bK-\widetilde \bK\|_\bM,
\end{align*}
where we used the facts that $\|\bI-\bR_1^{\pm1}\|_\bM\leq \kappa(\bM,\widetilde \bM),$  $\|\bR_1^{\pm1}\|\leq R(\bM,\widetilde \bM),$   $\|\bK\|_\bM=1,$ and  $\|\widetilde \bK\|_\bM\leq \|\bR_1\|_2\|\bR^{-1}\|_2\leq R(\bM,\widetilde \bM)^2.$
\end{proof}

\section{The proof of Lemma \ref{lem: windowstability}}\label{sec: the proof of lem: window stability}

\begin{proof}

Let $\cA\coloneqq\|\cW- \widetilde{\cW}\|_{\bM}$ and $\mathcal{C}\coloneqq\|\widetilde{\cW}\|_{\bM}.$ 
To prove \eqref{eqn: scatteringstabilitynopermwindow}, we need to  show
\begin{equation}\label{eqn: korder}
    \sum_{\pathvar\in\mathcal{J}^\ell}\|\bS[\pathvar]\bx-\widetilde\bS[\pathvar]\bx\|^2_{\bM}\leq 2\mathcal{A}^2\cdot\left(\sum_{k=0}^\ell\mathcal{C}^{k}\right)^2\|\bx\|_{\bM}^2.
\end{equation} 
For $\ell=0$, the zeroth-order windowed scattering coefficient of $\bx$
is given by 
    $\bS[\pathvar_{e}]\bx=\Phi\bx$,
where $\pathvar_{e}$ is the empty-index.
Therefore, by the definition of $\cA$ we have
\begin{equation*}
    \sum_{\pathvar\in\mathcal{J}^0}\|\bS[\pathvar]\bx-\widetilde\bS[\pathvar]\bx\|_{\bM}^2    =\|\Phi\bx-\widetilde\Phi\bx\|_{\bM'}^2 \leq\|\mathcal{W}\bx-\widetilde{\mathcal{W}}\bx\|_{\bM}^2= \mathcal{A}^2\|\bx\|_{\bM}^2,
\end{equation*}
and so  \eqref{eqn: korder} holds when $\ell=0.$ For the case where $\ell\geq 1,$ 
we note that for all $\pathvar\in\mathcal{J}^\ell,$  we have 
\begin{align*}
    \|\bS[\pathvar]\bx-\widetilde\bS[\pathvar]\bx\|_{\bM}&=    \|\Phi \bU[\pathvar]\bx-\widetilde\Phi \widetilde\bU[\pathvar]\bx\|_{\bM}\\
    &\leq \|(\Phi-\widetilde\Phi)\bU[\pathvar]\bx\|_{\bM} +  \|\widetilde\Phi\bU[\pathvar]\bx-\widetilde\Phi\widetilde\bU[\pathvar]\bx\|_{\bM}\\
    &\leq \|\Phi-\widetilde\Phi\|_{\bM}\|\bU[\pathvar]\bx\|_{\bM} +  \|\widetilde\Phi\|_\bM\|\bU[\pathvar]\bx-\widetilde\bU[\pathvar]\bx\|_{\bM}.\\
    &= \|\Phi-\widetilde\Phi\|_{\bM}\|\bU[\pathvar]\bx\|_{\bM} +  \|\widetilde \Phi\|_{\bM}\|\bU[\pathvar]\bx-\widetilde\bU[\pathvar]\bx\|_{\bM},
\end{align*}
and so squaring both sides and summing over $\pathvar$ implies
\begin{align*}    
    \sum_{\pathvar\in\mathcal{J}^\ell}\|\bS[\pathvar]\bx-\widetilde\bS[\pathvar]\bx\|_{\bM}^2
    &\leq2\|(\Phi-\widetilde\Phi)\|_{\bM}^2\sum_{\pathvar\in\mathcal{J}^\ell}\|\bU[\pathvar]\bx\|_{\bM}^2 +2\|\widetilde\Phi\|_{\bM}^2\sum_{\pathvar\in\mathcal{J}^\ell}  \|\bU[\pathvar]\bx-\widetilde\bU[\pathvar]\bx\|_{\bM}^2. 
\end{align*}
Therefore, \eqref{eqn: korder} and thus \eqref{eqn: scatteringstabilitynopermwindow}, follow from applying Lemma \ref{lem: nonexpansiveU} stated below, noting that $\|\Phi-\widetilde\Phi\|_{\bM}^2\leq \cA^2$ and $\|\widetilde\Phi\|_{\bM}^2\leq \mathcal{C}^2$, and using the fact that $a^2 + b^2 \leq (a + b)^2$ when $a,b \geq 0$.  
For a proof of Lemma \ref{lem: nonexpansiveU}  please see Appendix \ref{sec: the proof of lemmas for scattering stability}.
\end{proof}
\begin{lemma}\label{lem: nonexpansiveU}
Let $\mathcal{A}\coloneqq\|\mathcal{W}-\widetilde{\mathcal{W}}\|_{\bM}$ and $\mathcal{C}\coloneqq \|\widetilde{\mathcal{W}}\|_{\bM}$. Then, for all $\ell\geq 1$,
\begin{align*}
    \sum_{\pathvar\in\mathcal{J}^\ell}\|\bU[\pathvar]\bx\|^2_{\bM}   &\leq \|\bx\|^2_{\bM},\quad\text{and}\quad 
    \sum_{\pathvar\in\mathcal{J}^\ell}\|\bU[\pathvar]\bx-\widetilde\bU[\pathvar]\bx\|^2_{\bM} \leq \mathcal{A}^2 \left(\sum_{k=0}^{\ell-1}\mathcal{C}^{k}\right)^2\|\bx\|^2_{\bM}.
\end{align*}
\end{lemma}

\section{The proof of Proposition \ref{thm: Cstability}}\label{sec: proofCStability}

\begin{proof}[The proof of Proposition \ref{thm: Cstability}]
By Lemma \ref{lem: polynomialproperties}, we have that 
\begin{equation*}
r_j(\widetilde \bK)=(\widetilde \bM)^{-1}r(\widetilde \bT)\widetilde \bM
\end{equation*}
for all $j\in\mathcal{J}$. Lemma \ref{lem: polynomialproperties} also implies that $\|r_j(\widetilde{\bT})\|_2=\|r_j(\widetilde{\bK})\|_\bM$. Therefore, we have 
\begin{align*}
    \|r_j(\widetilde \bK)\bx\|_\bM&=\|\bM((\widetilde \bM)^{-1}r_j(\widetilde \bT)\widetilde \bM)(\bM^{-1}\bM))\bx\|_2 \\
    &=\|\bR_2^{-1} r_j(\widetilde \bT)\bR_2\bM\bx\|_2\\
    &\leq \|\bR_2^{-1}\|_2\|r_j(\widetilde \bT)\|_2\|\bR_2\|_2\|\bM\bx\|_2\\
    &=R(\bM,\widetilde \bM)^2\|r_j(\widetilde \bT)\|_2 \|\bx\|_\bM \\
    & = R(\bM,\widetilde \bM)^2\|r_j(\widetilde{\bK})\|_{\bM} \|\bx\|_\bM.
\end{align*}
Therefore, the result follows by squaring both sides and summing over $j$.
\end{proof}

\section{The proof of Theorem \ref{thm: scattering stability no window}}\label{sec: proof of scattering stability no window}

\begin{proof}
Let $\Pi\in S_n$ be a permutation and let $G'=\Pi(G)$.
By Theorem \ref{thm: equivariance}, we have $\Pi\bS^\ell=(\bS^\ell)'\Pi$. 
Therefore, we 
may use the definitions of the nonwindowed scattering transform and the weighted inner product $\langle\cdot,\cdot\rangle_{\bM'}$ to see that for each path $\pathvar$ we have
\begin{align*}
    &|\overline{\bS}[\pathvar]  \bx - \widetilde{\overline{\bS}}[\pathvar] \widetilde{\bx}|\\=& |\overline{\bS'}[\pathvar] \Pi \bx - \widetilde{\overline{\bS}}[\pathvar] \widetilde{\bx}|\\ =& | \langle \bm{\mu}', \bU' [\pathvar] \Pi \bx \rangle_{\bM'} - \langle \widetilde{\bm{\mu}}, \widetilde{\bU} [\pathvar] \widetilde{\bx} \rangle_{\widetilde{\bM}}|\\
    \leq& | \langle \bm{\mu}', \bU' [\pathvar] \Pi \bx- \widetilde{\bU} [\pathvar] \widetilde{\bx}\rangle_{\bM'}|  + |
    \langle (\bm{\mu}'-\tilde{\bm{\mu}}),\widetilde{\bU} [\pathvar] \widetilde{\bx} \rangle_{\bM'}| +
    |\langle \widetilde{\bm{\mu}}, \widetilde{\bU} [\pathvar] \widetilde{\bx} \rangle_{\bM'}-\langle \widetilde{\bm{\mu}}, \widetilde{\bU} [\pathvar] \widetilde{\bx} \rangle_{\widetilde{\bM}}|\\
   =& | \langle \bm{\mu}', \bU' [\pathvar] \Pi \bx- \widetilde{\bU} [\pathvar] \widetilde{\bx}\rangle_{\bM'}|  + |
    \langle (\bm{\mu}'-\tilde{\bm{\mu}}),\widetilde{\bU} [\pathvar] \widetilde{\bx} \rangle_{\bM'}| +
    |\langle((\bM')^*\bM'-\widetilde{\bM}^*\widetilde{\bM}) \widetilde{\bm{\mu}}, \widetilde{\bU} [\pathvar] \widetilde{\bx} \rangle_{2}|\\
    =&\hspace{-.05in}:  I[\pathvar] + II[\pathvar] + III[\pathvar].
\end{align*}
To bound $I[\pathvar]$, we use the Cauchy Schwarz inequality to observe
\begin{align*}
    |\langle \bm{\mu}', \bU' [\pathvar] \Pi \bx- \widetilde{\bU} [\pathvar] \widetilde{\bx}\rangle_{\bM'}|
&\leq |\langle \bm{\mu}', \bU' [\pathvar] \Pi \bx- \bU'  [\pathvar] \widetilde{\bx}\rangle_{\bM'}| + |\langle \bm{\mu}', \bU' [\pathvar] \widetilde{\bx}- \widetilde{\bU} [\pathvar] \widetilde{\bx}\rangle_{\bM'}|\\
&\leq 
|\overline{\bS'}[\pathvar] \Pi \bx - \overline{\bS'}[\pathvar] \widetilde{\bx}| + \|\bmu'\|_{\bM'}\|\bU' [\pathvar] \widetilde{\bx}- \widetilde{\bU} [\pathvar] \widetilde{\bx}\|_{\bM'}
\end{align*}
 Therefore, applying Theorem \ref{thm: nonexpansive} and Lemma \ref{lem: nonexpansiveU} yields 
 \begin{align}
     \sum_{\pathvar\in\mathcal{J}^\ell} I[\pathvar]^2 &\leq \frac{2 \| (\bM')^{-1}\|_2^2}{n\min_i |\bu_0(i)|^2 }\|\Pi \mathbf{x}-\widetilde{\mathbf{x}}\|^2_{\bM'} + 2\|\bmu'\|_{\bM'}^2\|\mathcal{W}'-\widetilde{\mathcal{W}}\|_{\bM'}^2 \left(\sum_{k=0}^{\ell-1}\|\widetilde{W}\|_{\bM'}^k \right)^2\|\tilde{\bx}\|^2_{\bM'} \nonumber \\
     &\leq \frac{2 R(\bM, \mathbf{I})^4}{n\min_i |\bu_0(i)|^2 }\|\Pi \mathbf{x}-\widetilde{\mathbf{x}}\|^2_2 + 2 R(\bM, \mathbf{I})^2 \|\bmu\|_{\bM}^2\|\mathcal{W}'-\widetilde{\mathcal{W}}\|_{\bM'}^2 \left(\sum_{k=0}^{\ell-1}\|\widetilde{W}\|_{\bM'}^k \right)^2\|\tilde{\bx}\|^2_2 \label{eqn: Ip bound}
 \end{align}
For $II[\pathvar]$, we again use the Cauchy Schwarz inequality to see
\begin{align*}
    |\langle (\bm{\mu}'-\tilde{\bm{\mu}}),\widetilde{\bU} [\pathvar] \widetilde{\bx} \rangle_{\bM'}|
    &\leq R(\bM',\widetilde{\bM}) \|\bm{\mu}'-\tilde{\bm{\mu}}\|_{\bM'}\|\widetilde{\bU} [\pathvar] \widetilde{\bx} \|_{\widetilde{\bM}}.
\end{align*}
Therefore, again applying Lemma \ref{lem: nonexpansiveU} implies 
\begin{equation}\label{eqn: IIp bound}
     \sum_{\pathvar\in\mathcal{J}^\ell} II[\pathvar]^2 \leq R(\bM',\widetilde{\bM})^2 \|\bm{\mu}'-\tilde{\bm{\mu}}\|_{\bM'}^2 \|\widetilde{\bx}\|^2_{\widetilde{\bM}}.
\end{equation}
Lastly, to bound $III[\pathvar]$, we use the Cauchy Schwarz inequality once more, 
\begin{align*}
    |\langle((\bM')^*\bM'-\widetilde{\bM}^*\widetilde{\bM}) \widetilde{\bm{\mu}}, \widetilde{\bU} [\pathvar] \widetilde{\bx} \rangle_{\widetilde{2}}|&\leq \|(\bM')^*\bM'-\widetilde{\bM}^*\widetilde{\bM}\|_2 \|\widetilde{\bm{\mu}}\|_2\|\widetilde{\bU} [\pathvar] \widetilde{\bx}\|_2\\
    &\leq R(\widetilde{\bM},\bI)^2 \|(\bM')^*\bM'-\widetilde{\bM}^*\widetilde{\bM}\|_2 \|\widetilde{\bm{\mu}}\|_{\widetilde{\bM}}\|\widetilde{\bU} [\pathvar] \widetilde{\bx}\|_{\widetilde{\bM}},
\end{align*}
and so summing over $\pathvar$ and once more applying Lemma \ref{lem: nonexpansiveU} gives 
\begin{equation*}
     \sum_{\pathvar\in\mathcal{J}^\ell} III[\pathvar]^2 \leq R(\widetilde{\bM}, \mathbf{I})^4 \|(\bM')^*\bM'-\widetilde{\bM}^*\widetilde{\bM}\|^2_2  \|\widetilde{\bx}\|^2_{\widetilde{\bM}}.
\end{equation*}
To bound $\|(\bM')^*\bM'-\widetilde{\bM}^*\widetilde{\bM}\|_2$ we write $\widetilde{\bM}=\mathbf{R}_2(\bM',\widetilde{\bM})\bM'$ so that 
\begin{align*}
    \|(\bM')^*\bM'-\widetilde{\bM}^*\widetilde{\bM}\|_2&=\|(\bM')^*\bM'-(\bM')^*\mathbf{R}_2(\bM',\widetilde{\bM})^*\mathbf{R}_2(\bM',\widetilde{\bM})\bM'\|_2   \\
    &=\|(\bM')^*\big(\bI-\mathbf{R}_2(\bM',\widetilde{\bM})^*\mathbf{R}_2(\bM',\widetilde{\bM})\big)\bM'\|_2 \\
    &\leq \|\bM'\|_2^2
    \|\bI-\mathbf{R}_2(\bM',\widetilde{\bM})^*\mathbf{R}_2(\bM',\widetilde{\bM})\|_2\\
    &\leq R(\bM,\bI)^2 \|\bI-\mathbf{R}_2(\bM',\widetilde{\bM})^*\mathbf{R}_2(\bM',\widetilde{\bM})\|_2.
\end{align*}
The triangle inequality implies that 
\begin{align*}
    \|\bI-\mathbf{R}_2(\bM',\widetilde{\bM})^*\mathbf{R}_2(\bM',\widetilde{\bM})\|_2&\leq \|\bI-\mathbf{R}_2(\bM',\widetilde{\bM})^*\|_2+\|\mathbf{R}_2(\bM',\widetilde{\bM})^*(\bI-\mathbf{R}_2(\bM',\widetilde{\bM}))\|_2\\
    &\leq \kappa(\bM',\widetilde{\bM})(1+\bR(\bM',\widetilde{\bM}))
\end{align*}
Therefore, combining this with \eqref{eqn: Ip bound} and \eqref{eqn: IIp bound} yields 
\begin{align*}
    \sum_{p\in\mathcal{J}^\ell} |\overline{\bS}[\pathvar]  \bx - \widetilde{\overline{\bS}}[\pathvar] \widetilde{\bx}|^2 &\leq  3\Bigg( \frac{2 R(\bM, \mathbf{I})^4}{n\min_i |\bu_0(i)|^2 }\|\Pi \mathbf{x}-\widetilde{\mathbf{x}}\|^2_2 \\
    &\quad\quad\quad+ 2 R(\bM, \mathbf{I})^2 \|\bmu\|_{\bM}^2\|\mathcal{W}'-\widetilde{\mathcal{W}}\|_{\bM'}^2 \left(\sum_{k=0}^{\ell-1}\|\widetilde{W}\|_{\bM'}^k \right)^2\|\tilde{\bx}\|^2_2 \\
    &\quad\quad\quad+ R(\bM',\widetilde{\bM})^2 \|\bm{\mu}'-\tilde{\bm{\mu}}\|_{\bM'}^2 \|\widetilde{\bx}\|^2_{\widetilde{\bM}} \\
    &\quad\quad\quad+ R(\widetilde{\bM}, \mathbf{I})^4R(\bM,\bI)^2 \kappa(\bM',\widetilde{\bM})(1+\bR(\bM',\widetilde{\bM}))\|\widetilde{\bm{\mu}}\|_{\widetilde{\bM}}^2\|\widetilde{\bx}\|_{\widetilde{\bM}}\Bigg).
\end{align*}
The result follows by taking the infimum over $\Pi.$
\end{proof}

\section{The Proof of Lemma \ref{lem: nonexpansiveU}  }\label{sec: the proof of lemmas for scattering stability}

\begin{proof}
When $\ell=1,$ the result follows immediately from the fact that we have assumed that $C\leq 1$ in \eqref{eqn: frameAB} and the fact that $M$ is nonexpansive. Now, suppose by induction that the first inequality holds for $\ell.$ Let $\bx\in\mathbf{L}^2(\bM).$ Then
\begin{align}
    \sum_{\pathvar\in\mathcal{J}^{\ell+1}}\|\bU[\pathvar]\bx\|^2_{\bM}  &= \sum_{\pathvar\in\mathcal{J}^{\ell+1}}\|M\Psi_{j_{\ell+1}}\cdots M\Psi_{j_1}\bx\|^2_{\bM} \nonumber \\
    &=    \sum_{\pathvar\in\mathcal{J}^{\ell}}\left(\sum_{j_{\ell+1}\in\mathcal{J}}\|M\Psi_{j_{\ell+1}}(M\Psi_{j_{\ell}}\cdots M\Psi_{j_1}\bx)\|^2_{\bM}\right) \nonumber \\
    &\leq    \sum_{\pathvar\in\mathcal{J}^{\ell}}\|M\Psi_{j_{\ell}}\cdots M\Psi_{j_1}\bx\|^2_{\bM} \nonumber \\
    &\leq \|\bx\|^2_\bM, \label{eqn: lemma 4.10 proof}
\end{align}
with the last inequality following from the inductive assumption.

To prove the second inequality, let $t_\ell \coloneqq \left(\sum_{\pathvar\in\mathcal{J}^\ell}\|\bU[\pathvar]\bx - \widetilde{\bU}[\pathvar]\bx \|^2_{\bM}\right)^{1/2}.$ Since the modulus operator is nonexpansive, 
the definition of $\cA$ 
implies 
$t_1\leq \cA \|\bx\|_{\bM}$. Now, by induction, suppose  the result holds for $\ell.$ Then, recalling that $\bU[\pathvar]=M\Psi_{j_\ell}\cdots M\Psi_{j_1},$ we have
\begin{align*}
t_{\ell+1}&= \left(    \sum_{\pathvar\in\mathcal{J}^{\ell+1}}\|M\Psi_{j_{\ell+1}}\cdots M\Psi_{j_1}\bx-M\widetilde{\Psi}_{j_{\ell+1}}\cdots M\widetilde{\Psi}_{j_1}\bx\|_{\bM}^2\right)^{1/2}\\ 
&\leq \left(    \sum_{\pathvar\in\mathcal{J}^{\ell+1}}\|(\Psi_{j_{\ell+1}}-\widetilde{\Psi}_{j_{\ell+1}})M\Psi_{j_{\ell}}\cdots M\Psi_{j_1}\bx\|_{\bM}^2\right)^{1/2} \\
&\quad+\left(    \sum_{\pathvar\in\mathcal{J}^{\ell+1}}\|\widetilde{\Psi}_{j_{\ell+1}}(M\Psi_{j_{i}}\cdots M\Psi_{j_1}\bx-M\widetilde{\Psi}_{j_{\ell}}\cdots M\widetilde{\Psi}_{j_1}\bx)\|_{\bM}^2\right)^{1/2}\\
&\leq \mathcal{A}\left(    \sum_{\pathvar\in\mathcal{J}^{\ell+1}}\|M\Psi_{j_{\ell}}\cdots M\Psi_{j_1}\bx\|_{\bM}^2\right)^{1/2} \\
&\quad+\mathcal{C}\left(    \sum_{\pathvar\in\mathcal{J}^{\ell}}\|M\Psi_{j_{\ell}}\cdots M\Psi_{j_1}\bx-M\widetilde{\Psi}_{j_{\ell}}\cdots M\widetilde{\Psi}_{j_1}\bx\|_{\bM}^2\right)^{1/2}\\
&\leq \mathcal{A}\|\bx\|_{\bM} +t_\ell\mathcal{C}\|\bx\|_{\bM}  
\end{align*}
by 
the definitions of $\cA$ and $\mathcal{C}$ and by \eqref{eqn: lemma 4.10 proof}. By the inductive hypothesis, we have that 
\begin{equation*} t_\ell\leq \mathcal{A}\sum_{k=0}^{\ell-1}\mathcal{C}^{k}\|\bx\|_{\bM}.
\end{equation*}
Therefore,
\begin{equation*}
    t_{\ell+1}\leq \mathcal{A}\|\bx\|_{\bM}+\mathcal{A}\sum_{k=0}^{\ell-1}\mathcal{C}^{k+1}\|\bx\|_{\bM}=\mathcal{A}\sum_{k=0}^{\ell}\mathcal{C}^{k}\|\bx\|_{\bM}.
\end{equation*}
Squaring both sides completes the proof of the second inequality.
\end{proof}

\section{Implementation Details}\label{sec: implementation}
We compute the classification accuracy over a number of datasets with different scattering implementations. We use the scattering architecture from \cite{gao:graphScat2018} as implemented in \cite{tong2020data} for computation. Specifically we have $J=4$ with $2$ levels of scattering followed by aggregating the first 4 (centered) moments over the nodes for a graph invariant feature extractor. We use the clustering coefficient and the degree as input features. We then apply a linear classifier trained with gradient descent in pytorch~\cite{paszke_pytorch_2019} with 10 different seeds. Using a radial basis function or MLP did not seem to improve performance, so we use standard linear regression after scattering. We split the data as 80\% training 10\% validation and 10\% test following~\cite{tong2020data}. Code can be found at \url{https://github.com/atong01/trainable_symmetry/}.

\end{document}